\documentclass{article}

\usepackage{microtype}
\usepackage{graphicx}
\usepackage{subfigure}
\usepackage{enumitem}
\usepackage{booktabs} 
\usepackage{tikz-cd}
\usepackage{float}
\usepackage{capt-of}
\usepackage{etoc}
\usetikzlibrary{calc}

\usepackage[colorlinks=true, linkcolor=blue]{hyperref}

\newcommand{\narmeen}[1]{\textcolor{red}{(Narmeen: #1)}}

\usepackage[accepted]{icml2025}

\makeatletter
\renewcommand{\ICML@appearing}{}
\def\@icmlcorrespondingauthor{} 
\makeatother

\usepackage{amsmath}
\usepackage{amssymb}
\usepackage{mathtools}
\usepackage{amsthm}
\usepackage{thmtools, thm-restate}
\usepackage[capitalize,noabbrev]{cleveref}

\theoremstyle{plain}
\newtheorem{theorem}{Theorem}

\newtheorem{corollary}[theorem]{Corollary}
\newtheorem{lemma}{Lemma}[section]
\crefname{lemma}{Lemma}{Lemmas}
\numberwithin{equation}{section}
\theoremstyle{definition}

\theoremstyle{remark}

\usepackage[most]{tcolorbox}
\newtcolorbox{lemmaproofbox}{
  breakable, colback=white, colframe=black, arc=0mm,
  boxrule=0.25pt, left=10pt, right=10pt, top=10pt, bottom=10pt
}
\usepackage[textsize=tiny]{todonotes}

\icmltitlerunning{}

\begin{document}

\twocolumn[
\icmltitle{Spectral Superposition: A Theory of Feature Geometry}

\icmlsetsymbol{corr}{$\dagger$}
\icmlsetsymbol{senior}{$\ddagger$}

\begin{icmlauthorlist}
\icmlauthor{Georgi Ivanov}{Theopha,Harvard,corr}
\icmlauthor{Narmeen Oozeer}{Martian}
\icmlauthor{Shivam Raval}{Harvard}
\icmlauthor{Tasana Pejovic}{Martian}
\icmlauthor{Shriyash Upadhyay}{Martian}
\icmlauthor{Amir Abdullah}{Martian,Thoughtworks,senior}
\end{icmlauthorlist}

\icmlaffiliation{Theopha}{Theopha}
\icmlaffiliation{Martian}{Martian}
\icmlaffiliation{Harvard}{Harvard University}
\icmlaffiliation{Thoughtworks}{Thoughtworks}

\icmlcorrespondingauthor{}{georgi@theopha.com}

\icmlkeywords{Machine Learning, Feature Geometry}

\vskip 0.3in
]

\printAffiliationsAndNotice{\textsuperscript{$\dagger$}Corresponding Author.\textsuperscript{$\ddagger$}Senior Author.}

\thispagestyle{empty} 


\begin{abstract}

Neural networks represent more features than they have dimensions via superposition, forcing features to share representational space. Current methods decompose activations into sparse linear features but discard geometric structure. We develop a theory for studying the geometric structre of features by analyzing the spectra (eigenvalues, eigenspaces, etc.) of weight derived matrices. In particular, we introduce the frame operator $F = WW^\top$, which gives us a spectral measure that describes how each feature allocates norm across eigenspaces. While previous tools could describe the pairwise interactions between features, spectral methods capture the global geometry (``how do all features interact?'').
In toy models of superposition, we use this theory to prove that capacity saturation forces spectral localization: features collapse onto single eigenspaces, organize into tight frames, and admit discrete classification via association schemes, classifying all geometries from prior work (simplices, polygons, antiprisms).
The spectral measure formalism applies to arbitrary weight matrices, enabling diagnosis of feature localization beyond toy settings. These results point toward a broader program: applying operator theory to interpretability.

\end{abstract}

\section{Introduction}
\begin{figure*}
    \centering
    \begin{tikzpicture}[
    >=Stealth,
    node distance=2.5cm,
    label/.style={font=\small\itshape, align=center},
    stage/.style={font=\footnotesize, align=center, text width=2cm},
    arrow/.style={->, thick, shorten >=2pt, shorten <=2pt}
]

\begin{scope}[local bounding box=stage1,
    x={(0.5cm, -0.25cm)},
    y={(0.5cm, 0.25cm)},
    z={(0cm, 0.5cm)}]
    
    \node[stage] at (0, 0, 3.5) {Representation\\in model};
    
    \foreach \i in {-1.5,-1,...,1.5} {
        \draw[gray!40, thin] (\i, -1.5, 0) -- (\i, 1.5, 0);
        \draw[gray!40, thin] (-1.5, \i, 0) -- (1.5, \i, 0);
    }
    
    \draw[thick, blue!70!black, dashed] (0, 0, 0) -- (0, 0, -2.2);
    \fill[blue!70!black] (0, 0, -2.2) circle (2pt);
    \node[font=\footnotesize, blue!70!black] at (0.3, 0.3, -2.2) {T};
    
    \coordinate (R1) at (1.3, 0, 0);
    \coordinate (P1) at (-0.65, -1.13, 0);
    \coordinate (S1) at (-0.65, 1.13, 0);
    
    \draw[thick, red!70!black] (R1) -- (P1) -- (S1) -- cycle;
    \fill[red!70!black] (R1) circle (2pt);
    \fill[red!70!black] (P1) circle (2pt);
    \fill[red!70!black] (S1) circle (2pt);
    \node[font=\footnotesize, red!70!black] at (1.7, 0, 0) {R};
    \node[font=\footnotesize, red!70!black] at (-0.9, -1.5, 0) {P};
    \node[font=\footnotesize, red!70!black] at (-0.9, 1.5, 0) {S};
    
    \draw[thick, blue!70!black] (0, 0, 0) -- (0, 0, 2.2);
    \fill[blue!70!black] (0, 0, 2.2) circle (2pt);
    \node[font=\footnotesize, blue!70!black] at (0.3, 0.3, 2.2) {H};
    
\end{scope}

\draw[arrow] (1.3, 0) -- node[above, font=\scriptsize] {feature} node[below, font=\scriptsize] {extraction} (2.7, 0);

\begin{scope}[shift={(4, 0)}, local bounding box=stage2]
    \node[stage, above, text width=2.5cm] at (0, 1.2) {Separate concept\\vectors\\(current methods)};
    \node[font=\footnotesize] at (0, 0.8) {$\vec{R}$};
    \node[font=\footnotesize] at (0, 0.4) {$\vec{P}$};
    \node[font=\footnotesize] at (0, 0) {$\vec{S}$};
    \node[font=\footnotesize] at (0, -0.4) {$\vec{H}$};
    \node[font=\footnotesize] at (0, -0.8) {$\vec{T}$};
\end{scope}

\draw[arrow] (5.3, 0) -- node[above, font=\scriptsize] {spectral} node[below, font=\scriptsize] {measure} (6.7, 0);

\begin{scope}[shift={(8, 0)}, local bounding box=stage3]
    \node[stage, above, text width=2.5cm] at (0, 1.2) {Independent\\activation\\subspaces};
    \node[font=\footnotesize, align=left, blue!70!black] at (0, 0.4) {Rank 1: $\vec{H}$, $\vec{T}$};
    \node[font=\footnotesize, align=left, red!70!black] at (0, -0.3) {Rank 2: $\vec{R}$, $\vec{P}$, $\vec{S}$};
\end{scope}

\draw[arrow] (9.5, 0) -- node[above, font=\scriptsize] {spectral} node[below, font=\scriptsize] {signature} (11, 0);

\begin{scope}[shift={(12.5, 0)}, local bounding box=stage4]
    \node[stage, above] at (0, 1.5) {Recovered\\geometry};
    
    \draw[thick, blue!70!black] (0, -0.7) -- (0, 0.7);
    \node[font=\footnotesize, blue!70!black] at (-0.25, 0.7) {H};
    \node[font=\footnotesize, blue!70!black] at (-0.25, -0.7) {T};
    \fill[blue!70!black] (0, 0.7) circle (2pt);
    \fill[blue!70!black] (0, -0.7) circle (2pt);
    
    \coordinate (R4) at (1, 0.5);
    \coordinate (P4) at (0.567, -0.25);
    \coordinate (S4) at (1.433, -0.25);
    \draw[thick, red!70!black] (R4) -- (P4) -- (S4) -- cycle;
    \node[font=\footnotesize, red!70!black] at (1, 0.75) {R};
    \node[font=\footnotesize, red!70!black] at (0.32, -0.4) {P};
    \node[font=\footnotesize, red!70!black] at (1.68, -0.4) {S};
    \fill[red!70!black] (R4) circle (2pt);
    \fill[red!70!black] (P4) circle (2pt);
    \fill[red!70!black] (S4) circle (2pt);
\end{scope}

\end{tikzpicture}

\centering
\caption{A diagram showing the concepts of "rock(R)/paper(P)/scissors(S)" and "heads(H)/tails(T)" embedded in a three dimensional space. This is an example of structural interference, where a model causes concepts to interfere because they are related (the hand-game features interfere and the coin features interfere based on the relationships between the objects in each set, but the two sets do not interfere with each other). Current methods, which only extract features, will create a basis with 5 vectors. We capture the relationships between features in activation space using the \textit{spectral measure} to bin features by the eigenspaces they occupy. When each feature is spectrally localized (occupies a single eigenspace), we can recover the full geometry by covering it with an additional \textit{spectral signature}. These tools from operator theory let us explore the global character of interference: not just how pairs of features interact, but how \textit{all} features relate.}
\label{fig:rps-ht}

\end{figure*}
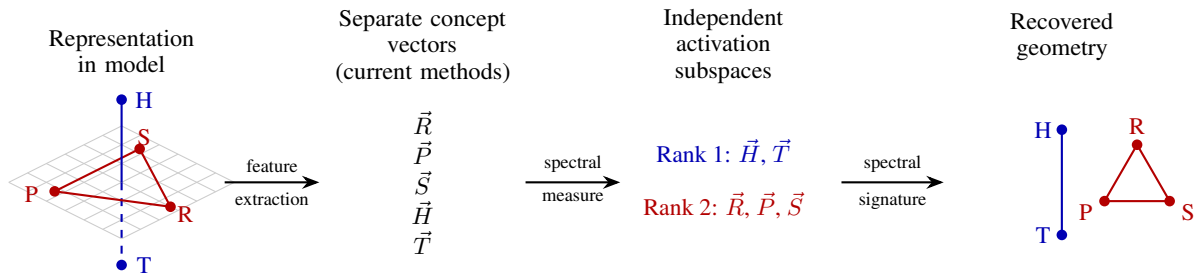

Neural networks represent far more features than they have dimensions in their activation space, a phenomenon known as \textit{superposition} \citep{toymodels, superposition_chan}. When a model needs to encode $f$ features in $d$ dimensions with $f > d$, it cannot give each feature its own orthogonal direction. Instead, features must share representational space, compressed into overlapping directions that exploit the high-dimensional geometry available to them.

A direct consequence of superposition is \textit{feature interference}: multiple features compete for the same representational directions, so modifying one inevitably perturbs others. In practice, we see activation steering methods \citep{steering_base_1, steering_base_2, steering_base_3} often produce unintended side effects, where pushing on one concept inadvertently shifts others \citep{bad_steering_1, bad_steering_2, bad_steering_3}.

Current interpretability approaches address interference by decomposing activations into sparse, independent linear features. Sparse autoencoders (SAEs) and related methods attempt to recover a dictionary of features that activate sparsely across inputs, effectively trying to undo the compression that superposition creates. The hope is to find a basis where each feature corresponds to a single, interpretable direction.

However, this approach discards something important. There are at least two distinct reasons why a model might cause features to share dimensions:

\begin{itemize}
    \item \textbf{Incidental interference.} Some features are anti-correlated in when they activate, they rarely or never co-occur in the same input \citep{lecomte2023incidental}. If ``marine biology'' and ``abstract algebra'' are both sparse features that almost never appear together, the model can efficiently reuse the same dimensions for both, disambiguating by context. This kind of sharing is a compression trick we would like to undo: the features are conceptually independent, just packed together for efficiency.
    \item \textbf{Structural interference.} Other features share dimensions because there is genuine structure in the data that makes this natural. The days of the week might be encoded in a shared subspace precisely \textit{because} the model needs to represent relationships between them, their cyclic structure, their ordering, the fact that Monday is ``between'' Sunday and Tuesday \citep{engels2024not}. Here, the geometry is not an artifact to be removed but a meaningful representation of how concepts relate.
\end{itemize}

A growing body of work demonstrates that models do encode meaningful geometric relationships between features, showing evidence of structural sharing that sparse linear decomposition discards. Models solving modular arithmetic trace helices in representation space \citep{modularaddition}. Spatial reasoning tasks induce geometric embeddings \citep{spaceandtime}. The Evo 2 DNA foundation model organizes biological species according to their phylogenetic relationships on a curved manifold \citep{treeoflife}. Similar geometric organization has been observed across domains and architectures \citep{ categoricalconcepts, anthropic2025manifolds, projectingassumptions, sequencemodels, flattohierarchical, pan2025hidden, nguyen2025angular}. These findings suggest that geometry is not an incidental artifact but a core mechanism by which features coexist and interact.

When features are embedded with non-trivial geometry, treating them as independent directions loses information. Consider a stylized example (Figure~\ref{fig:rps-ht}) of how models might represent game-playing: five features representing ``rock,'' ``paper,'' ``scissors,'' ``heads,'' and ``tails,'' embedded in three dimensions. The rock-paper-scissors features form a triangle in a shared plane; the heads-tails features lie along a separate axis. This is structural interference, the hand-game features share a subspace because their relationships matter, and likewise for the coin features, but the two groups remain orthogonal. Current feature extraction methods return five individual vectors. The geometric structure—which features share a subspace, which are orthogonal, how they are arranged, is lost.

We want to recover this global structure: not just features in isolation, but how they organize collectively into sub-geometries that may carry functional meaning. \citet{higher_order_features} show that inducing refusal in Gemma2-2B-Instruct requires ablating 2,538 SAE features, and that ``backup'' features activate when others are suppressed, suggesting that model functionality lives in collective geometric structure, not individual directions.

\textbf{Our contributions are as follows:}
\begin{itemize}
    \item We motivate the use of spectral theory (eigenspaces, etc.) to study interference by developing an intuition through our game-playing example: interference only occurs when features share subspaces of activation space (the eigenspaces).
    \item We develop a spectral theory of superposition, providing generalizable tools that can be applied in both toy and realistic settings to study feature geometry.
    \item We apply these tools to the canonical toy model of superposition \citep{toymodels}, fully characterizing the geometry of features that form in this setting.
\end{itemize}

\section{Motivating Spectral Methods}
\label{section:warm-up}

In this section, we will motivate the study of feature geometry using spectral theory by looking more closely at our games of chance example (Figure \ref{fig:triangle-digon}). There are three notable features about this example:
\begin{enumerate}
    \item \textbf{The geometric structure is permutation- invariant.} Permuting feature vectors within clusters does not change the global structure and interaction. 
    \item \textbf{The HT and RPS subspaces are independent.} No coin toss would provide a better winning chance at RPS and vice versa, and thus the concepts are encoded in orthogonal subspaces.
    \item \textbf{The geometric arrangement of features encodes precise symmetries.} In RPS: the cyclic rule of rock beats scissors, scissors beat paper, paper beats rock. In HT: and the binary opposition of heads or tails.
\end{enumerate}

The classic way of studying interference between features is the Gram Matrix $M = W^\top W$, where each entry $M_{ij} = \langle W_i, W_j\rangle = W_i^\top W_j$ is the Euclidean norm between column vectors. We could try to use this method to determine which vectors interact and which don't.
\vspace{0.25cm}
\begin{figure}[h]\label{fig:triangle-digon}
    \centering
    \includegraphics[width=\linewidth]{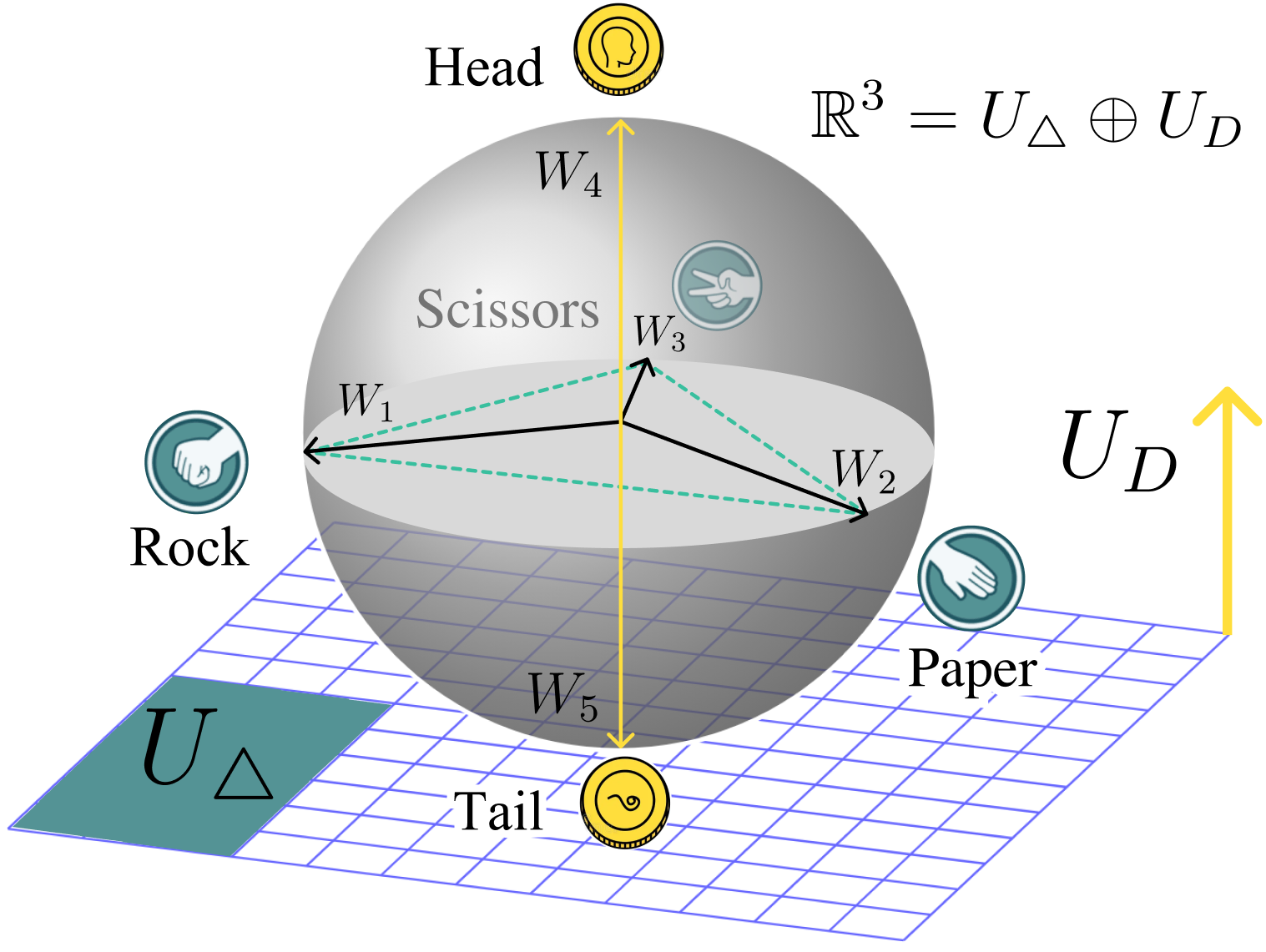}
    \label{fig:placeholder}
    \vspace{-20pt}

    \centering
    $$W = \begin{bmatrix}
      \vdots  & & \vdots  \\
      W_1 & \cdots  & W_5 \\
        \vdots & & \vdots 
      \end{bmatrix}=  \begin{bmatrix}
    \frac{1}{2} & \frac{1}{2} & -1 & 0 & 0 \\
    \frac{\sqrt{3}}{2} & \frac{-\sqrt{3}}{2} & 0 & 0 & 0 \\
    0 & 0 & 0 & 1 & -1
    \end{bmatrix}$$

    \caption{The games of chance example in more detail, both geometrically and as a matrix. Here, $W_\triangle = [W_1, W_2, W_3]$ represent the feature concepts for ``rock", ``paper", and ``scissors", respectively, arranged as an equilateral triangle in the $xy$-plane. Similarly, $W_D= [W_4, W_5]$ correspond to ``heads" and ``tails" and form an antipodal pair along the $z$-axis.}
\end{figure}

Given how we have written $W$ (Figure \ref{fig:triangle-digon}), this would work: direct matrix multiplication reveals a block matrix $M = M_\triangle\oplus M_D= (W_\triangle^\top W_\triangle)\oplus (W_D^\top W_D)$. However, this clean block-diagonal structure is an artifact of indexing choice. Only permutations (\ref{appendix:orbit} via matrices $P_\gamma$) within the local feature clusters $\Omega_\triangle = \{1, 2, 3\}$ and $\Omega_D = \{4, 5\}$ leave the block decomposition intact.

Imagine that we rearrange the column vectors or rotate the matrix; this structure would no longer be visually evident in $W$ or in $M$, but the underlying geometry would remain unchanged.

Our goal in this section is to find a way to express the global symmetries of features in $W$ (to capture the geometry), in a way that is basis invariant. This is precisely what spectral analysis allows us to do.

\noindent\textbf{2.1. Geometric Invariance}

At a high level, our argument proceeds as follows: geometric structure induces symmetry constraints; these constraints can be expressed as algebraic relations among matrices; and these relations are fully captured by spectral data. In this subsection, we assume the geometry is known and show how spectral structure encodes it. In subsection 2.2, we reverse the direction: given only the weight matrix, we recover the geometry from spectra in activation space.

We established that we need to get information about interference in a permutation invariant way. As such, consider arbitrary column labels $\Omega_\triangle = \{i_1, i_2, i_3 \}$ and $\Omega_D=\{i_4, i_5\}$ with $W_\triangle = [W_{i_1}, W_{i_2}, W_{i_3}]$ and $W_D = [W_{i_4}, W_{i_5}]$. 

The group that preserves the equilateral triangle is the dihedral group $D_3 \cong S_3$. We can denote these elements in permutation notation $\sigma=(i\ j\ k)$, where $i\rightarrow j\rightarrow k\rightarrow i$ represents the order in which the sub-indices get switched around (Figure \ref{fig:dihedral}).

\begin{figure}[h]
    \vspace{-10pt}
    \centering
    \includegraphics[width=\linewidth]{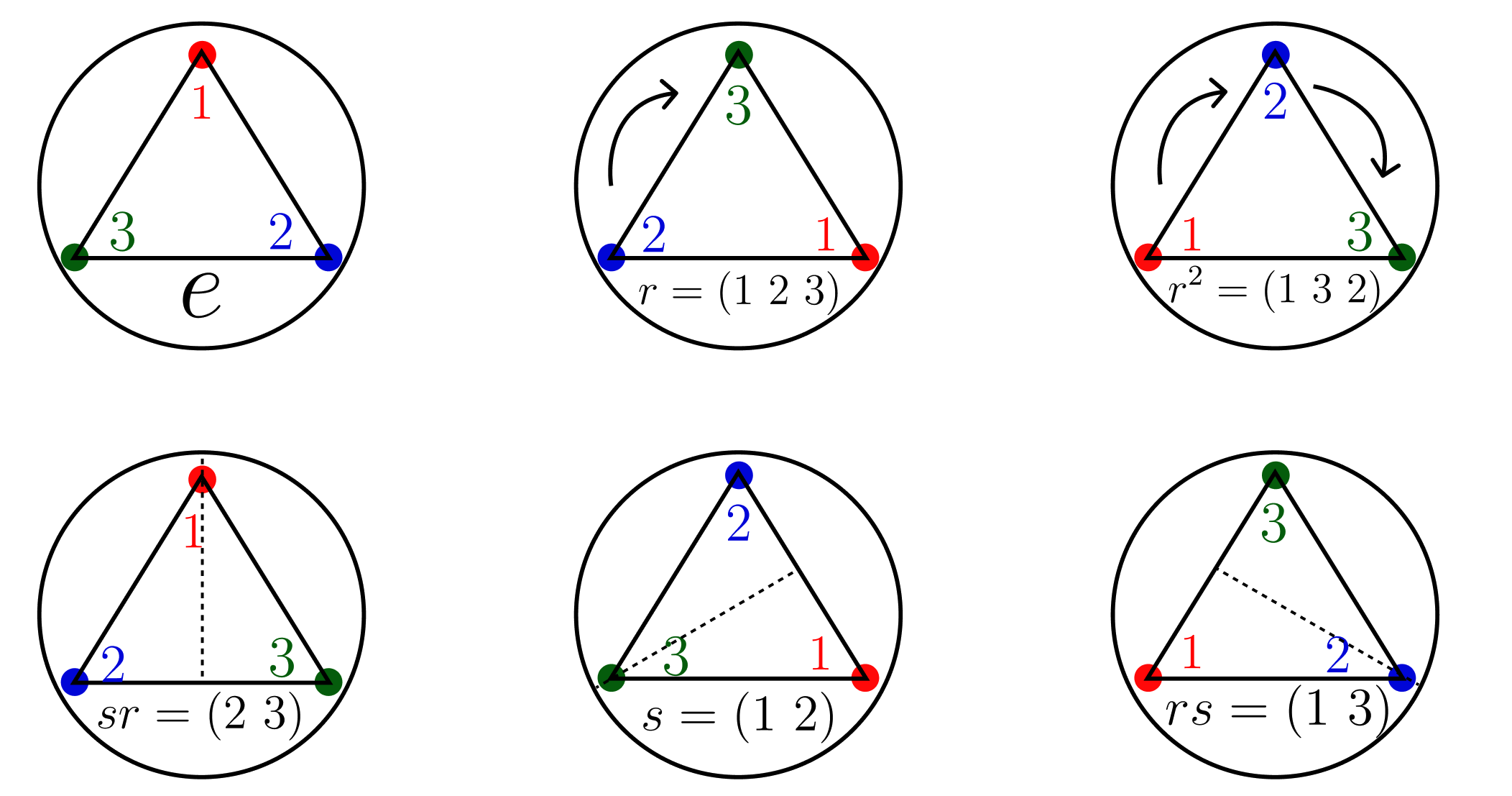}
    \label{fig:dihedral}
    \vspace{-20pt}
    \caption{The elements of $D_3$ in permutation notation.}
\end{figure}

We can express the group action in terms of adjacency matrices over $\Omega\times \Omega$, each representing the stable orbits (the trajectory of a single index after applying every possible permutation $\sigma\in D_3$). In our context,every vertex is accessible within $1$ step. As such, outside of the trivial adjacency matrix $A_0 =I_3$, representing the stationary orbit $\mathcal{O}_0 = \text{diag}(\Omega)$, we only have $A_1 = J_3-I_3$ (P.f. in \ref{appendix:simplex}), where $J_3\in \mathbb{R}^3$ is the matrix of all ones. We will focus on $\mathcal{A}_\triangle= \operatorname{span}\left\{A_0, A_1\right\}$, called the Bose-Mesner algebra. Why work with this algebra rather than simply diagonalizing $M_\triangle$ directly? The answer is that $\mathcal{A}_\triangle$ captures exactly the matrices that commute with all symmetry operations---it is the centralizer of $D_3$ (Pf.\ in \ref{appendix:orbit}). Any $D_3$-invariant operator, including $M_\triangle$, must lie in this algebra and hence decompose over its spectral projectors $S_0, S_1$. This is more powerful than generic diagonalization: it tells us not just the eigenvalues, but \textit{why} they arise from symmetry. The Bose-Mesner algebra is a central object in the theory of association schemes \cite{rabailey} and comes equipped with its own spectral theory over strata $\mathbb{R}^{\Omega_\triangle} \cong \mathbb{R}^3 \cong V_0 \oplus V_1$ (Lemma \ref{lemma:2.6}).
The corresponding stratum projectors $S_0$ and $S_1$ serve as a more convenient basis of $\mathcal{A}$. There are orthogonality relations connecting them (Lemma \ref{lemma:characters}): 
\vspace{-0.2cm}
\begin{equation}A_i = \sum_{e=0}^1 C(i, e) S_e, \quad S_e = \sum_{i=0}^1 D(e, i)A_i\end{equation}\vspace{-0.2cm}

The $C(i, e)$ entries are the eigenvalue of $A_i$ on $W_e$, forming a character table $C\in \mathbb{R}^{2\times 2}$. The values of $D(e, i)$ comprise a matrix that is the inverse of $C$. The same lemma states that $S_0 = |\Omega_\triangle|^{-1}J = J/3$ and provides us with explicit values for some of the coefficients, such as $C(0, 0) = C(0, 1) = 1$ and $C(1, 0) = a_i= 2$. Here $a_i$ is the non-trivial orbit's degree, representing the number of neighbors an index can reach under the group action. By using the latter information along with $(2.1)$, we can explicitly calculate $S_1 = A_0 - S_0 = I_3-\frac{1}{3}J_3$. With both projection matrices we recover $V_0 = \text{span}(\mathbf{1}_3)$ and $V_1 = \mathbf{1}_3^\perp$. Lastly, by solving for $A_1 = 2S_0 + C(1,1)S_1$, we complete the character table with $C(1, 1) = -1$ and $D$ along with it per Lemma \ref{lemma:characters}: $$C = \begin{bmatrix}1 & 1 \\ 2 & -1\end{bmatrix}, \quad D = \frac{1}{3}\begin{bmatrix}
1 & 1 \\ 2 & -1
\end{bmatrix}$$
Using the characters and the fact that $M_\triangle$ is $D_3$-invariant (Pf. in \ref{appendix:simplex}), we can express it over the basis $A_i$ as:
$$M_\triangle  =  I - \frac{1}{2} (J-I) =\frac{3}{2}I - \frac{1}{2}J = \frac{3}{2}S_1$$
Now, let us proceed with the digon, whose group is $\Gamma \cong S_3 \cong C_2$. This is the cyclic group of two elements $\{e, s\}$, namely the identity $e$ and $s$ is an element of order $2$, meaning that $s^2=1$. This visually represents swapping around the two nodes of the digon (i.e. going from heads-tails to tails-heads). Per the fact that this is a simplex (i.e one can reach from one node to every other), we have the same number of orbits/associate classes $\mathcal{C}_0 =\operatorname{diag}(\Omega_D)$ and $\mathcal{C}_1 = \{(4, 5); (4, 5)\}$ with adjacency matrices $A_0' = I_2$ and  $A_1'=J_2-I_2$. Similarly, Lemma \ref{lemma:2.6} for $\mathcal{A}_D = \operatorname{span}(A_0', A_1')$ yields $\mathbb{R}^{\Omega_D} \cong \mathbb{R}^2 = U_0\oplus U_1$ and by Lemma \ref{lemma:characters} $S_0' = \frac{1}{2}I$, $S_1' = I-\frac{1}{2}J$. As illustrated below, the spaces $\mathbb{R}^{\Omega_\triangle}$(left) and $\mathbb{R}^{\Omega_D}$ (right) decompose into a centroid subspace (red arrows) and a difference subspace (the shaded plane/yellow diagonal, corresponding to $S_1$ and $S_1'$). The identities $M_\triangle = \frac{3}{2}S_1$ and $\operatorname{rank}(M_\triangle)=2$ show that, within each symmetric cluster, the Gram operator acts as a projector that annihilates the centroid with eigenvalue $0$ and scales the difference plane by $\lambda_\triangle = \frac{3}{2}$ or $\lambda_D = 2$, respectively. This immediately implies an intrinsic rank drops $\operatorname{rank}(M_\triangle)=2$ and $\operatorname{rank}(M_D)=1$. 
 
\begin{figure}[h]
    \vspace{-10pt}
    \centering
    \includegraphics[width=0.6\linewidth]{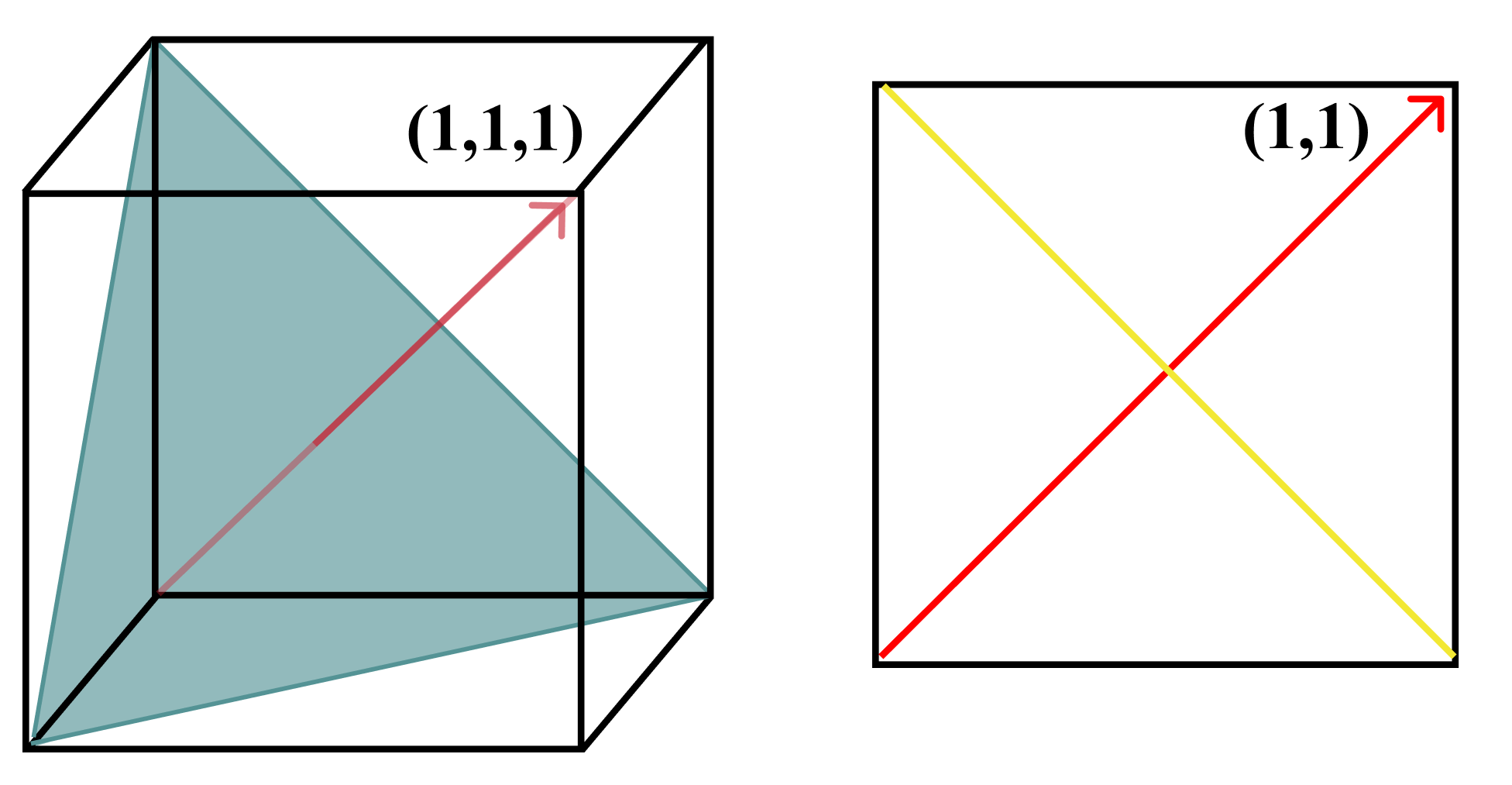}
    \label{fig:strata}
    \caption{Association Scheme Strata}
    \vspace{-10pt}
\end{figure}
 However, this description is not visible at the level of Gram entries once feature indices are scrambled. What survives the index permutation is only the spectral content (eigenvalues/eigenspaces up to relabeling), which is why we next pass to an activation-space operator that is itself invariant under relabeling. \vspace{0.1cm}
\newline 
\noindent\textbf{2.2. Spectral Bridge}

Section 2.1 showed that \textit{known} geometry is encoded in spectral data. But in practice, we observe only the weight matrix $W$, not the underlying feature structure. How can we recover geometry without knowing it in advance? The key is to work in activation space rather than index space. The Gram matrix $M = W^\top W$ transforms under index permutations as $M \mapsto P^\top M P$, scrambling its entries. The frame operator $F = WW^\top$, by contrast, is invariant: $(WP)(WP)^\top = WW^\top = F$. This makes $F$ the natural object for basis-invariant geometric analysis.

Up to now we have described the underlying symmetry as spectral modes of the Gram $M$ (index-dependent) in the abstract concept/index space $\mathbb{R}^{\Omega_\triangle}\oplus \mathbb{R}^{\Omega_D}$. However, the geometry that actually determines interference and capacity lives in activation space $\mathbb{R}^3$. Scrambling feature indices corresponds to right-multiplying $W$ by a permutation matrix $P_\gamma$, i.e. $W\rightarrow WP_\gamma$ (Pf. in \ref{appendix:orbit}). This conjugates the Gram $M\rightarrow P^\top MP$, obscuring the local block structure in the entries. Thus, if we want a genuinely label-free description of geometry-informed interference we must pass to an activation-space object that is invariant under such relabelings and preserves the spectral information we uncovered earlier. The canonical such object is the frame operator $F: = WW^\top\in \mathbb{R}^{3\times 3}$, which remains unchanged under index permutations $(WP_\gamma)(WP_\gamma)^\top = WP_\gamma P_\gamma^\top W^\top =WW^\top = F$. And while gluing back together $M_\triangle$ and $M_D$ in the concept space is non-trivial due to the arbitrary indexing, the process becomes straightforward for the Frame operator $F=\sum_{i=1}^n W_iW_i^\top $: 

 \vspace{-10pt}
 \begin{equation}
 F= \sum_{i\in \Omega_\triangle}W_iW_i^\top + \sum_{i\in \Omega_D}W_iW_i^\top = F_\triangle + F_D
 \end{equation}

We use the frame operator because of a precise spectral correspondence between $M$ and $F$. Let $M = \sum_e \lambda_e E_e$ be the Gram's spectral decomposition \ref{lemma:spectral-decomposition}. Then, by the Spectral Correspondence Lemma \ref{lemma:spectral_correspondence}, $M$ and $F$ share the same non-zero eigenvalues $\lambda_e>0$ and the primitive idempotents $E_e$ of $M$ lift to 
 \begin{equation}
 \label{equation:spectral_decomposition}
    P_e=\lambda_e^{-1}WE_eW^\top.
 \end{equation}
 We can apply this lemma to derive a corollary establishing the precise relationship between the local cluster projectors and the frame's decomposition (Corollary \ref{corollary:triangle-digon}): $F = 2P_D+\frac{3}{2}P_\triangle$, where $P_D$ and $P_\triangle$ are projectors onto the subspaces along which the digon and triangle live, respectively. Furthermore, for $C\in \mathcal\{\triangle, D\}$, we have the following consequence: $\frac{1}{\lambda_C} = \frac{\dim (U_C)}{\sum_{I\in \Omega_C}\|W_i\|^2}= \frac{\operatorname{rank}(P_C)}{|\Omega_C|},$ namely that the fractional dimensionality $D_i=d/p$, denoting the fraction of features that share dimension, is recoverable from spectral signature. In the following section, we provide detailed analysis about the meaning of our result and how the ideas applied here generalize to an arbitrary weight matrix setting.

\section{Generalizable Tools}
\label{section:generalizable}
In the previous section, we show that the geometry can recovered from the frame operator that lives in the latent space and can encode structured representations. In this section, we focus on a set of tools from the previous section that generalize to arbitrary weight matrices. In particular,  we show that the fact that features arrange along eigenspaces can determine their interaction patterns (given by the spectral bridge), the capacity consumed by a given feature (fractional dimensionality) as well as how localized the feature is.

\textbf{Spectral Bridge: }
The spectral bridge introduced in Section \ref{section:warm-up} is generalizable to arbitrary weight matrices $W$ and classifies the intertwining of the input space (via the Gram matrix $M$) and the activation space (via the Frame operator $F$). We are able to recover the geometry from the spectral behavior of $F$, classifying it by its invariance with the help of association schemes (Appendix \ref{appendix:orbit}). More specifically, $W$ intertwines the decompositions of $M$ and $F$. As a result, the nonzero eigenspaces of $M$ and $F$ are in one-to-one correspondence. An  equally important consequence of the decomposition of the frame operator into orthogonal subspaces is that feature interactions are confined to shared spectral subspaces: components supported in the same induced eigenspace can interact, components supported in different eigenspaces are non-interacting, and directions orthogonal to $\operatorname{Im}(W)$ are inert, meaning they are annihilated by $W^\top$ and therefore do not participate in feature interactions or capacity usage.
 
\textbf{Spectral Measure: } While the global Empirical Spectral Density \citep{htsr1} describes \emph{where capacity exists}, the collection of spectral measures describes \emph{how individual features use that capacity}. Two models may therefore exhibit similar heavy-tailed global spectra while possessing very different patterns of feature localization, interference, and effective dimensionality which is revealed by the spectral measure.

The primary object of study here is the per-feature dimensionality, defined in Toy models \cite{toymodels} to study interactions between features encoded as:
\begin{equation}D_i=\frac{\|W_i\|^4}{\sum_{j=1}^n\left({W}_i^\top{W}_j\right)^2}\end{equation}

It represents the degree of superposition with other features, where $D_i\approx 1$ when the feature vector $W_i$ has its own dedicated dimension (i.e. it is approximately orthogonal to $W_j$, where $i\neq j$) and $0\leq D_i\ll 1$ when it co-interacts with many others in superposition.  The current form hides some of the more interesting relations. We can rewrite it as an expression of the Gram $M$ and the frame operator $F$ matrices as follows (Lemmas \ref{lemma:dimensionality_gram} and \ref{lemma:dimensionality_kappa}):
\begin{equation}D_i = \frac{(M_{ii}^2)}{(M^2)_{ii}} = \frac{1}{\kappa_i} \|W_i\|^2,\end{equation}\label{eq:3.2}
where $\kappa_i = \frac{W_i^\top F W_i}{\|W_i\|^2}$ is the Rayleigh quotient of $F$ at $W_i$. Note that $\||W_i\||>0$ at initialization with probability $1$ and under the capacity saturation we observe, it is stable w.r.t. to rank collapse (after noticing $\epsilon u$ with $\epsilon\rightarrow 0$ preserves the bound \ref{lemma:cyclic_cs}). 

 We can view the Gram first form as a signal-to-noise ratio derived from pairwise correlations, while the second one recasts this as a spectral problem.
We now introduce the \emph{spectral measure}, whose role is to provide a measure of how features allocate their mass across eigenmodes.

Consider the weight matrix, with associated Gram and frame operators $M$ and $F$. As shown earlier, all nontrivial interaction structure in activation space is captured by the spectral decomposition (Equation \ref{equation:spectral_decomposition}) where the $P_e(t)$ are mutually orthogonal projectors corresponding to independent spectral modes. To describe how an individual feature participates in this structure, for each feature vector $W_i$, we define $p_{i,e} := \frac{\|P_e W_i\|^2}{\|W_i\|^2}$ which measures the fraction of the feature’s squared norm supported in spectral mode $e$. We note that these coefficients are nonnegative and Lemma \ref{lemma:spectral_measure_probability_distribution} shows that they sum to one, and therefore form a probability distribution. This motivates the definition of the spectral measure
\begin{equation}
    \label{equation:probability_distribution_eigenmodes}
    \mu_i(t) := \sum_{e \in \mathcal{E}_+} p_{i,e}(t)\,\delta_{\lambda_e(t)}.
\end{equation}

The spectral measure provides a compact description of how a feature allocates its mass across interaction modes. Moreover, operator-level quantities can now be interpreted probabilistically. Lemma \ref{lemma:moments} shows that all moments of $\mu_i$ are accessible via powers of the frame operator. In particular, the Rayleigh quotient admits the represention $\kappa_i
= \frac{W_i^\top F W_i}{\|W_i\|^2}
= \mathbb{E}_{\mu_i}[\lambda].
$

Finally, fractional dimensionality can be written as
\begin{equation}
    D_i(t) = \frac{\|W_i(t)\|^2}{\mathbb{E}_{\mu_i(t)}[\lambda]},
\end{equation}

making explicit that effective dimensionality is governed by how a feature distributes its mass across the spectrum.
While the ESD provides a \emph{global} summary of how capacity is distributed across spectral scales, it does not describe how individual features participate in this structure. To address this, our framework introduces a \emph{per-feature spectral measure}. For each feature vector $W_i$, the coefficients $p_{i,e}$ in Equation \ref{equation:probability_distribution_eigenmodes} form a probability distribution over spectral modes and quantify the fraction of the feature’s squared norm supported in each eigenspace.

This construction allows us to distinguish between two regimes at the level of individual features. A feature is said to be \emph{delocalized} if its spectral measure $\mu_i$ spreads mass across many eigenvalues, indicating participation in multiple interaction modes. Conversely, a feature is \emph{localized} if $\mu_i$ concentrates on a small number of eigenvalues, or collapses to a single Dirac mass. In the extreme case $\mu_i = \delta_{\lambda}$, the feature lies entirely within a single eigenspace of the frame operator.

\section{Solving Toy Models of Superposition}
\label{section:tms}

In this section, we show how some of machinery from the spectral framework developed above can be used to exhaustively explain superposition in toy models. We begin by empirically observing that trained models operate near capacity saturation, motivating a regime in which fractional dimensionality exhausts the rank of the weight matrix. Under this assumption, we prove that all features must spectrally localize, collapsing onto individual eigenspaces of the frame operator (Theorem \ref{theorem:spectral_localization}). Once localized, features no longer trade capacity across eigenspaces: fractional dimensionality becomes linearly determined by feature norm and spectral scale (Corollary \ref{corollary:projective_linearity_section4}), and features within each eigenspace organize into tight frames (Theorem \ref{theorem:decomp}). This tight-frame structure makes feature geometry identifiable via association schemes (Theorem \ref{theorem:spectral-identification}), enabling discrete geometric classification. Finally, we connect these static results to training dynamics, showing how gradient flow induces spectral mass transport and eigenvalue drift, and why capacity-saturated tight-frame configurations emerge as stable fixed points. Together, these results provide a complete picture of how superposition resolves into structured, stable geometry under spectral constraints.

The ``Toy Models of Superposition" (TMS) paper by Anthropic \cite{toymodels} is one of the seminal technical contributions to the understanding of polysemanticity. The phenomenon is explored by studying  synthetic input vectors $x$ that simulate the properties of the underlying features. Each has an associated sparsity $S_i$ (i.e $\mathbb{P}(x_i=0) = S_i$ and $x_i\sim \text{Unif}(0, 1)$ otherwise) and importance $I_i$ for each dimension $x_i$. Define the projection as a linear map $h = W x ,$ where each column $W_i$ corresponds to the direction in the lower-dimensional space that represents a feature $x_i$. To recover the original vector, we use the transpose $W^T$, giving the reconstruction $x'$. With the addition of the ReLU nonliearity, the model becomes $x' = \text{ReLU}(W^T h + b) = \text{ReLU}(W^T W x + b)$. \newline 
\textbf{Experimental Setup}\newline
To investigate the emergence of superposition in toy models, we are conducting a high-throughput parameter sweep of 3,200 experiments, using mean squared error (MSE), weighted by feature importance $L = \sum_x \sum_i I_i (x_i - x'_i)^2$, as in the original paper \cite{toymodels}. We fix the input dimension to $n=1024$ features and sample linearly 32 discrete values for $M\in [16,512]$ and 50 discrete values for sparsity $(S\in[0.0,0.99])$ ($2$ seeds per run). We set uniform importance $I_i=I=1$ and sparsity $S_i=S$, but discuss in the appendix \ref{appendix:non_uniform_sparsity} our results for non-uniform sparsity.

\subsection{Capacity Saturation is Spectral Measure Localization} 

Our first important empirical observation $(E1)$ is that the toy models utilize approximately all of their hidden dimension capacity $\sum_i D_i \approx m$. The following 
plot shows sum of individual feature fractional dimensionalities (as defined in \ref{eq:3.2}.2) over all 3200 training runs:
\label{empirical:1}
\begin{figure}[h]
\includegraphics[width=0.47\textwidth]{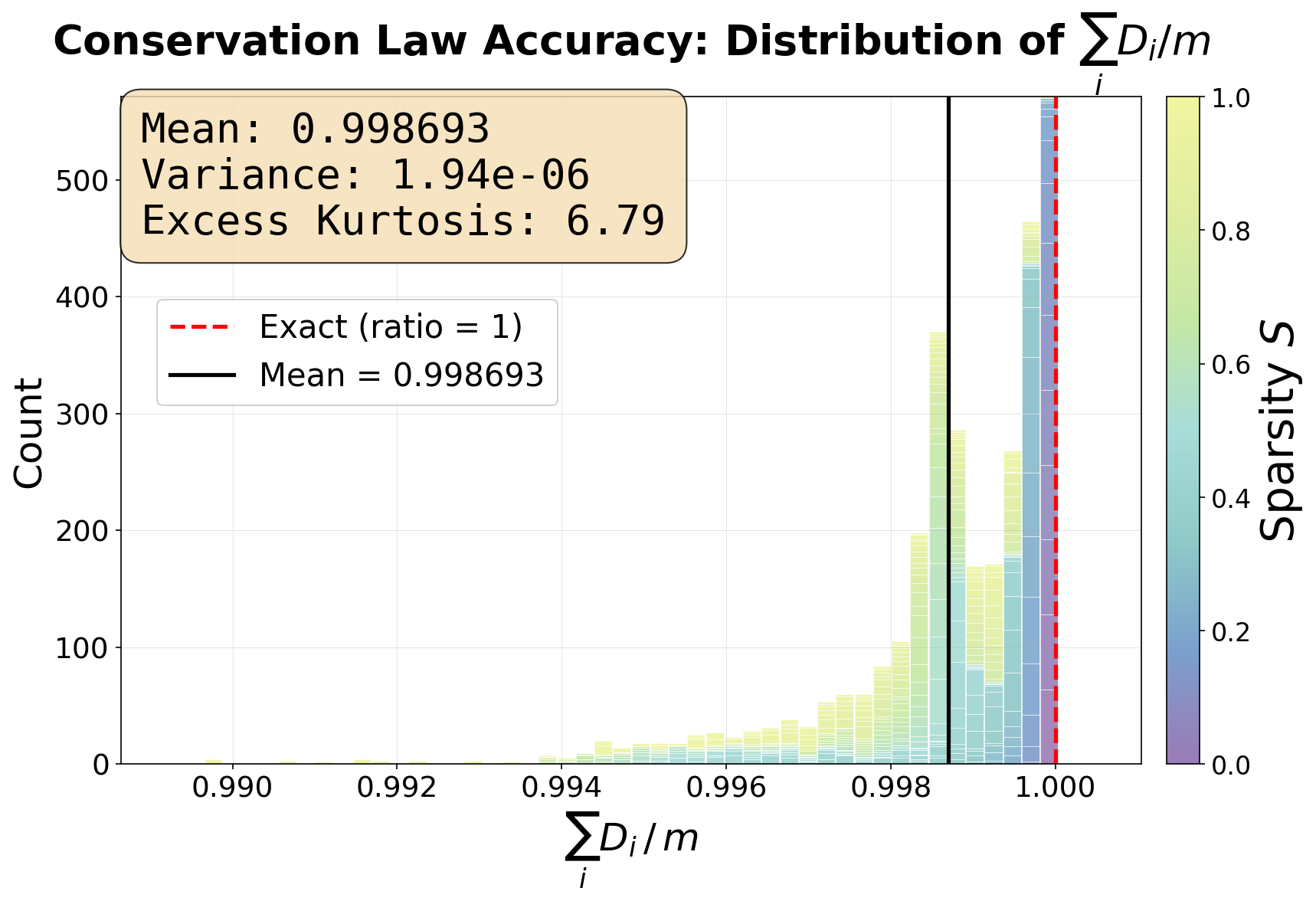}
\vspace{-10pt}
\caption{Capacity Saturation across 3,200 runs}
\vspace{-10pt}
\end{figure}

This is the most important fact that we use to justify our theoretical analysis and later predictions. It provides a significant reduction on the complexity of training dynamics, namely by forcing all of the individual features to localize to exactly one \textbf{eigenspace} of the frame operator, characterized by the corresponding projector $P_k$:
\begin{lemmaproofbox}
\begin{theorem}[Spectral Localization]\label{theorem:spectral_localization}
    Assume the model saturates the fractional dimensionality capacity bound, i.e. $\sum_{i=1}^n D_i = \text{rank}(W) = m$. Then, for every feature $i$, the spectral measure collapses to a single Dirac mass $\mu_i = \delta_{\lambda_k}$, centered at some eigenvalue $\lambda_k>0$. This is equivalent to $FW_i = \lambda_kW_i$.
\end{theorem}
\begin{proof}
Define the leverage scores $\ell_i:=W_i^\top F^+W_i$, so $\sum_i \ell_i=\mathrm{tr}(W^\top F^+W)=\mathrm{tr}(F^+F)=\mathrm{rank}(W)$.
We can apply Lemma~\ref{lemma:cyclic_cs} to each feature $W_i\in\mathrm{Im}(F)$ for the Cauchy--Schwarz (CW) bound with $a:=F^{1/2}x$ and $b:=F^{+/2}x$:
\[
\|W_i\|^4 \le (W_i^\top F W_i)(W_i^\top F^+W_i),
\]
After rearrangement, we get the per-feature upper bound on fractional dimensionality $D_i=\|W_i\|^4/(W_i^\top F W_i)\le \ell_i$. By summing over all features we get the bound $\sum_i D_i \le \sum_i \ell_i=\mathrm{rank}(W)$. This makes the interpretation of $\sum_i D_i=\mathrm{rank}(W)=m$ one of saturating all of the fractional dimensionality bounds, i.e. that $D_i = \ell_i$. By Lemma~\ref{lemma:cyclic_isometry} and the finite-atomic nature of $\mu_i$ (formalized in Lemma~\ref{lemma:l2_atomic}) \(\kappa_i=\frac{W_i^\top F W_i}{\|W_i\|^2}=\mathbb{E}_{\mu_i}[\lambda]\) and 
\(\frac{\ell_i}{\|W_i\|^2}=\frac{W_i^\top F^+W_i}{\|W_i\|^2}=\mathbb{E}_{\mu_i}[\lambda^{+}].\) Thus $D_i=\|W_i\|^2/\mathbb{E}_{\mu_i}[\lambda^+]$. Even though we have now re-interpreted both $D_i$ and $\ell_i$  through the spectral measure, the former CW is in $(\mathbb{R}^m, \langle \cdot, \cdot \rangle)$, whereas the latter lives in  $L^2(\mu_i)$ (Lemma \ref{lemma:l2_atomic}) which is feature-dependent. The fundamental mismatch occurs because feature vectors $W_i$ live in $\mathbb{R}^m$, whereas the quantities controlling capacity and interference are spectral averages over eigenvalues of $F$. This is further validation for the need to consider $W$ as an intertwining operator. The way to bridge the two is to define the cyclic subspace \(\mathcal{K}_x:=\mathrm{span}\{x,Fx,F^2x,\dots\},\)
which is the smallest $F$-invariant slice of $\mathbb{R}^m$ space that captures all quadratic forms $W_i^\top h(F)W_i$ and vectors $F^{1/2}W_i, F^{+/2}W_i$. Per the cyclic-space isometry $U_{W_i}: L^2(\mu_i)\rightarrow \mathcal{K}_{W_i}$ (Lemma \ref{lemma:cyclic_isometry}), the functions $u(\lambda) = \sqrt{\lambda}$ and $v(\lambda) = \lambda^{+/2}$ correspond exactly to the vectors $F^{1/2}W_i/\|W_i\|$ and $F^{+/2}W_i/\|W_i\|$. Hence, CW in $L^2(\mu_i)$ is literally the Euclidean CW on $\mathcal{K}_{W_i}\subseteq\mathbb{R}^m$, making the equality $D_i=\ell_i$ is equivalent to $\mathbb{E}_{\mu_i}[\lambda]\mathbb{E}_{\mu_i}[\lambda^{+}]=1$, since $\langle u(\lambda), v(\lambda)\rangle_{L^2(\mu_i)} = 1$ (Lemma \ref{lemma:cs_L2}). The equality holds if and only if $\sqrt \lambda = c\lambda^{+/2}$ $\mu_i$-a.s., i.e. if $\lambda$ is $\mu_i$-a.s. constant. Therefore, $\mu_i$ collapses to a Dirac mass $\delta_{\lambda(i)}$, meaning $W_i$ lies entirely in a single positive-eigenvalue eigenspace of $F$,
equivalently $FW_i=\lambda(i)W_i$. Q.E.D.
\end{proof}
\end{lemmaproofbox}
This result is in fact quite strong. It provides us with an immediate handle on global geometry of features. The first corollary (Pf. in \ref{corollary:projective_linearity}) is that if Spectral Localization holds, the fractional dimensionality scales linearly with the norm, where the slope corresponds to the inverse eigenvalue $\lambda_k$:
\begin{corollary} [Projective Linearity]
    \label{corollary:projective_linearity_section4}
    Assume Spectral Localization holds. Then, the fractional dimensionality $D_i$ of any feature is linearly determined by its feature norm $\propto \|W_i\|^2$ with slope $k$ equal to the reciprocal of the eigenvalue $\lambda_e$ of the subspace it occupies $D_i = \lambda_e^+ \|W_i\|^2$
\end{corollary}
Since the model saturates its capacity, and hence its rank (as proved above), this means that with this information we can recover the exact density (and hence fractional dimensionality) of features living in that eigenspace. To investigate the validity of the Projective Linearity Corollary across the entire training sweep, we analyzed the geometric stability of feature clusters from 3,200 experimental runs (3,276,800 total features). For each feature $w_i$, we computed its spectral measure weights $p_{i,e} = \|P_e w_i\|^2 / \|w_i\|^2$. We quantified the degree of Spectral Localization for each cluster $C$ by the mean maximum spectral mass of its constituent features, defined as $\mathcal{L}_C = \frac{1}{|C|} \sum_{i \in C} \max_e (p_{i,e})$. We plot against the linear coefficient of determination $R^2$. The error bars correspond to the absolute error $|k\lambda -1|$ for each cluster:
\begin{figure}[h]
    \includegraphics[width=0.47\textwidth]
    {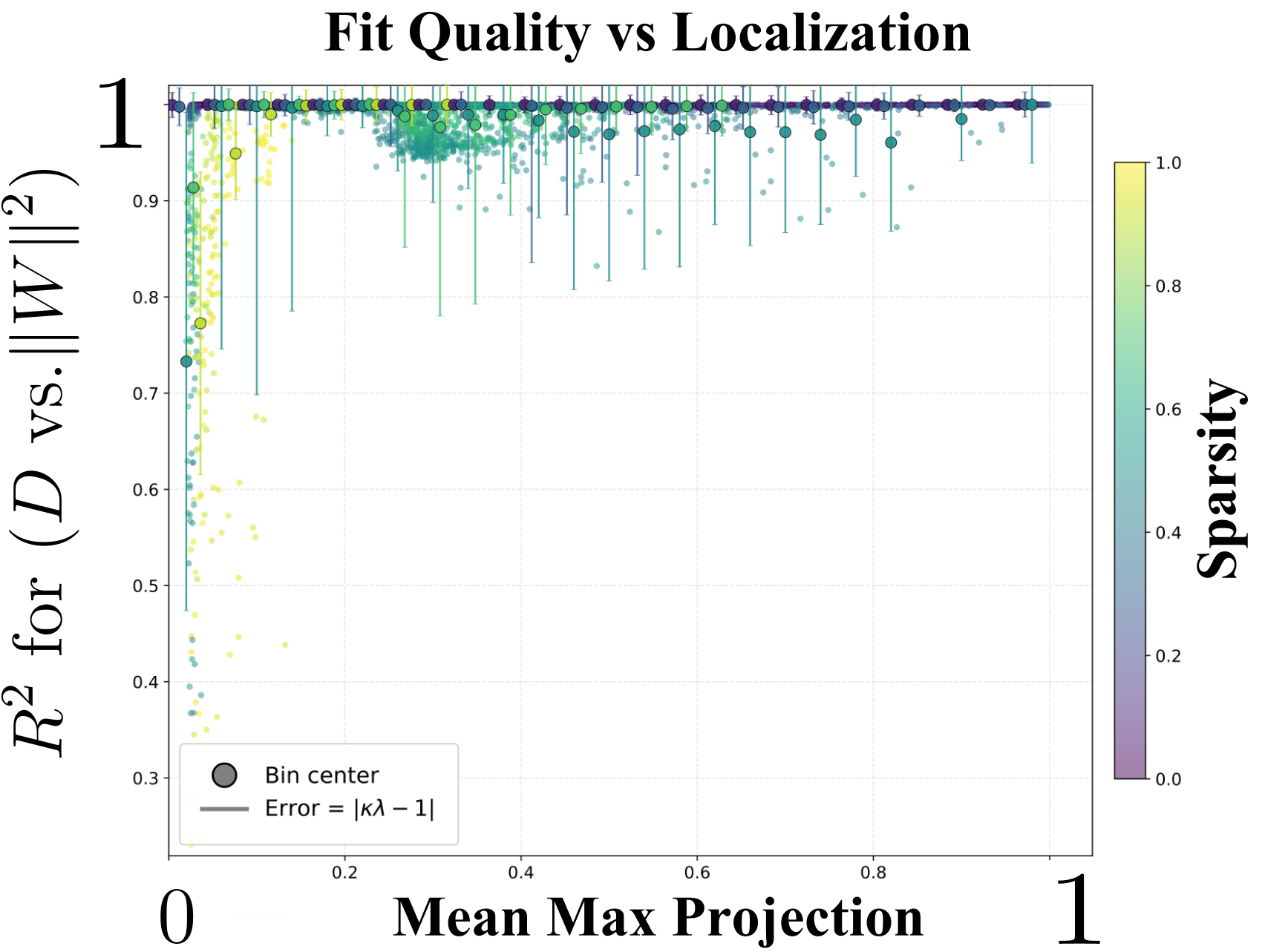}
    \caption{Projective Linearity test: y-axis represents the coefficient of determination $R^2$ of the  $\lambda_k^{-1}$ linear fit; error bars provide the absolute error $|k\lambda-1|$; $x$-axis represents the mean max metric for Spectral Measure localization}
\end{figure}

The added gradient coloring on the error bars reveals a strong relationship between the eigenvalue slope fit and feature sparsity - lower sparsity forces features into tighter configurations (and hence eigenspaces) due to the stronger interference, whereas high sparsity features are more diffuse across the spectrum. This observation aligns with our perturbative analysis for small deviations $\sigma_i = 1-D_i/\ell_i$ in the capacity bound (Appendix \ref{section:perturbative_analysis}). This is also the explanation behind the two peaks in the conservation law, clearly stratified across sparsity. While slack is minimal across all runs, the slight deviation in the high sparsity regime results in broad delocalization of a few features. The spectral localization also has a direct consequence for the type of geometry the model learns (Pf. \ref{theorem:decomp_tight_frame}): 


\begin{theorem}[Decomposition into Tight Frames]\label{theorem:decomp}
Assume Spectral Localization holds for all features $i\in \{1, \cdots,n\} $. Let $\Lambda = \{\lambda_1, \cdots, \lambda_k\}$ be the set of distinct eigenvalues of $F= WW^\top$. Then, the feature set partitions into disjoint clusters $C_k = \{i | \lambda(i) = \lambda_k\}$ and $F = \bigoplus_{k=1}^m \lambda_k P_{V_k}$, where $V_k = \operatorname{span}\{W_i| i\in C_k\}$. Furthermore, within each cluster $C_k$, the sub-matrix of weights $W_{C_k}$ forms a Tight Frame: $\sum_{i\in C_k} W_iW_i^\top  = \lambda_k I_k$ on $V_k$ with constant $\lambda_k$.
    
\end{theorem}
This result provides a complete description of the geometries learned in TMS \cite{toymodels}, including all the simplices, polygons and square antiprism. It is important to note that for any shape other than the simplex, the tight frame characterization is not unique. Nonetheless, we can use the Spectral Correspondence (Lemma \ref{lemma:spectral_correspondence}) to provide a method of classification of the exact geometry by its spectral signature in terms of association schemes (Pf. in \ref{theorem:spectral-identification}): 

\begin{theorem}[Spectral Identification of Geometry]
    Assume Spectral Localization holds, s.t per Theorem \ref{theorem:decomp} a feature cluster $C_k$ forms a Tight Frame. Let $M_k = W_{C_k}^\top W_{C_k}$ be the local Gram matrix. Then, the geometry of cluster $C_k$ is an instance of an association scheme $\mathcal{A}$ if and only if the eigenspaces of the Gram matrix $M_k$ coincide with the canonical stratums $\{W_0, W_1, \cdots, W_s\}$ of the algebra $\mathcal{A}$. 
\end{theorem}
This theorem has a direct corollary for the Simplex Algebra classification \ref{appendix:simplex}, namely that if the tight frame is made up of $d+1$ features in $d$ dimensions $\big(\frac{d+1}{d}I_d\big)$, then it necessarily must be a simplex geometry (Pf. in \ref{corollary:simplex}):

\begin{corollary} [Simplex Identification]
The geometry of a cluster of features $C_k$ is a simplex if and only if the corresponding Gram matrix $M_{k}$ has exactly two eigenspaces $W_0 = \ker(M_{k}) = \operatorname{span}(\mathbf{1})$, $W_1 = \operatorname{Im}(M_k) = \mathbf{1}^\perp$ with 
$\Lambda(W_0)=0$ and  $\Lambda(W_1) = \lambda_k = \frac{|C_k|}{\dim(V_k)}$
\end{corollary}

We provide an analysis of why these features persist during training in \ref{appendix:gradient_flow}, hypothesizing that they remain stable under gradient descent.

\section{Limitations and Conclusion}

Interpretability seeks to understand neural networks in terms of their internal structure. Current methods focus on identifying features—sparse, interpretable directions in activation space. This paper argues that the geometry relating features is equally important, and that this geometry can be recovered using classical mathematical tools. By studying the eigenspaces of weight-induced operators, rather than individual feature vectors, we preserve geometric relationships between features, including shared subspaces, symmetry, and structured interference -- all of which are discarded by sparse or linear analyses. The framework converts questions about relationships (pairwise) into questions about structure (global). Instead of asking ``how does feature i relate to feature j," you ask ``what is the geometry of the feature space, and where do i and j sit in it?" This is the difference between knowing the pairwise distances between cities versus having a map.

The map presented by this paper is still limited. For the spectral measure to completely characterize interference, models need to exhibit spectral localization (suggested by HT-SR phenomenology, but not yet directly measured). To recover the geometry fully, we require capacity saturation, which we empirically verify only in toy models. And our toy-model specific results are particularly clean primarily due to the gradient is mediated only through the gram matrix because the toy model is an autoencoder.

Despite these limitations, the weight-induced operators, rather than individual feature vectors, appear to be the appropriate units for analyzing superposition. We take this to be an argument for using operator theory to study why models behave the way they do.

Operator theory has seen widespread success in studying complex physical systems. The spectral theorem's role in quantum mechanics, decomposing states into eigenbases of observables, is analogous to our spectral localization result, which decomposes features into eigenspaces of the frame operator. Representation theory's classification of symmetry actions parallels our use of association schemes to discretely classify feature geometries. These are not mere analogies: the mathematical structures are identical, suggesting that interpretability may benefit from the same tools that made complex physical systems tractable.

More fundamentally, however, by looking at how different spaces are intertwined, operator theory is designed to provide maps instead of mere pairwise comparisons. By analyzing the components of models in operator theoretic terms, we hope we can move from the study of mere features to that of model behavior.

\section{Impact statement}
This paper proposes a novel theoretical framework for analyzing superposition in neural networks using spectral tools. The primary contribution is methodological: it introduces basis-invariant diagnostics for classifying feature geometry and interference in learned representations. The work is foundational and does not introduce new model capabilities, datasets, or applications. The results are intended to support research in interpretability and representation analysis. While improved understanding of internal representations may indirectly inform downstream methods, this work itself poses minimal risk and has no direct societal impact.

\bibliography{bibliography}
\bibliographystyle{plainnat}

\newpage
\appendix
\onecolumn

\section*{\LARGE Appendix}
\addcontentsline{toc}{section}{Appendix}

\vspace{1em}

\makeatletter
\renewcommand{\l@section}{\@dottedtocline{1}{0em}{2.3em}}
\renewcommand{\l@subsection}{\@dottedtocline{2}{2.3em}{3.2em}}
\renewcommand{\l@subsubsection}{\@dottedtocline{3}{5.5em}{4.1em}}
\makeatother

\contentsline{section}{\numberline{A}Feature Geometry}{\pageref{appendix:feature_geometry}}{}
\contentsline{subsection}{\numberline{A.1}Orbit Algebra Derivation}{\pageref{appendix:orbit}}{}
\contentsline{subsection}{\numberline{A.2}Association Schemes}{\pageref{appendix:association_schemes}}{}
\contentsline{subsection}{\numberline{A.3}Algebra Classification}{\pageref{appendix:algebra_classification}}{}
\contentsline{subsubsection}{\numberline{A.3.1}Simplex algebra}{\pageref{appendix:simplex}}{}
\contentsline{subsection}{\numberline{A.4}Tight Frame}{\pageref{appendix:tight_frame}}{}

\contentsline{section}{\numberline{B}Spectral Theory}{\pageref{appendix:spectral_theory}}{}
\contentsline{subsection}{\numberline{B.1}Spectral Measure}{\pageref{lemma:spectral_measure_probability_distribution}}{}
\contentsline{subsection}{\numberline{B.2}Hilbert Spaces}{\pageref{lemma:cyclic_cs}}{}

\contentsline{subsection}{\numberline{B.6}Perturbation Theory}{\pageref{appendix:perturbation_theory}}{}
\contentsline{subsection}{\numberline{B.7}First-order drift of Gram eigenprojections}{\pageref{lemma:kato-eigenvalue-coeff}}{}
\contentsline{subsection}{\numberline{B.8}Moments of the Spectral Measure}{\pageref{lemma:moments}}{}
\contentsline{subsection}{\numberline{B.9}Cyclic Space and Pseudoinverse Cauchy--Schwarz}{\pageref{lemma:cyclic_cs}}{}
\contentsline{subsection}{\numberline{B.10}\texorpdfstring{$L^2(\mu)$}{L2(mu)} for Finite Atomic Measures}{\pageref{lemma:l2_atomic}}{}
\contentsline{subsection}{\numberline{B.11}Cauchy--Schwarz in \texorpdfstring{$L^2(\mu)$}{L2(mu)}}{\pageref{lemma:cs_L2}}{}
\contentsline{subsection}{\numberline{B.12}Cyclic Isometry \texorpdfstring{$L^2(\mu_x)\simeq\mathcal{K}_x$}{L2(mu_x) ≃ K_x}}{\pageref{lemma:cyclic_isometry}}{}

\contentsline{subsection}{\numberline{B.13}Theorem Proofs and Corollaries}{\pageref{corollary:projective_linearity}}{}
\contentsline{subsection}{\numberline{B.14}Perturbative Case}{\pageref{section:perturbative_analysis}}{}

\contentsline{section}{\numberline{C}Localization vs Delocalization}{\pageref{appendix:localization_delocalization}}{}
\contentsline{subsection}{\numberline{C.1}HT-SR}{\pageref{appendix:martin_and_mahoney}}{}
\contentsline{subsection}{\numberline{C.2}Anderson Localization}{\pageref{sec:anderson_localization}}{}

\contentsline{section}{\numberline{D}Gradient Flow}{\pageref{appendix:gradient_flow}}{}
\contentsline{subsection}{\numberline{D.1}Explicit Form}{\pageref{lemma:gradient-kernel}}{}
\contentsline{subsection}{\numberline{D.2}Association Scheme Reduction}{\pageref{appendix:association_scheme}}{}

\contentsline{section}{\numberline{E}Non-uniform Sparsity}{\pageref{appendix:non_uniform_sparsity}}{}

\contentsline{section}{\numberline{F}Related Work}{\pageref{appendix:related_work}}{}

\vspace{2em}
\newpage

\section{Feature Geometry}
\label{appendix:feature_geometry}

\subsection{Orbit Algebra Derivation}
\label{appendix:orbit}
Let us fix a feature cluster $C$ of size $|C|=p$. We can index the set of vertices $V = \{1, \cdots, p\}$. Per the polytope geometry, we can assume the existence of an abstract symmetry group $\Gamma \subseteq S_p$ (Cayley's theorem) that acts transitively on $V$, i.e $\forall i, j\in V, \exists \gamma\in\Gamma, $ s.t. $\gamma\cdot i = j$. In simple terms, this means that every vertex can be reached from any other vertex as a starting point. This can be said quite succinctly using the notion of an "orbit" of a given vertex $i$. The latter is defined as $\Gamma \cdot i = \{\gamma \cdot i \ |\ \gamma\in \Gamma\}$, namely, the set of all vertices, accessible by the elements $\gamma$ with $i$ as a starting point. Then, transitivity is equivalent to the group having one orbit. To return back to the domain of linear algebra, we can define $P_\gamma\in R^{p\times p}$ to be the $\gamma$ permutation matrix representation: 
\begin{equation}
(P_\gamma)_{ij} = \delta_{i, \gamma(j)} = \left\{\begin{matrix}
1 & \text{if }i=\gamma(j)\\ 0 &\text{otherwise}
\end{matrix}\right.\end{equation}
We have the following neat fact about permutation matrices: 
\begin{lemmaproofbox}
\begin{lemma}[Orthogonality of Permutation Matrix] For every $\gamma\in \Gamma$, $P_\gamma^{-1} = P_\gamma^\top$
\end{lemma}\label{lemma:A.1}
\begin{proof} By definition $(5)$, we have $(P_\gamma^\top)_{ij} = (P_\gamma)_{ji} = \delta_{j, \gamma(i)}$, hence:
$$(P_\gamma P_\gamma^\top)_{ij} = \sum_{k=1}^p (P_\gamma)_{ik} (P_\gamma^\top)_{kj} = $$
$$\sum_{k=1}^p \delta_{i, \gamma(k)}\delta _{j, \gamma(k)} = \delta_{ij}  = (I_p)_{ij},$$
where $I_p\in \mathbb{R}^{p\times p}$ is the identity matrix. WLOG $P_\gamma^\top P_\gamma$ yields the same result. \end{proof}
\end{lemmaproofbox}
\noindent This in fact leads to a direct consequence that $$P_\Gamma = \{P_\gamma\in \mathbb{R}^{p\times p} : \gamma\in \Gamma\}\subseteq O(p),$$
where $O(p)$ is the $p$-dimensional orthogonal group, defined as the group of $p\times p$ orthogonal matrices. This is the over-arching "parent" symmetry that contains all of the polytopal geometries symmetries:
$$P_{D_n}\subset O(2)\subset O(3)\subset\cdots \subset O(n-1)\subset O(n)$$
This matrix formalism is also what allows us to rigorously define "stability" of cluster configurations using the centralizer (commutant): 
\begin{equation}
\mathcal{A}_\Gamma : = \{A\in \mathbb{R}^{p\times p}: P_\gamma A P_\gamma^\top  = A, \forall \gamma\in \Gamma\}
\end{equation}
This is the set of all $\mathbb{R}^{p\times p}$ matrices that are invariant when conjugated by any permutation $P_\gamma$. In fact, more than that, it is an algebra. To see why, we need to shift our attention to the the set of ordered pairs $V\times V$. It has a canonical group action of $\Gamma$ acting on it defined as:
$$\pi: \Gamma\times (V\times V)\rightarrow V\times V: \quad \gamma \cdot (i, j) = (\gamma(i), \gamma(j)),$$
which naturally partitions the set of all $|V\times V| = p^2$ pairs into disjoint orbits $\mathcal{O}_1, \cdots, \mathcal{O}_R$, where $R+1=r(\Gamma)$ is the rank of the permutation group (Lagrange's theorem). We can define an orbital matrix $A_r$ for every orbit $\mathcal{O}_r$ as follows:

\begin{equation}
    (A_r)_{ij} = \left\{\begin{matrix}1 & \text{if }(i, j)\in \mathcal{O}_r\\0& \text{otherwise}\end{matrix}\right.
\end{equation}
Since our group $\Gamma$ is transitive on $V$, the diagonal $\mathcal{O}_0=\{(i,i)\}$ is canonically an orbital  with the identity matrix $I_p$ as its representation. 
\begin{lemmaproofbox}
\begin{lemma}[Centralizer Membership]  A matrix $X$ is in $\mathcal{A}_\Gamma$ if and only if its entries are constant on the orbits of $\Gamma$ on $V\times V$. 
\end{lemma}\label{lemma:A.2}
\begin{proof} 
$(\Rightarrow)$Let $\gamma\in \Gamma$ and $X\in \mathcal{A}_\Gamma$ be arbitrary. By definition $(6)$:   $$P_\gamma X P_\gamma^\top =X\quad \Leftrightarrow\quad  P_\gamma X = XP_\gamma \quad(\text{Lemma }4.3)$$
By definition $(5)$, for all $i, j$ we have that the LHS is:
$$ \sum_{k=1}^p (P_\gamma)_{ik}X_{kj} = \sum_{k=1}^p\delta_{i,\gamma(k)} X_{kj}= X_{\gamma^{-1}(i), j}$$
And similarly for the RHS:
$$\sum_{k=1}^pX_{ik} (P_\gamma)_{kj} = \sum_{k=1}^p X_{ik} \delta_{k, \gamma(j)} = X_{i, \gamma(j)}$$
Hence, $X_{i, \gamma(j)} = X_{\gamma^{-1}(i), j}$. Since this is true for all $i,j$, we can set $i = \gamma^{-1}(k)$, i.e $$X_{\gamma^{-1}(\gamma(k)), j} = X_{k, j} = X_{\gamma(k), \gamma(j)}$$
And since $\gamma$ was arbitrary, this completes the forward direction. The opposite direction $(\Leftarrow)$ follows WLOG by taking arbitrary $i,j$ entries of the LHS and RHS above. 
\end{proof}
\end{lemmaproofbox}

\begin{lemmaproofbox}
    \begin{lemma}[Basis of the centralizer]\label{lemma:basis}
    $\{A_0, \cdots, A_R\}$ is the basis of $\mathcal{A}_\Gamma$
        
    \end{lemma}
    \begin{proof} The preceding lemma directly implies that $\forall_{r\in \{0, \cdots, R\}}A_r\subseteq \mathcal{A}_\Gamma$, since if $(i, j)\in \mathcal{O}_r$, by definition of the orbit: $$\forall{\gamma\in \Gamma}, \quad \gamma\cdot (i, j) = (\gamma(i), \gamma(j))\in \mathcal{O}_r,$$ implying that $(A_r)_{i, j} = 1 = (A_r)_{\gamma(i), \gamma(j)}$ and similarly $\forall (i, j)\not\in \mathcal{O}_r$, $(A_r)_{i, j} = 0 = (A_r)_{\gamma(i), \gamma(j)}$. This in fact extends to $\text{span}\{A_0, \cdots, A_R\}\subseteq \mathcal{A}_\Gamma$, since the orbits are disjoint: 
    \begin{equation} X = c_0A_0+\cdots cA_R\Leftrightarrow X_{ij} = c_r=X_{\gamma(i), \gamma(j)},\end{equation}
where $(i,j)\in \mathcal{O}_r$. Substituting $X$ for an arbitrary $A_r$ in $(8)$ implies that  $\{A_0, \cdots, A_R\}$ are linearly independent. Furthermore, since any invariant matrix $X\in \mathcal{A}_\Gamma$ must be constant on these orbits, $(8)$ and the preceding Lemma imply that $X$ can be uniquely written as: $$X = \sum_{r=0}^R c_r A_r,$$
with which we have completed the other direction $\mathcal{A}_\Gamma\subseteq \text{span}(A_0, \cdots, A_R\}$ of the equality, concluding that $\{A_0, \cdots, A_R\}$ is in fact a basis. 
    \end{proof}
\end{lemmaproofbox}
\noindent The last property we will need is the algebraic closure of the centralizer under matrix multiplication: 
\begin{lemmaproofbox}
    
\begin{lemma}[Algebraic Closure] If $X, Y\in \mathcal{A}_\Gamma$, then $XY\in \mathcal{A}_\Gamma$

\begin{proof}
    Let $X, Y\in \mathcal{A}_\Gamma$. By the definition of membership, we have $P_\Gamma X = XP_\Gamma$ and $P_\gamma Y = Y P_\gamma$. By multiplying $XY$ by $P_\gamma$, we get: $$P_\gamma(XY) = (P_\gamma X)Y = (XP_\gamma)Y$$
    $$= X(P_\gamma Y) = X(YP_\gamma) = (XY)P_\gamma$$
	  Hence $XY\in \mathcal{A}_\Gamma$ . 
      \end{proof}
\end{lemma}
\end{lemmaproofbox}
\noindent We can apply the latter two lemmas to the product of any two orbital matrices $$(A_rA_s)_{ij} = \sum_k (A_r)_{ik}(A_s)_{kj},$$ the individual entries of which can be interpreted as the number of "paths" of length $2$ from $i$ to $j$, where the first step is of orbital type $r$ and the second is orbital type $s$:

\begin{equation} A_rA_s\in \text{span}\{A_0,\cdots A_R\} \Rightarrow A_rA_s = \sum_{u=0}^R c_{rs}^uA_u\end{equation}
\begin{equation}c_{rs}^u = |\{k\in V| (i, k)\in \mathcal{O}_r \text{ and }(k, j)\in \mathcal{O}_s\}|\end{equation}
for any $(i, j)\in \mathcal{O}_u$. In fact, notice that if we chose $(i, j)\in \mathcal{O}_u$, per the invariance of Lemma A.4. this count must be independent of the specific choice of $(i, j)$ in that orbit. These structure constants are in fact known as the intersection numbers of the association scheme $\{A_0, \cdots, A_R\}$ \cite{rabailey}, which is how it is known in algebraic graph theory \cite{godsil-royle}. 
To explain the "sticky points" on the scatter plot described in toy models of superposition \cite{toymodels}, we need to show that once a matrix enters such a regime, it becomes "stable" with respect to group symmetry of the given polygon. 
\subsection{Association Schemes}
\label{appendix:association_schemes}
Every association scheme $\{A_0, \cdots A_s\}$ on a set $\Omega$ with $s$ associate classes makes up a Bose-Mesner algebra:
$$\mathcal{A} = \left\{\sum_{i=0}^s \mu_i A_i: \mu_0, \cdots, \mu_s\in \mathbb{R}\right\}$$
We have the following neat spectral theory, stated as Theorem 2.6 in \cite{rabailey}:
\begin{lemmaproofbox}
    \begin{lemma}\label{lemma:2.6}
        Let $\{A_0, A_1, \cdots A_s\}$ be the adjacency matrices of an association scheme on $\Omega$ and let $\mathcal{A}$ be its Bose-Mesner algebra. Then $\mathbb{R}^\Omega$ has $s+1$ orthogonal subspaces $W_0, W_1, \dots W_s$ (strata) with orthogonal projectors, such that: 
        \begin{enumerate}[label=(\roman*)]
            \item $\mathbb{R}^\Omega = W_0\oplus W_1\oplus \cdots \oplus W_s$
            \item each of $W_0, W_1\dots, W_s$ is a sub-eigenspace of every matrix in $\mathcal{A}$;
            \item for $i=0, 1, \dots, s$, the adjacency matrix $A_i$ is a linear combination of $S_0, S_1, \cdots, S_s$ (note that $S_0 = \frac{1}{|\Omega|}J$)
            \item for $e=0, 1, \dots s$, the stratum projector $S_e$ is a linear combination of $A_0;  A_1, \cdots, A_s$
        \end{enumerate}
    \end{lemma}
\end{lemmaproofbox}
For $i\in \mathcal{K}$ (the index set of the association scheme) and $e$ (the index set of the stratum projectors), we define $C_{(i, e)}$ to be the eigenvalue of $A_i$ on $W_e$ and $D_{(i, e)}$ be the coefficients in the following expansion: $$A_i = \sum_{e\in\mathcal{E}}C_{(i, e)}S_e, \quad S_e = \sum_{i\in \mathcal{K}} D_{(e, i)} A_i$$
The following facts hold (Lemma 2.9, Theorem 2.12, Corollary 2.14 and Corollary 2.15 from \cite{rabailey}):
\begin{lemmaproofbox}
    \begin{lemma} \label{lemma:characters}The matrices $C \in \mathbb{R}^{\mathcal{K\times E}}$ and:
    \begin{enumerate}[label=(\roman*)]
        \item $D\in\mathbb{R}^{\mathcal{E\times K}} $ are mutual inverses and: $$C_{(0, e)} = 1, \quad C_{(i, 0)} = a_i , \quad D_{(0, i)} = a_i,  \quad D_{(e, 0)} = \frac{d_e}{|\Omega|},$$
        where $a_i = p_{ii}^0$ is the valency of the $i$-th associate class and $d_e = \text{tr}(S_e)$;
        \item $$\sum_{e}C_{(i, e)}C_{(j, e)}d_e = \delta_{ij} a_i|\Omega|, $$ where $\delta_{ij}$ is the delta function, i.e. $\delta_{ij}=1$ if $i=j$ and $0$ otherwise. 
        \item $$\sum_{i\in \mathcal{K}}\frac{1}{a_i}C_{(i, e)}C_{(i, f)} = \delta_{ij} \frac{|\Omega|}{d_e}$$
        \item $$\sum_{i\in \mathcal{K}} D_{(e, i)}D_{(f, i)}a_i = \delta_{ef} \frac{d_e}{|\Omega|}$$
        \item $$\sum_{e}\frac{1}{d_e}D_{(e, i)}D_{(e, j)} = \frac{\delta_{ij}}{|\Omega|a_i}$$
        \end{enumerate}
    \end{lemma}
\end{lemmaproofbox}

\subsection{Algebra Classification}
\label{appendix:algebra_classification}
\subsubsection{Simplex algebra} \label{appendix:simplex}
A property of any simplex is its two-inner product, i.e for unit vectors $\{x_i\}_{i\in V}\subset \mathbb{R}^{p-1}$, we have $X_{ii}=\langle x_i, x_i\rangle = 1$ for all $i$ and $X_{ij}=\langle x_i, x_j\rangle = \frac{1}{1-p}$. We chose $\Gamma = S_p$ as the group. By the definition of $X$, its entry values depend only on whether its two indices are equal. And since $\gamma$ is a bijection we have that $\gamma^{-1}(i) = \gamma^{-1}(j)\Leftrightarrow i =j$, hence per the centralizer definition follows that: $$(P_\gamma MP_\gamma^\top)_{ij} = M_{\gamma^{-1}(i), \gamma^{-1}(j)} = \left\{\begin{matrix}
1 & i=j \\\frac{1}{1-p} & i\neq j
\end{matrix}\right. \quad = M_{ij}$$
This splits the group action on $V\times V$ into two orbitals, since outside the canonical diagonal $\mathcal{O}_0$. Define the other orbital as: $$\mathcal{O}_{off} = \{(i, j)\in V\times V: i\neq j\}$$
By the injectivity of $\gamma$, $i\neq j$ implies $\gamma(i)\neq \gamma(j)$, hence $\gamma(i, j)\in \mathcal{O}_{off}$. Now, take any $(i, j)$ and $(k, \ell)$ in $\mathcal{O}_{\text{off}}$; so $i \neq j$ and $k \neq \ell$. We can explicitly construct $\gamma \in S_p$ with $\gamma(i) = k$ and $\gamma(j) = \ell$. Define $\gamma$ on $\{i, j\}$ by $\gamma(i) = k, \gamma(j) = \ell$. Because $k \neq \ell$, this is injective on $\{i, j\}$. Extend $\gamma$ to a bijection on all of $V$ by mapping the remaining $p - 2$ elements of $V \setminus \{i, j\}$ bijectively onto $V \setminus \{k, \ell\}$ (this is possible because these sets have the same finite cardinality). This extension is a permutation $\gamma \in S_p$. Then $\gamma \cdot (i, j) = (k, \ell)$. Thus, there are exactly two orbitals, i.e. $r(\Gamma) = 2$ and $R=1$. Per \ref{lemma:basis} $\Gamma$-invariant matrix has form $$A = a I + c(J-I),$$
where $J$ is the matrix with all ones. 
With $A_0 = I, A_1 = J-I$, we can also compute the intersection numbers:
$$A_0A_0 = A_0,\quad  A_0A_1 = A_1, \quad A_1 A_0 = A_1$$
$$A_1^2 = (J-I)^2 = J^2-2J+I= pJ-2J+I $$
$$= (p-2)J+I=(p-2)A_1+(p-1)A_0$$
Hence, the only non-zero intersection numbers are: $$c_{00}^0 = 1, \quad c_{01}^1 = c_{10}^1 = 1, \quad c_{11}^0 = p-1, \quad c_{11}^1 = p-2$$

Using \ref{lemma:characters}(i) to say $$C_{(0,0)}=1,\quad C_{(0, 1)}=1, \quad C_{(1,0)}= a_1=c_{11}^0=p-1$$
we can explicitly calculate the projectors: $$A_0 = S_0 + S_1, \quad A_1 = (p-1)S_0 + C_{(1, 1)}S_1,$$
where since $S_0 = \frac{1}{p}J$ [\ref{lemma:2.6}(iii)], we have that $S_1 = I - \frac{1}{p}J$. Hence: $$J-I = \frac{p-1}{p}J +C_{(1,1)}\left[I-\frac{1}{p}J\right]$$
$$\frac{C_{(1,1)}+1}{p}J =(1+C_{(1,1)})I,$$ 
meaning that $C_{(1, 1)}=-1$. Using \ref{lemma:spectral_correspondence} we can transfer the primitive idempotents we found to the feature latent space $\mathbb{R}^m$
$$WS_0W^\top = \frac{1}{p}WJW^\top = \frac{1}{p}W\mathbf{1}\mathbf{1}^\top W^\top = \frac{1}{p}(W\mathbf{1})(W\mathbf{1})^\top = \frac{1}{p}\left(\sum_{i=1}^pW_i\right)\left(\sum_{i=1}^pW_i\right)^\top$$
Per the definition of the simplex Gram matrix we have $$\langle W_i, W_i\rangle = 1, \quad \langle W_i, W_j\rangle  = \frac{1}{1-p} = -\frac{1}{p-1} = \frac{1}{p-1}$$
Then, for arbitrary $k$: $$\langle W_k, \sum_{i=1}^pW_i\rangle  = \sum_{i=1}^p\langle W_k, W_i\rangle  = \langle W_k, W_k\rangle + \sum_{i\neq k} \langle W_k, W_i\rangle = 1 + (p-1)\frac{-1}{p-1}=0$$
This means in fact that $$\langle \sum_{k=1}^p W_k, \sum_{i=1}^pW_i\rangle = \sum_{k=1}^p\langle W_k, \sum_{i=1}^pW_i\rangle = 0\Leftrightarrow WS_0W^\top = 0$$
From the same definition $$M = I - \frac{1}{p-1}(J-I) = \frac{p}{p-1}I - \frac{1}{p-1}J = \frac{p}{p-1}\left(I-\frac{1}{p}J\right) = \frac{p}{p-1}S_1$$

\begin{lemmaproofbox}
    \begin{lemma}[Digon]\label{corollary:digon}
        Consider the setting from \ref{section:warm-up}. We can derive the exact association scheme of the digon in the same way we did for the triangle (and in fact the simplex above):
    \end{lemma}
    \begin{proof}
        Here, $\Gamma \cong S_3 \cong C_2$ is the cyclic group of two elements $\{e, s\}$, namely the identity $e$ and $s$ is an element of order $2$, meaning that $s^2=1$. This visually represents swapping around the two nodes of the digon (i.e. going from heads-tails to tails-heads). Per the fact that this is a simplex (i.e one can reach from one node to every other), we have the same number of orbits/associate classes $\mathcal{C}_0 =\operatorname{diag}(\Omega_D)$ and $\mathcal{C}_1 = \{(4, 5); (4, 5)\}$ with adjacency matrices $A_0' = I_2$ and  $A_1'=J_2-I_2$. Similarly, Lemma \ref{lemma:2.6} for $\mathcal{A}_D = \operatorname{span}(A_0', A_1')$ yields $\mathbb{R}^{\Omega_D} \cong \mathbb{R}^2 = U_0\oplus U_1$ and by Lemma \ref{lemma:characters} $S_0' = \frac{1}{2}I$, $S_1' = I-\frac{1}{2}J$ and:

\vspace{-14pt}

\begin{align*}
C' = \begin{bmatrix}1 & 1 \\ 1 & -1\end{bmatrix}, \quad D' =\frac{1}{2} \begin{bmatrix}1 & 1 \\ 1 & -1\end{bmatrix}    \\
M_D = I-(J-I) = 2I-J = 2S_1'
\end{align*}

    \end{proof}
\end{lemmaproofbox}

\subsection{Tight Frame}
\label{appendix:tight_frame}
\begin{lemmaproofbox}
\begin{lemma} Let $w_1, \cdots, w_p\in \mathbb{R}^d$ be unit vectors forming a tight frame: $$\sum_{j=1}^p w_j w_j^\top  = \frac{p}{d}I_d$$
Let $P_{ij} = \langle w_i, w_j\rangle$ be their $p\times p$ Gram matrix. Then, for every $i$:$$(P^2)_{ii} = \frac{p}{d}\Rightarrow D_i = \frac{1}{(P^2)_{ii}} = \frac{d}{p}$$
\end{lemma}
\begin{proof} $(P^2)_{ii} = \sum_j (w_i^\top w_j)^2 = w_i^\top \left(\sum_j w_j w_j^\top\right) w_i = w_i^\top \frac{p}{d}Iw_i = \frac{p}{d}$. By (L4.1)  $P_{ii}=1$.\end{proof}
\end{lemmaproofbox}

\noindent The latter result is what allows us to quantify the observed cluster values by the inverse fractional dimensions. More specifically, notice that per Lemma 4.2, we can categorize the geometries observed in toy models of superposition \cite{toymodels} expressing their corresponding fractional dimensionalities using the tight frame parameters $(p, d)$: 
\begin{itemize}
    \item $(2, 1)$ - Digon with $D_i = 1/2$
    \item $(3, 2)$ - Triangle with $D_i = 2/3$
    \item $(4, 3)$ - Tetrahedron with $D_i = 3/4$
    \item $(5, 2)$ - Pentagon with $D_i = 2/5$
    \item $(8, 3)$ - Square antiprism with $D_i= 3/8$
\end{itemize}
However, it is important to note that this alone is not sufficient to distinguish between the different types of symmetries. For example, the square antiprism and cube both have $(p, d) = (8, 3)$, but the former is the one that remains stable as is observed both in the original paper and the more general Thomson problem \cite{thomson}. 

\section{Spectral Theory}
\label{appendix:spectral_theory}
\begin{lemmaproofbox}
\begin{lemma}[Spectral decomposition of the Gram operator \citep{kato}, Thm.~2.10]\label{lemma:spectral-decomposition}
Let $W\in\mathbb{R}^{m\times n}$ and let $M := W^\top W\in\mathbb{R}^{n\times n}$ be the Gram matrix.
Then $M$ is symmetric positive semidefinite and admits an orthogonal spectral resolution
\[
M \;=\; \sum_{e\in\mathcal{E}} \lambda_e\, E_e ,
\]
where $\lambda_e\in\mathbb{R}_{\ge 0}$ are the distinct eigenvalues and $E_e\in\mathbb{R}^{n\times n}$ are the corresponding orthogonal spectral projectors. Concretely,
\[
E_e^\top = E_e,\qquad E_e^2 = E_e,\qquad E_eE_f = 0\ (e\neq f),\qquad \sum_{e\in\mathcal{E}} E_e = I_n,
\]
and
\[
ME_e = E_eM = \lambda_e E_e\qquad (\forall e\in\mathcal{E}).
\]
If $0$ is an eigenvalue, we denote it by $\lambda_0=0$ and $E_0$ is the orthogonal projector onto $\ker(M)$.
\end{lemma}
\end{lemmaproofbox}

\begin{lemmaproofbox}
    \begin{lemma}[Frame and Gram Kernels]\label{lemma:kernels} The following are direct consequences of the orthogonal decomposition of the image and the kernel:
 \begin{enumerate}[label=(\roman*)]
           \item $ \ker (WW^\top)= \ker (W^\top)$
           \item $\ker (W^\top W) = \ker (W)$
       \end{enumerate}
    \end{lemma}
    \begin{proof}
         $(\Rightarrow) \ker (W^\top) \subseteq \ker(WW^\top)$. 
Let $y\in \ker(W^\top)$ be arbitrary. Then, $$W^\top y = 0\Rightarrow WW^\top y = W0 = 0\Rightarrow y\in \ker (WW^\top)$$\\
$(\Leftarrow)$ Let $y\in \ker(WW^\top)$ be arbitrary. Then $WW^\top y = 0$. Since the RHS is $0$, we can multiply the left by $y^\top$, s..t, we get: $$y^\top WW^\top y = y^\top 0 = 0\Rightarrow (W^\top y)^\top W^\top x=||W^\top y\|= 0,$$
which by the definition of the norm yields $W^\top y = 0$, i.e. $y\in \ker(W^\top)$. With this we have completed the equality $\ker(W^\top) = \ker(WW^\top)$ from both sides. The proof for $(ii)$ is identical after switching $W^\top$ and $W$ around. 
    \end{proof}
\end{lemmaproofbox}

\begin{lemmaproofbox}
    \begin{lemma}[Spectral Correspondence]\label{lemma:spectral_correspondence} Let $M = W^TW = \sum_{e}\lambda_e E_e$ be the spectral decomposition of the Gram matrix. Define $P_e:={\lambda_e}^{+}WE_eW^\top \in \mathbb{R}^{m\times m}$ for each $e = \{e: \lambda_e>0\}$. Then:
        \begin{enumerate}[label=(\roman*)]
            \item $P_e$ symmetric $P_e^\top = P_e$ and idempotent $P_e^2 = P_e$ (i.e. orthogonal projector)
            \item $P_eP_f = 0$ for $e\neq f$
            \item The frame operator decomposes as $F=WW^\top = \sum_{e}\lambda_e P_e$. 
            \item $\text{Im}(P_e) = W(\text{Im}(E_e))$ 
         \end{enumerate}
    \end{lemma}
    \begin{proof}
        \begin{enumerate}[label=(\roman*)]
            \item Symmetry is immediate from symmetry of $E_e$: $$P_e^\top  = {\lambda_e}^{+} (WE_eW^\top)^\top  = {\lambda_e}^{+}WE_e^\top W^\top = {\lambda_e}^{+}W E_e W^\top = P_e$$
Using idempotence of $E_e$, $W^\top W = M$ and $ME_e = \lambda_e E_e$ we get: $$P_e^2 = \left({\lambda_e}^{+}WE_eW^\top\right)\left({\lambda_e}^{+}WE_eW^\top\right)=$$
$$=\big({\lambda_e}^{+}\big)^2WE_eME_eW^\top = {\lambda_e}^{+}WE_eW^\top = P_e $$
As such, since $P_e$ is symmetric and idempotent, it is an orthogonal projector onto $\text{Im}(P_e)$ \cite{kato}. 
    \item Computing: $$P_eP_f = \left({\lambda_e}^{+}WE_eW^\top\right)\left(\lambda_f^+WE_fW\top\right) = $$$$\lambda_e^+\lambda_f^+ WE_eME_fW^\top  = \lambda_e^+\lambda_f^+WE_e\left(\sum_g \lambda_g E_g\right)E_f = {\lambda_e}^{+}WE_eE_f=0$$
Since $E_eE_f=0$ unless $g=f$ . Since also each $P_e$ is symmetric, this implies that $\text{Im}(P_e)\perp \text{Im}(P_f)$ for $e\neq f$ 
\item Rewrite for $\lambda_e>0$, $\lambda_e P_e = WE_e W^\top$, and observe that by \ref{lemma:kernels}(ii) $WE_0=0$:
$$\sum_{e}\lambda_eP_e = \sum_{e} WE_e W^\top = W\left(\sum_{e}E_e\right)W^\top = W\left(I_n-E_0\right)W^\top=  WW^\top - WE_0W^\top = WW^\top$$

\item $(\Rightarrow)$ Let $y\in \mathbb{R}^m$ be arbitrary. Then $$P_ey = {\lambda_e}^{+}WE_e W^\top y$$
Let $u:= E_eW^\top y$. Then $u\in \text{Im}(E_e)$, and $$P_ey = {\lambda_e}^{+} Wu\in W(\text{Im}(E_e)), $$
hence $\text{Im}(P_e)\subseteq W(\text{Im}(E_e))$ , since $y$ was arbitrary
\newline $(\Leftarrow)$ Let $u\in \text{Im}(E_e)$ be arbitrary. Then $E_e u = u$ and $Mu = \sum_{e}E_eu=  \lambda_e u$. Evaluating $P_e$ on $Wu$, we have: $$P_e(Wu) = {\lambda_e}^{+}WE_eW^\top Wu = {\lambda_e}^{+}WE_eMu = $$
$$= {\lambda_e}^{+} WE_e(\lambda_eu) = W(E_uu) = Wu$$
With this we complete $W(\text{Im}(E_e))\subseteq \text{Im}(P_e)$, and as such the equality $\text{im}(P_e) = W(\text{Im}(E_e))$ 

        \end{enumerate}
    \end{proof}
\end{lemmaproofbox}
There is an immediate corollary of $(A.3)(iv)$, namely $\dim \text{Im}(P_e) = \dim \text{Im}(E_e)$. The proof is simply the observation that $W$ is injective since if $u\in \text{Im}(E_e)$ and $Wu=0$, then $Mu = W^\top W u = 0 = Mu = \lambda_e u$ and $\lambda_e>0$. 

\begin{lemmaproofbox}
    \begin{lemma}[Spectral Intertwining]\label{lemma:spectral_intertwining} Define $P_e:={\lambda_e}^{+}WE_eW^\top \in \mathbb{R}^{m\times m}$ for each $e = \{e: \lambda_e>0\}$, where $M = W^TW = \sum_{e}\lambda_e E_e$ is the spectral decomposition of the Gram operator. Then:
        \begin{enumerate}[label=(\roman*)]
            \item $P_eW = WE_e$ and $W^\top P_e = E_e W^\top$. 
            \item $FP_e = P_eF = \lambda_eP_e$. 
            \item $\mathbb{R}^m = \left(\bigoplus_{e}\text{Im}(P_e)\right)\oplus \text{ker}(W^\top)$
         \end{enumerate}
    \end{lemma}
    \begin{proof}
        \begin{enumerate}[label=(\roman*)]
            \item $$P_eW = {\lambda_e}^{+}WE_e W^\top W = {\lambda_e}^{+}WE_e M = {\lambda_e}^{+}W(\lambda_e E_e) = WE_e$$
            $$W^\top P_e =  {\lambda_e}^{+}W^\top WE_eW^\top = {\lambda_e}^{+}ME_eW^\top = E_eW^\top $$
            \item $$FP_e = WW^\top {\lambda_e}^{+}WE_eW^\top = {\lambda_e}^{+} W(W^\top W)E_e W^\top =$$
            $$= {\lambda_e}^{+}WME_eW^\top = {\lambda_e}^{+}W(\lambda_eE_e)W^\top = \lambda_eP_e$$
	   $$P_eF = {\lambda_e}^{+}WE_eW^\top WW^\top  = {\lambda_e}^{+}WE_eMW^\top  = \lambda_eP_e$$
            \item Let $M^+ = \sum_{e}{\lambda_e}^{+}E_e$ be the Moore-Penrose pseudoinverse of $M$. Then: $$\sum_{e}P_e = \sum_{e}{\lambda_e}^{+} WE_e W^\top = WM^+ W^\top$$
   and $WM^+W^\top$ is the orthogonal projector into $\text{Im}(W) = \ker(W^\top) ^\perp = \ker(WW^\top)^\perp$ by \ref{lemma:kernels}. 
         \end{enumerate}
    \end{proof}
\end{lemmaproofbox}

\begin{lemmaproofbox}
    \begin{corollary}[Triange-Digon Case]\label{corollary:triangle-digon}
        Let the setting be as in Section \ref{section:warm-up}. Then, $F = 2P_D+\frac{3}{2}P_\triangle$, where $P_D$ and $P_\triangle$ are projectors onto the subspaces along which the digon and triangle live, respectively. 
    \end{corollary}
    \begin{proof}
        Now, let us define $U_\triangle: \operatorname{span}\{W_i: \Omega_\triangle\}\subset \mathbb{R}^3$ and $U_D: \operatorname{span}\{W_i: i\in \Omega_D\}\subset\mathbb{R}^3$ To tie this back to spectral description $M_\triangle$ and $M_D$ we found (which are equivalent and as such we can generalize to an arbitrary $C$), let us define the injection operator from a cluster $C\in \{\triangle, D\} $ to the total concept space $E_C: \mathbb{R}^{\Omega_C}\rightarrow \mathbb{R}^\Omega$, where $ (E_C x)_i =\sum_{j\in \Omega_C} x_i \delta_{ij}$ (Fig \ref{fig:intertwining_operator}). 
\begin{center}
    \includegraphics[width=0.7\linewidth]{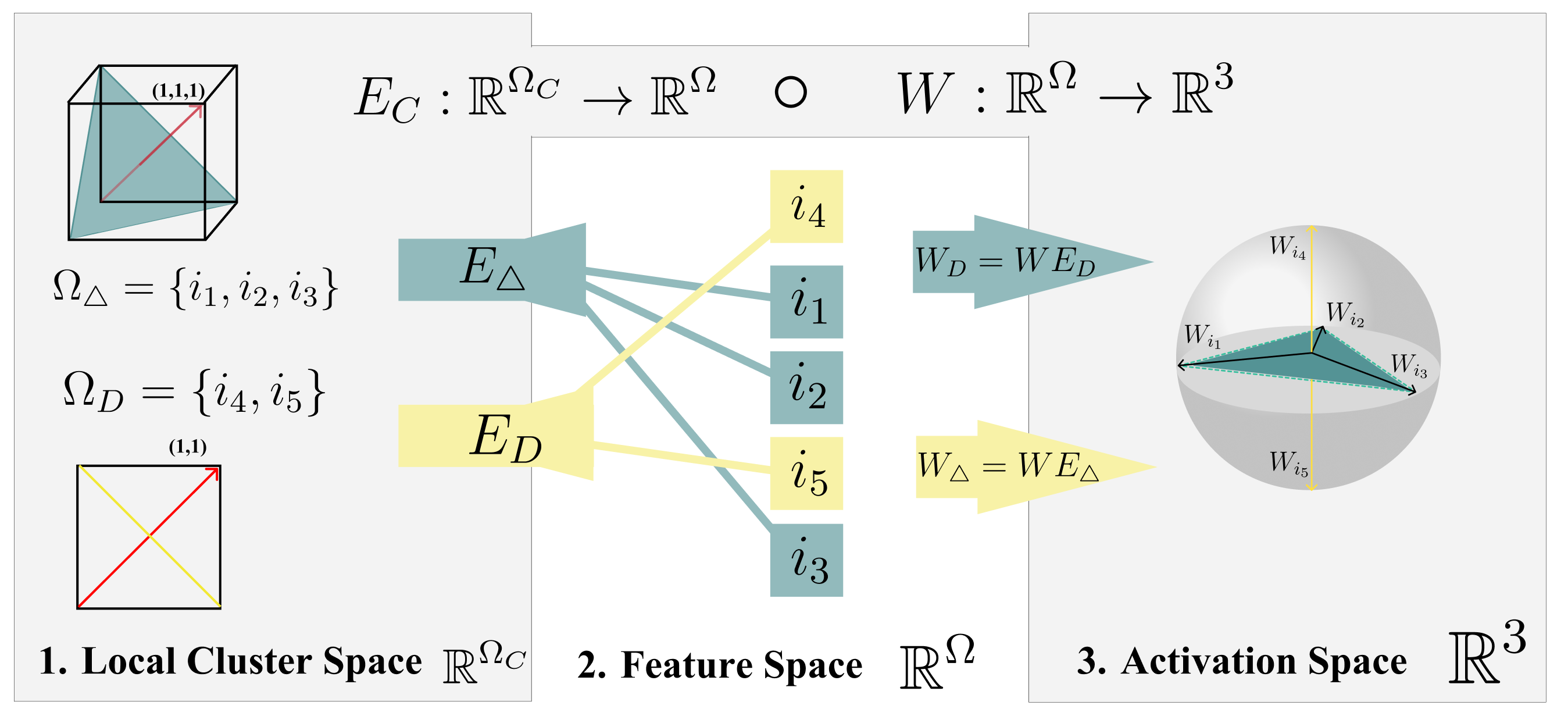}
    \captionof{figure}{$W$ is an Intertwining Operator of Spectral Modes}
    \label{fig:intertwining_operator}
\end{center}
Then, if $W_C :=WE_C\in \mathbb{R}^{3\times |\Omega_C|}$, we have that $M_C = W_C^\top W_C$ and $F_C = W_CW_C^\top = \sum_{i\in \Omega_C}W_iW_i^\top$ since summation over a set is permutation-stable. Furthermore, our spectral results get directly transferred: $M_C = \lambda_C S_1^{(C)}$, $S_0^{(C)} = \frac{1}{|\Omega_C|}\mathbf{11}^\top$, $S_1^{(C)} = I- S_0^{(C)}$. Since the centroid is being sent to the kernel means $S_0^{(C)}$ $M_C \mathbf{1} = \lambda_CS_1^{(C)}\mathbf{1} = 0$. There is a direct consequence for the injected matrices: multiplying by $\mathbf{1}^\top$ on the left yields:
\begin{equation*}0 = \mathbf{1}^\top M_C\mathbf{1} = \mathbf{1}^\top  W_C^\top W_C \mathbf{1} = \|W_c\mathbf{1}\|^2 = \left\|\sum_{i\in \Omega_C}W_C\right\|^2\end{equation*}
If \(M_C x = \lambda_C x\) with \(\lambda_C>0\), then
\(F_CW_C x = W_CM_C x = \lambda_C (W_C x).\) So \(W_C\) maps the \(\lambda_C\)-eigenspace of \(M_C\) into the \(\lambda_C\)-eigenspace of \(F_C\). Now since \(M_C = \lambda_C S^{(C)}_1\), the \(\lambda_C\)-eigenspace in concept space is \(\operatorname{Im}(S^{(C)}_1)\) (the “difference” subspace), and the (0)-eigenspace is \(\operatorname{Im}(S^{(C)}_0)=\operatorname{span}(\mathbf{1})\), which we just showed is killed by \(W_C\). Therefore the only nontrivial activation geometry produced by the cluster is \(U_C := W_C\big(\operatorname{Im}(S^{(C)}_1)\big)
= \operatorname{Im}(W_C)
\subseteq; \mathbb{R}^3,\) and \(F_C\) acts as \(\lambda_C\) on \(U_C\) and as (0) on \(U_C^\perp\). Equivalently,\(F_C = \lambda_C, P_C,\) where \(P_C\) is the orthogonal projector onto \(U_C\). Lastly, define \(P_C := \lambda_C^{-1} W_C S^{(C)}_1 W_C^\top.\)
Then \(P_C\) is a projector (idempotent and symmetric) by a direct computation using \(M_C S^{(C)}_1 = \lambda_C S^{(C)}_1\):
\begin{align*}
P_C^2
= \lambda_C^{-2} W_C S_1 W_C^\top W_C S_1 W_C^\top
= \lambda_C^{-2} W_C S_1 M_C S_1 W_C^\top\\
= \lambda_C^{-1} W_C S_1 W_C^\top
= P_C.
\end{align*}
Then clearly \(F_C = W_CW_C^\top = W_C(S_0+S_1)W_C^\top = W_CS_1W_C^\top = \lambda_C P_C\) because \(W_CS_0=0\) from centroid cancellation. With this, we have recovered $(2.2)$. Now, we can use all of the spectral information to give both the global and local geometric classification that would otherwise have been inaccessible from an arbitrarily scrambled Gram matrix. Since $F_C = \lambda_C = P_C$, we have:
$$\operatorname{tr}(F_C) = \lambda_C \operatorname{tr}(P_C) = \lambda_C \operatorname{rank}(P_C) = \lambda_C \dim(U_C)$$ 
But also, $\operatorname{tr}(F_C) = \sum_{i\in \Omega_C}\|W_i\|^2$, so for our unit-norm features, this is $|\Omega_C|$, i.e. $\frac{1}{\lambda_C} = \frac{\dim (U_C)}{\sum_{I\in \Omega_C}\|W_i\|^2}= \frac{\operatorname{rank}(P_C)}{|\Omega_C|}$ and as such, we recover our main diagram \ref{fig:triangle-digon}.
\end{proof}
\end{lemmaproofbox}

\subsection{Perturbation Theory}
\label{appendix:perturbation_theory}

\begin{lemmaproofbox}
\begin{lemma}[Analytic expansion of the Gram operator]\label{lemma:kato-analytic-expansion}
Let $x\mapsto M(x)\in\mathbb{R}^{n\times n}$ be analytic in a neighborhood of $x=0$, with $M(0)=M$.
Then there exist matrices $M^{(k)}\in\mathbb{R}^{n\times n}$ such that
\[
M(x) \;=\; M + xM^{(1)} + x^2 M^{(2)} + \cdots,
\qquad
M^{(k)} \;=\; \frac{1}{k!}\left.\frac{d^k}{dx^k}M(x)\right|_{x=0}.
\]
In our setting one may take $M(x)=W(x)^\top W(x)$ and $F(x)=W(x)W(x)^\top$.
When $M(\cdot)$ follows the Gram flow $\dot M = -\{M,\Phi\}$ and the derivative exists at the snapshot,
one may identify $M^{(1)}$ with $\dot M(0)=-\{M,\Phi\}$.
\end{lemma}
\end{lemmaproofbox}
\noindent\textit{Reference:} \citep{kato}, Ch.~II, \S 1, Eq.~(1.2).

\begin{lemmaproofbox}
\begin{lemma}[Reduced resolvent of an eigenvalue of $M$]\label{lemma:kato-reduced-resolvent}
Fix an eigenvalue $\lambda$ of $M$ and let $E$ be the corresponding spectral projector.
The reduced resolvent of $M$ at $\lambda$ is the operator $S_\lambda$ that inverts $M-\lambda I_n$
on the invariant complement $(I_n-E)\mathbb{R}^n$.
Equivalently, $S_\lambda$ is characterized by
\[
(M-\lambda I_n)S_\lambda = S_\lambda(M-\lambda I_n)=I_n-E,
\qquad
S_\lambda E = ES_\lambda = 0.
\]
(For symmetric $M=\sum_f \lambda_f E_f$, one may write explicitly
$S_\lambda = \sum_{f:\lambda_f\neq \lambda}(\lambda_f-\lambda)^{-1}E_f$.)
\end{lemma}
\end{lemmaproofbox}
\noindent\textit{Reference:} \citep{kato}, Ch.~I, \S 5.3, Eqs.~(5.26)--(5.29) (definition via holomorphic part of the resolvent and inverse-on-complement identities).

\begin{lemmaproofbox}
\begin{lemma}[First-order drift of Gram eigenprojections]\label{lemma:kato-eigenprojection-coeff}
Let $M(x)=M + xM^{(1)} + O(x^2)$ be an analytic perturbation, and let $E(x)$ be the spectral projector
corresponding to an isolated eigenvalue $\lambda$ of $M$ (equivalently, an isolated $\lambda$-group).
Let $E:=E(0)$ and let $S_\lambda$ be the reduced resolvent of $M$ at $\lambda$.
Then the first Taylor coefficient $E^{(1)} := \left.\frac{d}{dx}E(x)\right|_{x=0}$ satisfies
\[
E^{(1)} \;=\; -\,E\,M^{(1)}\,S_\lambda \;-\; S_\lambda\,M^{(1)}\,E.
\]
In particular, along the Gram flow $\dot M=-\{M,\Phi\}$, the instantaneous eigenprojection drift is obtained by substituting
$M^{(1)}=\dot M$.
\end{lemma}
\end{lemmaproofbox}
\noindent\textit{Reference:} \citep{kato}, Ch.~II, \S 2, Eq.~(2.14) (semisimple case; automatically satisfied for symmetric $M$).

\begin{lemmaproofbox}
\begin{lemma}[First-order drift of Gram eigenvalues]\label{lemma:kato-eigenvalue-coeff}
Let $M(x)=M + xM^{(1)} + O(x^2)$ be analytic and let $\lambda$ be an eigenvalue of $M$
with spectral projector $E$ of rank $d=\mathrm{tr}(E)$.
Let $\bar\lambda(x)$ denote the \emph{weighted mean} of the eigenvalues of $M(x)$ bifurcating from $\lambda$
(equivalently $\bar\lambda(x)=\frac{1}{d}\mathrm{tr}(M(x)E(x))$ where $E(x)$ is the associated spectral projector).
Then
\[
\bar\lambda^{(1)} := \left.\frac{d}{dx}\bar\lambda(x)\right|_{x=0}
\;=\;
\frac{1}{d}\,\mathrm{tr}\!\big(M^{(1)}E\big).
\]
Along the Gram flow, $\dot{\bar\lambda}(t)=\frac{1}{d}\mathrm{tr}(\dot M(t)\,E(t))$ with $\dot M(t)=-\{M(t),\Phi(t)\}$.
If $\lambda$ is simple with unit eigenvector $u$, this reduces to $\lambda^{(1)} = u^\top M^{(1)}u$.
\end{lemma}
\end{lemmaproofbox}
\noindent\textit{Reference:} \citep{kato}, Ch.~II, \S 2, Eq.~(2.32) (in this PDF’s numbering).

\subsection{Spectral Measure}
Overview: Lemma \ref{lemma:dimensionality_gram} recasts the dimensionality of a feature i denoted by $D_i$ extracted from the weights as a function of the gram matrix. \ref{lemma:dimensionality_kappa} has a corresponding expression of $D_i$ in terms the feature norm scaled by the inverse of the Rayleigh coefficient of the frame operator (which determines the range of eigenvalues that the frame operator could take) 
\begin{lemmaproofbox} 
\begin{lemma} \label{lemma:dimensionality_gram}
Let $W_i\in \mathbb{R}^m\backslash\{0\}$ be column $i$ of the weight matrix $W$ and $M = W^\top W$. 
Then: \begin{equation*}D_i = \frac{M_{ii}^2}{(M^2)_{ii}}\tag{A1}\end{equation*}
\end{lemma}
\begin{proof} By definition of the Gram matrix, $M_{ij} = \langle W_i, W_j\rangle$, i.e $M_{ii} = \|W_i\|^2$ and $(M^2)_{ii} = \sum_j M_{ij}M_{ji} = \sum_j (W_i^\top W_j)^2$. Hence: 
$$D_i = \frac{\|W_i\|^2}{\sum_j (\hat{W}_i^\top W_j)^2} =\frac{\|W_i\|^2}{\frac{1}{\|W_i\|^2}\sum_j (W_i^\top W_j)^2 }=$$
$$\frac{\|W_i\|^4}{\sum_j (W_i^\top W_j)^2 } = \frac{M_{ii}^2}{(M^2)_{ii}} $$
\end{proof}
\end{lemmaproofbox}

\begin{lemmaproofbox} 
    \begin{lemma} \label{lemma:dimensionality_kappa} Let $W_i\in \mathbb{R}^m\backslash \{0\}$ be column $i$ of $W\in \mathbb{R}^{m\times n}$ and $D_i$ be defined as in $(A1)$. Then:
        $$D_i = \frac{1}{\kappa_i} \|W_i\|^2, \quad \kappa_i = \frac{W_i^\top FW_i}{\|W_i\|^2}$$
    \end{lemma}
    \begin{proof} Rewriting:
        $$\forall_i \ (M^2)_{ii}= e_i^\top M^2 e_i = e_i^\top(W^\top W)(W^\top W)e_i =$$
$$=(We_i)^\top F(We_i) = W_i^\top F W_i$$
Hence by using Lemma A1 $$D_i = \frac{M_{ii}^2}{(M^2)_{ii}}= \frac{\|W_i\|^4}{W_i^\top FW_i} = \frac{\|W_i\|^2}{W_i^\top F W_i} \|W_i\|^2 = \frac{1}{\kappa_i}\|W_i\|^2,$$
    \end{proof}
\end{lemmaproofbox}
By having expressed the fractional dimensionality as a functional of both $M = W^\top W$ and $F = WW^\top$ (\ref{lemma:dimensionality_gram}, \ref{lemma:dimensionality_kappa})), we can use \ref{lemma:spectral_correspondence} to restate the fractional dimensionality in purely spectral terms. Since all subsequent analysis is going to be dynamic, we can WLOG use previous the results for any specific temporal snapshot $W(t)$ of the weight matrix and its corresponding Gram and Frame operators $M(t) = W(t)^\top W(t)$ and $F(t) = W(t)W(t)^\top$. Using the spectral resolution provided to us from \ref{lemma:spectral_correspondence}(iii) $F(t) = \sum_{e\in\mathcal{E}_+} \lambda_e(t) P_e(t)$, we can define the per-feature spectral measure $\mu_i(t)$ using the spectral weights:

\begin{lemmaproofbox}
    \begin{lemma}[Spectral measure]\label{lemma:spectral_measure_probability_distribution}
Fix a snapshot of the weight matrix $W\in\mathbb{R}^{m\times n}$ and let
$M := W^\top W$ and $F := WW^\top$.
Let $M=\sum_{e\in E}\lambda_e E_e$ be the orthogonal spectral resolution from Lemma~B.1, and for each $e\in E_+:=\{e:\lambda_e>0\}$
let $P_e := \lambda_e^{+}\, W E_e W^\top$ be the induced activation-space projector from Lemma~B.3.
For any feature (column) $W_i\neq 0$, define
\[
p_{i,e} \;:=\; \frac{\|P_e W_i\|^2}{\|W_i\|^2}\qquad (e\in E_+),
\qquad
\mu_i \;:=\; \sum_{e\in E_+} p_{i,e}\,\delta_{\lambda_e}.
\]
Then:
\begin{enumerate}
\item $p_{i,e}\ge 0$ for all $e\in E_+$ and $\sum_{e\in E_+} p_{i,e}=1$.
In particular, $\mu_i$ is a well-defined probability measure (finite atomic) supported on $\{\lambda_e:e\in E_+\}\subset(0,\infty)$.
Equivalently, for any Borel set $B\subset\mathbb{R}$,
\[
\mu_i(B) \;=\; \sum_{e\in E_+:\,\lambda_e\in B} p_{i,e}.
\]
\item The weights admit the equivalent closed form
\[
p_{i,e} \;=\; \frac{\lambda_e\,(E_e)_{ii}}{\|W_i\|^2}\qquad (e\in E_+).
\]
\end{enumerate}
\end{lemma}
\begin{proof}
We use only Lemmas~B.1--B.4.

(1) Since each $P_e$ is an orthogonal projector (Lemma~B.3(i)), $\|P_eW_i\|^2\ge 0$, hence $p_{i,e}\ge 0$.

To show the weights sum to $1$, first note that $W_i = W e_i\in\mathrm{Im}(W)$.
By Lemma~B.4(iii), $\mathbb{R}^m = \big(\bigoplus_{e\in E_+}\mathrm{Im}(P_e)\big)\oplus \ker(W^\top)$ and $\sum_{e\in E_+}P_e$ is the orthogonal projector onto $\mathrm{Im}(W)=\ker(W^\top)^\perp$.
Therefore $\sum_{e\in E_+}P_e W_i = W_i$.
Because the projectors are mutually orthogonal (Lemma~B.3(ii)), the vectors $\{P_eW_i\}_{e\in E_+}$ are pairwise orthogonal, so
\[
\|W_i\|^2 \;=\; \Big\|\sum_{e\in E_+}P_eW_i\Big\|^2
\;=\; \sum_{e\in E_+}\|P_eW_i\|^2.
\]
Dividing by $\|W_i\|^2>0$ yields $\sum_{e\in E_+}p_{i,e}=1$.
Hence $\mu_i=\sum_{e\in E_+}p_{i,e}\delta_{\lambda_e}$ has total mass $1$, so it is a probability measure; its action on Borel sets is the standard one for finite atomic measures.

(2) Since $P_e$ is an orthogonal projector (Lemma~B.3(i)), $\|P_eW_i\|^2 = W_i^\top P_e W_i$.
Using $P_e=\lambda_e^{+}WE_eW^\top$ (Lemma~B.3) and $W_i=We_i$, we compute
\[
W_i^\top P_e W_i
= e_i^\top W^\top(\lambda_e^{+}WE_eW^\top)We_i
= \lambda_e^{+}\, e_i^\top (W^\top W)\,E_e\,(W^\top W)\, e_i
= \lambda_e^{+}\, e_i^\top M E_e M e_i.
\]
By Lemma~B.1, $ME_e=\lambda_eE_e$ and $E_eM=\lambda_eE_e$, so $ME_eM=\lambda_e^2E_e$.
Thus $W_i^\top P_e W_i = \lambda_e^{+}\lambda_e^2\, e_i^\top E_e e_i = \lambda_e (E_e)_{ii}$.
Dividing by $\|W_i\|^2$ gives the claimed formula for $p_{i,e}$.
\qedhere
\end{proof}
\end{lemmaproofbox}

\begin{lemmaproofbox}
\label{lemma:moments}
\begin{lemma}[Moments of the spectral measure]
Fix $i$ with $W_i\neq 0$ and let $\mu_i=\sum_{e\in E_+}p_{i,e}\delta_{\lambda_e}$ be as in Lemma~B.11.
Then for every integer $r\ge 1$,
\[
\mathbb{E}_{\mu_i}[\lambda^r]
\;:=\; \int \lambda^r\, d\mu_i(\lambda)
\;=\; \frac{W_i^\top F^r W_i}{\|W_i\|^2}.
\]
Moreover, writing the Moore--Penrose pseudoinverse as $F^{+}=\sum_{e\in E_+}\lambda_e^{+}P_e$,
\[
\frac{W_i^\top F^{+} W_i}{\|W_i\|^2}
\;=\; \int \lambda^{+}\, d\mu_i(\lambda)
\;=\; \mathbb{E}_{\mu_i}[\lambda^{+}],
\qquad\text{where }\lambda^{+}:=\frac{1}{\lambda}\text{ on }(0,\infty).
\]
\end{lemma}
\begin{proof}
By Lemma~B.3(iii), $F=\sum_{e\in E_+}\lambda_e P_e$, and by Lemma~B.3(ii) the projectors satisfy
$P_eP_f=\delta_{ef}P_e$ (orthogonal idempotents).
Therefore, for any integer $r\ge 1$,
\[
F^r
=\Big(\sum_{e\in E_+}\lambda_e P_e\Big)^r
=\sum_{e\in E_+}\lambda_e^r P_e,
\]
since all mixed products $P_{e_1}\cdots P_{e_r}$ vanish unless $e_1=\cdots=e_r$.

Hence
\[
W_i^\top F^r W_i
= \sum_{e\in E_+}\lambda_e^r\, W_i^\top P_e W_i
= \sum_{e\in E_+}\lambda_e^r\, \|P_eW_i\|^2,
\]
where we used that $P_e$ is an orthogonal projector (Lemma~B.3(i)), so $W_i^\top P_e W_i=\|P_eW_i\|^2$.
Dividing by $\|W_i\|^2$ and recalling $p_{i,e}=\|P_eW_i\|^2/\|W_i\|^2$ gives
\[
\frac{W_i^\top F^r W_i}{\|W_i\|^2}
= \sum_{e\in E_+} p_{i,e}\lambda_e^r
= \int \lambda^r\, d\mu_i(\lambda)
= \mathbb{E}_{\mu_i}[\lambda^r].
\]

The pseudoinverse identity is identical: by definition $F^{+}=\sum_{e\in E_+}\lambda_e^{+}P_e$,
so
\[
\frac{W_i^\top F^{+} W_i}{\|W_i\|^2}
= \sum_{e\in E_+} p_{i,e}\lambda_e^{+}
= \int \lambda^{+}\, d\mu_i(\lambda)
= \mathbb{E}_{\mu_i}[\lambda^{+}].
\]
\end{proof}
\end{lemmaproofbox}

This allows us to readily reformulate the Rayleigh quotient as the first moment of the spectral measure $\mu_i(t)$: 
\begin{corollary}
$\kappa_i = \mathbb{E}_{\mu_i(t)}[\lambda]$
\end{corollary}

This means we have the following dynamic law for the Rayleigh quotient throughout training: $$\dot \kappa_i(t) = \sum_k \dot \lambda_k (t) p_{ik}(t) + \sum_k \lambda_k(t)\dot p_{ik}(t),$$
where the first term represents the \textbf{eigenvalue drift} and the second term is the \textbf{mass transport} of features across eigenspaces (e.g. eigenhopping). 

\begin{lemmaproofbox}
\begin{lemma}[Cyclic space, fractional powers, and pseudoinverse Cauchy--Schwarz]
\label{lemma:cyclic_cs}
Let $F\in\mathbb{R}^{m\times m}$ be symmetric PSD with spectral decomposition \(F=\sum_{e=0} \lambda_e P_e,\qquad \lambda_k\ge 0\), where the $P_k$ are orthogonal projectors, $P_kP_\ell=\delta_{k\ell}P_k$, and $\sum_k P_k=I$.
Let
\[
P:=P_{\mathrm{Im}(F)}=\sum_{\lambda_k>0}P_k=I-P_0,
\qquad
F^{1/2}:=\sum_k \lambda_k^{1/2}P_k,
\qquad
F^{+/2}:=\sum_{k}\lambda_k^{+1/2}P_k.
\]
For $x\in\mathbb{R}^m$, define the cyclic subspace
\[
\mathcal{K}_x:=\mathrm{span}\{x,Fx,F^2x,\dots\}.
\]
Then:
\begin{enumerate}
\item $\mathcal{K}_x=\mathrm{span}\{P_kx:\,P_kx\neq 0\}$. In particular, $\dim(\mathcal{K}_x)\le m$ and there exists
$r\le m$ such that $\mathcal{K}_x=\mathrm{span}\{x,Fx,\dots,F^{r-1}x\}$.
\item $F^{1/2}x\in\mathcal{K}_x$ and $F^{+/2}x\in\mathcal{K}_x$.
\item For all $x\in\mathbb{R}^m$,
\[
(x^\top Px)^2\ \le\ (x^\top Fx)\,(x^\top F^+x),
\]
and if $x\in\mathrm{Im}(F)$ (equivalently $Px=x$), then
\[
(x^\top x)^2\ \le\ (x^\top Fx)\,(x^\top F^+x).
\]
Moreover, equality holds iff $F^{1/2}x$ and $F^{+/2}x$ are linearly dependent, equivalently iff $Fx=c\,Px$ for some
$c\in\mathbb{R}$ (and iff $Fx=cx$ when $x\in\mathrm{Im}(F)$).
\end{enumerate}
\end{lemma}

\begin{proof}
(1) For every integer $t\ge 0$,
\[
F^t x=\sum_{k=0}^K \lambda_k^t\,P_kx,
\]
so $\mathcal{K}_x\subseteq \mathrm{span}\{P_kx:\,P_kx\neq 0\}$.

For the reverse inclusion, let $S:=\{k:\,P_kx\neq 0\}$ and let $\{\lambda_k\}_{k\in S}$ be the distinct eigenvalues
present in $x$. For each fixed $j\in S$, choose a polynomial $q_j$ of degree $\le |S|-1$ such that
$q_j(\lambda_j)=1$ and $q_j(\lambda_k)=0$ for all $k\in S\setminus\{j\}$ (Lagrange interpolation on the finite set
$\{\lambda_k\}_{k\in S}$). Then
\[
q_j(F)x=\sum_{k\in S} q_j(\lambda_k)\,P_kx = P_jx,
\]
and since $q_j(F)x$ is a linear combination of $\{F^t x\}_{t=0}^{|S|-1}$, we have $P_jx\in\mathcal{K}_x$.
Thus $\mathrm{span}\{P_kx:\,P_kx\neq 0\}\subseteq \mathcal{K}_x$, proving equality. The stabilization
$\mathcal{K}_x=\mathrm{span}\{x,Fx,\dots,F^{r-1}x\}$ follows with $r:=\dim(\mathcal{K}_x)\le |S|\le m$.

(2) Using the spectral expansions,
\[
F^{1/2}x=\sum_k \lambda_k^{1/2}P_kx,\qquad
F^{+/2}x=\sum_{\lambda_k>0}\lambda_k^{-1/2}P_kx,
\]
and each $P_kx\in\mathcal{K}_x$ by (1), so both vectors lie in $\mathcal{K}_x$.

(3) Note $F^{1/2}F^{+/2}=F^{+/2}F^{1/2}=P$. Let $a:=F^{1/2}x$ and $b:=F^{+/2}x$. Then
\[
\langle a,b\rangle = x^\top F^{1/2}F^{+/2}x = x^\top Px,\qquad
\|a\|^2=x^\top Fx,\qquad
\|b\|^2=x^\top F^+x.
\]
Cauchy--Schwarz in $\mathbb{R}^m$ yields $(x^\top Px)^2=\langle a,b\rangle^2\le \|a\|^2\|b\|^2=(x^\top Fx)(x^\top F^+x)$.
If $x\in\mathrm{Im}(F)$ then $Px=x$ and the stated simplification follows.

Equality in Cauchy--Schwarz holds iff $a$ and $b$ are linearly dependent, i.e. $F^{1/2}x=c\,F^{+/2}x$. Left-multiplying
by $F^{1/2}$ gives $Fx=c\,F^{1/2}F^{+/2}x=c\,Px$. If $Px=x$, this reduces to $Fx=cx$.
\end{proof}
\end{lemmaproofbox}
\begin{lemmaproofbox}
\begin{lemma}[$L^2(\mu)$ for finite atomic measures]
\label{lemma:l2_atomic}
Let $\mu=\sum_{k=1}^s p_k\,\delta_{\lambda_k}$ be a probability measure on $\mathbb{R}$ with $p_k>0$ and $\sum_k p_k=1$.
Then $L^2(\mu)$ can be identified with $\mathbb{R}^s$ via $f\mapsto (f(\lambda_1),\dots,f(\lambda_s))$, with inner product
\[
\langle f,g\rangle_{L^2(\mu)}=\int f(\lambda)g(\lambda)\,d\mu(\lambda)=\sum_{k=1}^s p_k\,f(\lambda_k)g(\lambda_k),
\]
and norm $\|f\|_{L^2(\mu)}^2=\sum_{k=1}^s p_k\,|f(\lambda_k)|^2$. In particular, every function on
$\{\lambda_1,\dots,\lambda_s\}$ is in $L^2(\mu)$, and $L^2(\mu)$ is a finite-dimensional Hilbert space.
\end{lemma}

\begin{proof}
Since $\mu$ is supported on the finite set $\{\lambda_1,\dots,\lambda_s\}$ with strictly positive masses $p_k$, two functions
$f,g$ are equal $\mu$-a.s. iff $f(\lambda_k)=g(\lambda_k)$ for all $k$. Thus an $L^2(\mu)$ equivalence class is uniquely
specified by the vector of its values $(f(\lambda_1),\dots,f(\lambda_s))\in\mathbb{R}^s$.

Moreover,
\[
\int |f(\lambda)|^2\,d\mu(\lambda)=\sum_{k=1}^s p_k\,|f(\lambda_k)|^2 < \infty
\]
for every such vector (finite sum), so every function on the support belongs to $L^2(\mu)$. The formula for the inner product
is immediate from the definition of integration against an atomic measure. Completeness follows because $L^2(\mu)$ is
finite-dimensional.
\end{proof}
\end{lemmaproofbox}
\begin{lemmaproofbox}
\begin{lemma}[Cauchy--Schwarz in $L^2(\mu)$]
\label{lemma:cs_L2}
Let $(X,\mathcal{B},\mu)$ be a probability space and let $u,v\in L^2(\mu)$. Then
\[
|\langle u,v\rangle_{L^2(\mu)}|^2 \le \|u\|_{L^2(\mu)}^2\,\|v\|_{L^2(\mu)}^2.
\]
If $\|v\|_{L^2(\mu)}>0$, equality holds iff there exists $c\in\mathbb{R}$ such that $u=c\,v$ $\mu$-a.s.
(If $\|v\|_{L^2(\mu)}=0$, then $v=0$ $\mu$-a.s. and the inequality is trivial.)
\end{lemma}

\begin{proof}
If $\|v\|_2=0$ the claim is immediate. Otherwise define
\[
c:=\frac{\langle u,v\rangle}{\|v\|_2^2}.
\]
Then
\[
0\le \|u-cv\|_2^2
= \|u\|_2^2 - 2c\langle u,v\rangle + c^2\|v\|_2^2
= \|u\|_2^2 - \frac{\langle u,v\rangle^2}{\|v\|_2^2},
\]
which rearranges to $\langle u,v\rangle^2\le \|u\|_2^2\|v\|_2^2$. Equality holds iff $\|u-cv\|_2^2=0$, i.e. $u=cv$
$\mu$-a.s.
\end{proof}
\end{lemmaproofbox}
\begin{lemmaproofbox}
\begin{lemma}[Cyclic isometry $L^2(\mu_x)\simeq \mathcal{K}_x$]
\label{lemma:cyclic_isometry}
Let $F\in\mathbb{R}^{m\times m}$ be symmetric PSD with spectral decomposition $F=\sum_{k=0}^K \lambda_k P_k$.
Fix $x\in\mathrm{Im}(F)\setminus\{0\}$, and define the (positive-spectrum) spectral weights and measure
\[
p_k:=\frac{\|P_kx\|^2}{\|x\|^2}\quad(\lambda_k>0),
\qquad
\mu_x:=\sum_{\lambda_k>0} p_k\,\delta_{\lambda_k}.
\]
For $f$ defined on $\mathrm{supp}(\mu_x)$, define functional calculus on $\mathrm{Im}(F)$ by
\[
f(F):=\sum_{\lambda_k>0} f(\lambda_k)\,P_k.
\]
Define $U_x:L^2(\mu_x)\to \mathcal{K}_x$ by
\[
U_x(f):=\frac{1}{\|x\|}\,f(F)x.
\]
Then $U_x$ is a well-defined unitary isomorphism (linear, bijective, and inner-product preserving). In particular, for
all $f,g\in L^2(\mu_x)$,
\[
\langle U_x(f),U_x(g)\rangle_{\mathbb{R}^m}=\int f(\lambda)g(\lambda)\,d\mu_x(\lambda),
\]
and for any real-valued $h$ on $\mathrm{supp}(\mu_x)$,
\[
\frac{x^\top h(F)x}{\|x\|^2}=\int h(\lambda)\,d\mu_x(\lambda).
\]
\end{lemma}

\begin{proof}
\emph{Well-definedness.} If $f=g$ $\mu_x$-a.s., then $f(\lambda_k)=g(\lambda_k)$ for every atom with $p_k>0$.
If $p_k=0$ then $P_kx=0$, so changing $f(\lambda_k)$ does not change $f(F)x$. Hence $f(F)x=g(F)x$ and $U_x$ is
well-defined on $L^2(\mu_x)$ equivalence classes.

\emph{Isometry.} Using orthogonality $P_kP_\ell=\delta_{k\ell}P_k$,
\begin{align*}
\langle U_x(f),U_x(g)\rangle
&=\frac{1}{\|x\|^2}\,\langle f(F)x,\,g(F)x\rangle \\
&=\frac{1}{\|x\|^2}\left\langle \sum_{\lambda_k>0} f(\lambda_k)P_kx,\ \sum_{\lambda_\ell>0} g(\lambda_\ell)P_\ell x\right\rangle \\
&=\frac{1}{\|x\|^2}\sum_{\lambda_k>0} f(\lambda_k)g(\lambda_k)\,\|P_kx\|^2
=\sum_{\lambda_k>0} p_k f(\lambda_k)g(\lambda_k)
=\int f g\,d\mu_x.
\end{align*}

\emph{Surjectivity.} By Lemma~\ref{lemma:cyclic_cs}(1), $\mathcal{K}_x=\mathrm{span}\{P_kx:\,P_kx\neq 0\}
=\mathrm{span}\{P_kx:\,p_k>0\}$. Given $y=\sum_{p_k>0} a_k P_kx\in\mathcal{K}_x$, define $f$ on $\mathrm{supp}(\mu_x)$ by
$f(\lambda_k):=a_k\|x\|$. Then $f\in L^2(\mu_x)$ (finite support) and
\[
U_x(f)=\frac{1}{\|x\|}\sum_{p_k>0} f(\lambda_k)P_kx=\sum_{p_k>0} a_k P_kx=y.
\]
Injectivity follows from the isometry: if $U_x(f)=0$ then $\|f\|_{L^2(\mu_x)}^2=\|U_x(f)\|^2=0$, hence $f=0$ $\mu_x$-a.s.
Thus $U_x$ is unitary.

\emph{Quadratic form identity.} For any $h$ on $\mathrm{supp}(\mu_x)$,
\[
x^\top h(F)x=\sum_{\lambda_k>0} h(\lambda_k)\,x^\top P_kx
=\sum_{\lambda_k>0} h(\lambda_k)\,\|P_kx\|^2,
\]
and dividing by $\|x\|^2$ gives $\int h\,d\mu_x$.
\end{proof}
\end{lemmaproofbox}

\section{Theorem Proofs and Corollaries}
\label{sec:theorems_proofs_corollaries}

\begin{lemmaproofbox}
\textbf{Corollary 2} (Projective Linearity). \textit{Assume Spectral Localization holds. Then, the fractional dimensionality $D_i$ of any feature is linearly determined by its feature norm $\propto \|W_i\|^2$ with slope $k$ equal to the reciprocal of the eigenvalue $\lambda_e$ of the subspace it occupies $D_i = \lambda_e^+ \|W_i\|^2$}.
\begin{proof}[Proof of Corollary~\ref{corollary:projective_linearity_section4} (Projective Linearity)]\label{corollary:projective_linearity}
By Lemma \ref{lemma:dimensionality_kappa},
\[
D_i \;=\; \frac{1}{\kappa_i}\,\lvert W_i\rvert^2,
\qquad
\kappa_i \;=\; \frac{W_i^\top F W_i}{\lvert W_i\rvert^2}.
\]
Under Spectral Localization (Theorem \ref{theorem:spectral_localization}), $F W_i=\lambda_e\,W_i$ for some $F$ eigenvalue $\lambda_e>0$. Hence
\[
W_i^\top F W_i \;=\; \lambda_e\,W_i^\top W_i \;=\; \lambda_e\,\lvert W_i\rvert^2,
\]
so $\kappa_i=\lambda_e$ and therefore
\[
D_i \;=\; \frac{\lvert W_i\rvert^2}{\lambda_e} \;=\; \lambda_e^{+}\,\lvert W_i\rvert^2,
\]
where $\lambda_e^+=1/\lambda_e$ since $\lambda_e>0$ in Theorem \ref{theorem:spectral_localization}. 
\end{proof}
\end{lemmaproofbox}

\begin{lemmaproofbox}
\textbf{Theorem 3} (Decomposition into Tight Frames)\textit{ Assume Spectral Localization (Theorem \ref{theorem:spectral_localization}) holds for all features $i\in\{1,\dots,n\}$. Let
$\Lambda=\{\lambda_1,\dots,\lambda_r\}$ be the distinct positive eigenvalues of $F=WW^\top$,
and define the index clusters and their spans
\[
C_k := \{\,i : \lambda(i)=\lambda_k\,\},
\qquad
V_k := \mathrm{span}\{\,W_i : i\in C_k\,\}.
\]
Then $\{C_k\}_{k=1}^r$ partitions $\{1,\dots,n\}$, the subspaces $V_k$ are pairwise orthogonal, and
\[
F \;=\; \sum_{k=1}^r \lambda_k\,P_{V_k}
\quad \text{(on $\mathrm{Im}(W)$; plus $0$ on $\ker(F)$).}
\]
Moreover, within each cluster $C_k$,
\[
\sum_{i\in C_k} W_i W_i^\top \;=\; \lambda_k\,I
\quad \text{on }V_k,
\]
i.e.\ $(W_{C_k})$ forms a tight frame for $V_k$ with frame constant $\lambda_k$.}
\begin{proof}\label{theorem:decomp_tight_frame}
By Theorem \ref{theorem:spectral_localization}, each $W_i$ is an eigenvector of $F$ with eigenvalue $\lambda(i)>0$, i.e.\ $F W_i=\lambda(i)W_i$.
Fix distinct eigenvalues $\lambda_k\neq\lambda_\ell$. Since $F$ is symmetric PSD, eigenspaces for distinct eigenvalues are orthogonal, hence
$W_i\perp W_j$ whenever $\lambda(i)\neq\lambda(j)$. Therefore the sets $C_k$ are disjoint and partition $\{1,\dots,n\}$, and the spans $V_k$ are pairwise orthogonal.

For any $v\in V_k$, write $v=\sum_{i\in C_k}\alpha_i W_i$. Then
\[
Fv=\sum_{i\in C_k}\alpha_i F W_i
=\sum_{i\in C_k}\alpha_i \lambda_k W_i
=\lambda_k v,
\]
so $V_k\subseteq \operatorname{Im}{P_k}$. Conversely, since $F=WW^\top$, we have $\mathrm{Im}(F)=\mathrm{Im}(W)=\mathrm{span}\{W_i\}_i$. Any eigenvector $v$ of $F$ with eigenvalue $\lambda_k>0$ lies in $\mathrm{Im}(F)$, hence decomposes uniquely as
$v=\sum_{\ell=1}^r v_\ell$ with $v_\ell\in V_\ell$ (orthogonal direct sum). Applying $F$ gives
\[
Fv=\sum_{\ell=1}^r \lambda_\ell v_\ell,
\qquad\text{but also}\qquad
Fv=\lambda_k v=\sum_{\ell=1}^r \lambda_k v_\ell,
\]
so $(\lambda_\ell-\lambda_k)v_\ell=0$ for all $\ell$, forcing $v_\ell=0$ for $\ell\neq k$. Thus
\[
\operatorname{Im}({P_k})\cap \mathrm{Im}(W)=V_k,
\]
and $F$ acts as $\lambda_k I$ on $V_k$. This yields
\[
F=\sum_{k=1}^r \lambda_k P_{V_k}\quad\text{on }\mathrm{Im}(W),
\]
and $F=0$ on $\ker(F)=\mathrm{Im}(W)^\perp$.

For the tight-frame claim, define
\[
F_k := \sum_{i\in C_k} W_i W_i^\top.
\]
Take $v\in V_k$. For any $j\notin C_k$, orthogonality of distinct eigenspaces gives $W_j^\top v=0$, hence
\[
Fv=\sum_{i=1}^n W_i W_i^\top v=\sum_{i\in C_k} W_i W_i^\top v=F_k v.
\]
But also $Fv=\lambda_k v$ for $v\in V_k$, so $F_k v=\lambda_k v$ for all $v\in V_k$, i.e.\ $F_k=\lambda_k I$ on $V_k$.
This is the tight-frame condition with frame constant $\lambda_k$.
\end{proof}
\end{lemmaproofbox}

\begin{lemmaproofbox}
\textbf{Theorem 4 }(Spectral Identification of Geometry)
\textit{Assume Spectral Localization (Theorem \ref{theorem:spectral_localization}) holds so that each cluster $C_k$ forms a tight frame as in Theorem~3.
Let $M_k := W_{C_k}^\top W_{C_k}$ be the local Gram matrix. Then the geometry of cluster $C_k$
is an instance of an association scheme algebra $\mathcal A$ if and only if the eigenspace
decomposition of $M_k$ aligns with the canonical strata $(W_0,\dots,W_s)$ of $\mathcal A$.}

\begin{proof}\label{theorem:spectral-identification}
We use standard association-scheme facts (Appendix A.2). If $(A_0,\dots,A_s)$ are adjacency matrices of an
association scheme on $\Omega$, its Bose--Mesner algebra $\mathcal A$ admits a canonical orthogonal decomposition
\[
\mathbb R^\Omega = W_0\oplus\cdots\oplus W_s
\]
with orthogonal projectors $(S_0,\dots,S_s)$, and every $X\in\mathcal A$ preserves each stratum $W_e$.
Moreover, $(S_e)$ form a basis of mutually orthogonal idempotents for $\mathcal A$, so each $X\in\mathcal A$
has a unique expansion
\[
X=\sum_{e=0}^s \theta_e S_e,
\]
and acts as scalar $\theta_e$ on $W_e=\mathrm{Im}(S_e)$.

\emph{($\Rightarrow$)} If the cluster geometry is an instance of $\mathcal A$, then the defining invariant matrix (here $M_k$) lies in $\mathcal A$.
Thus $M_k=\sum_e \theta_e S_e$, so each stratum $W_e$ is an invariant subspace on which $M_k$ acts as a scalar.
Equivalently, the spectral decomposition of $M_k$ is governed by the stratum decomposition.

\emph{($\Leftarrow$)} Conversely, suppose the eigenspace decomposition of $M_k$ aligns with the strata, i.e.\ $M_k$ is block-scalar on
$\mathbb R^\Omega=\bigoplus_e W_e$. Then there exist scalars $\theta_e$ such that $M_k|_{W_e}=\theta_e I$ for all $e$.
Therefore $M_k=\sum_e \theta_e S_e\in \mathcal A$ since $S_e\in\mathcal A$ and $\mathcal A$ is a linear space.
Hence the cluster geometry is an instance of the association scheme algebra.
\end{proof}
\end{lemmaproofbox}

\begin{lemmaproofbox}
\textbf{Corollary 5} (Simplex Identification)
\textit{A cluster $C_k$ is simplex geometry if and only if its Gram matrix $M_k$ has exactly two invariant subspaces
\[
W_0=\ker(M_k)=\mathrm{span}(\mathbf 1),
\qquad
W_1=\mathrm{Im}(M_k)=\mathbf 1^\perp,
\]
with eigenvalues $\Lambda(W_0)=0$ and $\Lambda(W_1)=\lambda_k=\frac{|C_k|}{\dim(V_k)}$.}
\begin{proof}\label{corollary:simplex}
In the simplex association scheme (Appendix A.3.1), the Bose--Mesner algebra has exactly two strata:
the constant-vector space $W_0=\mathrm{span}(\mathbf 1)$ and its orthogonal complement $W_1=\mathbf 1^\perp$,
with projectors $S_0=\frac{1}{p}J$ and $S_1=I-\frac{1}{p}J$ where $p=|C_k|$.
By Theorem~4, being an instance of the simplex scheme is equivalent to having spectral/invariant decomposition exactly along these two strata.

For eigenvalues: in the simplex model the vectors are unit and centered (sum to zero), and Appendix A.3.1 yields
\[
M_k = \frac{p}{p-1}\,S_1.
\]
Thus $M_k\mathbf 1=0$ and $M_k$ acts as scalar $\frac{p}{p-1}$ on $\mathbf 1^\perp$.

Under the tight-frame parameterization of a unit tight frame of $p$ vectors spanning a $d$-dimensional subspace,
the frame constant is $p/d$ (trace identity; cf.\ Lemma A.7). For a simplex, $p=d+1$, so
\[
\frac{p}{p-1}=\frac{d+1}{d}=\frac{p}{d}=\frac{|C_k|}{\dim(V_k)},
\]
matching the stated $\lambda_k$.
\end{proof}
\end{lemmaproofbox}

\section{Perturbative case} \label{section:perturbative_analysis}
\begin{lemmaproofbox}
\begin{theorem}[Defect decomposition: near-saturation $\Leftrightarrow$ small aggregate slack]\label{thm:defect-decomposition}
Let $W\in\mathbb{R}^{m\times n}$ with columns $w_1,\dots,w_n$, and let the frame operator be $F:=WW^\top$.
Let $r:=\mathrm{rank}(W)=\mathrm{rank}(F)$. Define leverage scores $\ell_i := w_i^\top F^{+} w_i$ and fractional
dimensionalities $D_i := \|w_i\|^4/(w_i^\top F w_i)$ (for $w_i\neq 0$). Define relative slack
$\sigma_i := 1 - D_i/\ell_i\in[0,1]$. Then:
\begin{enumerate}
\item (\emph{Capacity budget}) $\sum_{i=1}^n \ell_i = r$.
\item (\emph{Pointwise bound}) $D_i \le \ell_i$ for all $i$, with equality iff $\mu_i$ is a Dirac mass
(equivalently: $w_i$ is an eigenvector of $F$).
\item (\emph{Exact defect identity})
\[
r - \sum_{i=1}^n D_i
\;=\;
\sum_{i=1}^n (\ell_i - D_i)
\;=\;
\sum_{i=1}^n \ell_i\,\sigma_i.
\]
In particular, if $r=m$ then $m-\sum_i D_i=\sum_i \ell_i\sigma_i$.
\end{enumerate}
\end{theorem}

\begin{proof}
(1) Using $\ell_i = w_i^\top F^+ w_i$ and cyclicity of trace,
\[
\sum_{i=1}^n \ell_i
= \sum_{i=1}^n w_i^\top F^+ w_i
= \mathrm{tr}\!\left(\sum_{i=1}^n F^+ w_i w_i^\top\right)
= \mathrm{tr}\!\left(F^+ \sum_{i=1}^n w_i w_i^\top\right)
= \mathrm{tr}(F^+ F).
\]
For symmetric PSD $F$, the matrix $F^+F$ is the orthogonal projector onto $\mathrm{Im}(F)$, hence
$\mathrm{tr}(F^+F)=\mathrm{rank}(F)=r$.

(2) Fix $i$ and set $x:=w_i$. Since $x$ is a column of $W$, we have $x\in \mathrm{Im}(W)=\mathrm{Im}(F)$.
Let $a:=F^{1/2}x$ and $b:=F^{+1/2}x$. By Cauchy--Schwarz,
\[
\langle a,b\rangle^2 \le \|a\|^2\|b\|^2.
\]
Compute each term:
\[
\langle a,b\rangle
= x^\top F^{1/2}F^{+1/2}x
= x^\top (F^+F)^{1/2} x
= x^\top \Pi_{\mathrm{Im}(F)} x
= \|x\|^2,
\]
because $x\in\mathrm{Im}(F)$. Also $\|a\|^2=x^\top F x$ and $\|b\|^2=x^\top F^+ x=\ell_i$.
Thus
\[
\|x\|^4 \le (x^\top F x)(x^\top F^+ x)
\quad\Longrightarrow\quad
\frac{\|x\|^4}{x^\top F x} \le x^\top F^+ x
\quad\Longrightarrow\quad
D_i \le \ell_i.
\]
Equality holds in Cauchy--Schwarz iff $a$ and $b$ are linearly dependent, i.e.
$F^{1/2}x = c\,F^{+1/2}x$ for some $c\in\mathbb{R}$.
Multiplying by $F^{1/2}$ gives $Fx = c\,\Pi_{\mathrm{Im}(F)}x = c x$, so $x$ is an eigenvector of $F$
(with eigenvalue $c>0$ since $x\in\mathrm{Im}(F)$ and $x\neq 0$).
Equivalently, $x$ lies entirely in a single positive-eigenvalue eigenspace, which is the statement that the feature spectral measure $\mu_i$ is a Dirac mass.

(3) By definition $\ell_i\sigma_i = \ell_i - D_i$, hence
\[
\sum_{i=1}^n \ell_i\sigma_i = \sum_{i=1}^n (\ell_i - D_i) = \Big(\sum_{i=1}^n \ell_i\Big) - \sum_{i=1}^n D_i.
\]
Using (1), $\sum_i \ell_i = r$, yielding $r-\sum_i D_i = \sum_i \ell_i\sigma_i$.
The full-rank specialization is the same identity with $r=m$.
\end{proof}
\end{lemmaproofbox}

\begin{lemmaproofbox}
\begin{theorem}[Delocalized residue bound under $\varepsilon$-saturation]\label{cor:tail-mass}
Assume $\sum_{i=1}^n D_i \ge (1-\varepsilon)\,r$. Then for any $\tau\in(0,1)$,
\[
\frac{1}{r}\sum_{i:\,\sigma_i\ge\tau}\ell_i
\;\le\;
\frac{\varepsilon}{\tau}.
\]
\end{theorem}

\begin{proof}
From Theorem~\ref{thm:defect-decomposition}(3),
\[
\sum_{i=1}^n \ell_i\sigma_i = r - \sum_{i=1}^n D_i \le \varepsilon r.
\]
On the index set $S_\tau := \{i:\sigma_i\ge\tau\}$ we have $\ell_i\sigma_i \ge \tau \ell_i$, so
\[
\varepsilon r \;\ge\; \sum_{i=1}^n \ell_i\sigma_i \;\ge\; \sum_{i\in S_\tau}\ell_i\sigma_i \;\ge\; \tau\sum_{i\in S_\tau}\ell_i.
\]
Divide by $\tau r$ to obtain $(1/r)\sum_{i:\sigma_i\ge\tau}\ell_i \le \varepsilon/\tau$.
\end{proof}
\end{lemmaproofbox}

\begin{lemmaproofbox}
\begin{theorem}[Eigenvector residual equals eigenvalue variance]\label{lem:residual-variance}
Let $F=\sum_{k:\lambda_k>0}\lambda_k P_k$ be the spectral decomposition of $F$ on $\mathrm{Im}(F)$.
For $w_i\neq 0$, define weights $p_{ik}:=\|P_k w_i\|^2/\|w_i\|^2$ and the associated spectral measure
$\mu_i := \sum_{k:\lambda_k>0} p_{ik}\,\delta_{\lambda_k}$.
Let $\kappa_i := \mathbb{E}_{\mu_i}[\lambda]$. Then
\[
\frac{\|F w_i - \kappa_i w_i\|^2}{\|w_i\|^2}
\;=\;
\mathrm{Var}_{\mu_i}(\lambda)
\;=\;
\mathbb{E}_{\mu_i}\big[(\lambda-\kappa_i)^2\big].
\]
In particular,
\[
\mathrm{CV}_i
:=
\frac{\sqrt{\mathrm{Var}_{\mu_i}(\lambda)}}{\mathbb{E}_{\mu_i}[\lambda]}
=
\frac{\|F w_i - \kappa_i w_i\|}{\kappa_i\|w_i\|}
\]
vanishes iff $w_i$ is an eigenvector of $F$ (equivalently $\mu_i$ is Dirac).
\end{theorem}

\begin{proof}
Write $w_i=\sum_k P_k w_i$. Then
\[
F w_i = \sum_k \lambda_k P_k w_i,
\qquad
\kappa_i w_i = \sum_k \kappa_i P_k w_i,
\]
so
\[
F w_i - \kappa_i w_i = \sum_k (\lambda_k-\kappa_i)\,P_k w_i.
\]
Using orthogonality of distinct spectral subspaces,
\[
\|F w_i - \kappa_i w_i\|^2
= \sum_k (\lambda_k-\kappa_i)^2 \|P_k w_i\|^2.
\]
Divide by $\|w_i\|^2$ to get
\[
\frac{\|F w_i - \kappa_i w_i\|^2}{\|w_i\|^2}
= \sum_k p_{ik}(\lambda_k-\kappa_i)^2
= \mathbb{E}_{\mu_i}\big[(\lambda-\kappa_i)^2\big]
= \mathrm{Var}_{\mu_i}(\lambda).
\]
Variance is zero iff $\lambda$ is $\mu_i$-a.s.\ constant, i.e.\ $\mu_i$ is Dirac on some $\lambda(i)>0$,
equivalently $w_i$ lies entirely in a single eigenspace, i.e.\ $F w_i=\lambda(i) w_i$.
The formula for $\mathrm{CV}_i$ is immediate by dividing by $\kappa_i^2$ and taking square-roots.
\end{proof}
\end{lemmaproofbox}

\begin{lemmaproofbox}
\begin{theorem}[Band-localization $\Rightarrow$ quasi-tightness (Kantorovich bound)]\label{thm:kantorovich-band}
Fix $i$ and suppose $\mu_i$ is supported on $[\lambda_i^{-},\lambda_i^{+}]\subset(0,\infty)$.
Let $\kappa_i^\star:=\lambda_i^{+}/\lambda_i^{-}\ge 1$.
Then
\[
\kappa_i\,h_i
=
\mathbb{E}_{\mu_i}[\lambda]\;\mathbb{E}_{\mu_i}[\lambda^{-1}]
\;\le\;
\frac{(\kappa_i^{\star}+1)^2}{4\kappa_i^{\star}},
\]
and consequently
\[
\sigma_i
=
1-\frac{D_i}{\ell_i}
\le
1-\frac{4\kappa_i^{\star}}{(\kappa_i^{\star}+1)^2}.
\]
Equivalently, with $\omega_i := (\lambda_i^{+}-\lambda_i^{-})/(\lambda_i^{+}+\lambda_i^{-})\in[0,1)$,
\[
\sigma_i \le \omega_i^2.
\]
\end{theorem}

\begin{proof}
Step 1 (Kantorovich inequality in the scalar form used here).
Let $X$ be any random variable supported on $[a,b]\subset(0,\infty)$. Pointwise on $[a,b]$,
\[
(x-a)(b-x)\ge 0
\;\Longleftrightarrow\;
x^2-(a+b)x+ab\le 0
\;\Longleftrightarrow\;
x+\frac{ab}{x}\le a+b.
\]
Taking expectations gives
\[
\mathbb{E}[X] + ab\,\mathbb{E}[X^{-1}] \le a+b.
\]
By AM--GM, $\mathbb{E}[X] + ab\,\mathbb{E}[X^{-1}] \ge 2\sqrt{ab\,\mathbb{E}[X]\mathbb{E}[X^{-1}]}$, hence
\[
2\sqrt{ab\,\mathbb{E}[X]\mathbb{E}[X^{-1}]}\le a+b
\;\Longrightarrow\;
\mathbb{E}[X]\mathbb{E}[X^{-1}] \le \frac{(a+b)^2}{4ab}.
\]

Step 2 (apply to $\lambda\sim\mu_i$).
Take $X=\lambda$, $a=\lambda_i^-$, $b=\lambda_i^+$. Then
\[
\mathbb{E}_{\mu_i}[\lambda]\mathbb{E}_{\mu_i}[\lambda^{-1}]
\le
\frac{(\lambda_i^-+\lambda_i^+)^2}{4\lambda_i^-\lambda_i^+}
=
\frac{(\kappa_i^\star+1)^2}{4\kappa_i^\star}.
\]
This is the first displayed inequality.

Step 3 (convert to a slack bound).
Using the spectral weights $p_{ik}$, one has
\[
w_i^\top F w_i = \sum_k \lambda_k \|P_k w_i\|^2 = \|w_i\|^2\,\mathbb{E}_{\mu_i}[\lambda],
\qquad
w_i^\top F^+ w_i = \sum_k \lambda_k^{-1} \|P_k w_i\|^2 = \|w_i\|^2\,\mathbb{E}_{\mu_i}[\lambda^{-1}].
\]
Hence $\kappa_i=\mathbb{E}_{\mu_i}[\lambda]$ and $\ell_i=\|w_i\|^2 h_i$ with $h_i:=\mathbb{E}_{\mu_i}[\lambda^{-1}]$.
Also $D_i = \|w_i\|^4/(w_i^\top F w_i) = \|w_i\|^2/\kappa_i$. Therefore
\[
\frac{D_i}{\ell_i}
=
\frac{\|w_i\|^2/\kappa_i}{\|w_i\|^2 h_i}
=
\frac{1}{\kappa_i h_i}
=
\frac{1}{\mathbb{E}_{\mu_i}[\lambda]\mathbb{E}_{\mu_i}[\lambda^{-1}]}.
\]
Combining with the upper bound on $\kappa_i h_i$ yields
\[
\frac{D_i}{\ell_i}
\ge
\frac{4\kappa_i^\star}{(\kappa_i^\star+1)^2}
\quad\Longrightarrow\quad
\sigma_i = 1-\frac{D_i}{\ell_i}
\le
1-\frac{4\kappa_i^\star}{(\kappa_i^\star+1)^2}.
\]
Finally, with $\omega_i=(\lambda_i^+-\lambda_i^-)/(\lambda_i^++\lambda_i^-)$ and $\kappa_i^\star=\lambda_i^+/\lambda_i^-$,
\[
1-\frac{4\kappa_i^\star}{(\kappa_i^\star+1)^2}
=
\frac{(\kappa_i^\star-1)^2}{(\kappa_i^\star+1)^2}
=
\left(\frac{\lambda_i^+-\lambda_i^-}{\lambda_i^++\lambda_i^-}\right)^2
=
\omega_i^2.
\]
\end{proof}
\end{lemmaproofbox}

\begin{lemmaproofbox}
\begin{theorem}[A three-class taxonomy consistent with near-saturation]\label{cor:taxonomy}
Theorem~\ref{thm:kantorovich-band} and Corollary~\ref{cor:tail-mass} justify:
\begin{enumerate}
\item \textbf{Localized (Dirac) features:} $\sigma_i=0$ (equivalently $\mathrm{CV}_i=0$), i.e.\ $\mu_i$ is Dirac and $w_i$ is an eigenvector.
\item \textbf{Band-localized features:} $\omega_i\ll 1$, hence $\sigma_i\le\omega_i^2\ll 1$ though $\mu_i$ need not be Dirac.
\item \textbf{Broadband (delocalized-residue) features:} $\omega_i=\Theta(1)$ and $\mathrm{CV}_i$ nontrivial; these incur $\sigma_i>0$,
but under $\varepsilon$-saturation their total leverage mass is small.
\end{enumerate}
\end{theorem}

\begin{proof}
(1) By Theorem~\ref{thm:defect-decomposition}(2), $\sigma_i=0$ iff $D_i=\ell_i$ iff $w_i$ is an eigenvector,
equivalently $\mu_i$ is Dirac. Lemma~\ref{lem:residual-variance} shows this is also equivalent to $\mathrm{CV}_i=0$.

(2) If $\mu_i$ is supported on a narrow band with $\omega_i\ll 1$, Theorem~\ref{thm:kantorovich-band} gives
$\sigma_i\le \omega_i^2\ll 1$, capturing “perturbations” of perfect localization.

(3) If $\omega_i=\Theta(1)$ (broad support), then typically $\mathrm{Var}_{\mu_i}(\lambda)$ is non-negligible, hence
$\mathrm{CV}_i$ is nontrivial by Lemma~\ref{lem:residual-variance}, and $\sigma_i>0$ unless $\mu_i$ collapses to Dirac.
Under $\varepsilon$-saturation, Corollary~\ref{cor:tail-mass} bounds the leverage-weighted mass of any set with
$\sigma_i\ge \tau$, which forces the broadband/high-slack population to occupy only a small fraction of the leverage budget.
\end{proof}
\end{lemmaproofbox}

\paragraph{Empirical interpretation in the perturbative regime.}
$W$ stays full-rank across sparsity, while $\sum_i D_i/m$ falls slightly below $1$ in a sparsity-dependent way. In the full-rank regime $r=m$, Theorem~\ref{thm:defect-decomposition} forces the identity \(m-\sum_i D_i=\sum_i \ell_i\sigma_i,\) so the observed deviation from saturation cannot be attributed to rank defect; it an aggregate slack term.

At the same time, the paper observes that higher sparsity leads to features becoming “more diffuse across the spectrum”
(i.e.\ less spectrally localized). 
Lemma~\ref{lem:residual-variance} formalizes “diffuse across the spectrum” as larger eigenvalue spread
$\mathrm{Var}_{\mu_i}(\lambda)$ (equivalently larger $\mathrm{CV}_i$), which is precisely the obstruction to being an eigenvector/Dirac.
This explains why increasing sparsity can correlate with stronger delocalization diagnostics even when global capacity usage remains close to saturated.

The apparent coexistence of “minimal slack” with “some delocalization” is resolved by Corollary~\ref{cor:tail-mass}:
if $\sum_i D_i\ge (1-\varepsilon)r$ with small $\varepsilon$ (near-saturation), then for any fixed delocalization threshold $\tau$,
\[
\frac{1}{r}\sum_{i:\sigma_i\ge\tau}\ell_i \le \frac{\varepsilon}{\tau}.
\]
Thus, near-saturation does not forbid delocalized/broadband features; it forbids them from carrying much total leverage mass. Empirically, the leverage-weighted tail mass of high-slack features remains small across sparsity, matching this prediction: delocalization is allowed, but it is confined to a small “residual” subset in leverage (Appendix \ref{appendix:plots}). 

Finally, Theorem~\ref{thm:kantorovich-band} gives a clean perturbative mechanism: if most features remain band-localized
with small relative half-width $\omega_i\ll 1$, then they automatically have $\sigma_i\le \omega_i^2\ll 1$, so they contribute
negligibly to aggregate slack. The sparsity-driven delocalization can then be understood as an increasing \emph{broadening of $\mu_i$}
for a minority of features (broadband residue), which raises their $\mathrm{CV}_i$ and $\sigma_i$ without materially changing the
global budget because their total leverage weight is small.

\section{Localization vs Delocalization}
\label{appendix:localization_delocalization}
\subsection{HT-SR}
\label{appendix:martin_and_mahoney}
 Empirical work by Charles Martin and Michael Mahoney has demonstrated that during training models "self-regularize" as their generalization capabilities increase. The so-called phenomenology of Heavy-Tailed Self-Regularization (HT-SR) \cite{htsr1, htsr2, htsr3}) is inspired by Random Matrix Theory (RMT) and its object of study is the empirical spectral density (ESD) of the layer weight matrices. More specifically, consider the real-valued weight matrices $\mathbf{W}_l\in \mathbb{R}^{N\times M}$  with singular value decomposition (SVD) $\mathbf{W}_l = \mathbf{U}\mathbf{\Sigma}\mathbf{V}^*$. Here, $\nu_i = \mathbf{\Sigma}_{ii}$ is the $i$th singular value, $p_i = \nu_i^2/\sum_i \nu_i^2$ the spectral probability distribution over singular values and $\mathbf{X}_l = \frac{1}{N}\mathbf{W}_l^T\mathbf{W}_l$ is the sample covariance matrix. By computing its eigenvalues $\mathbf{X}\mathbf{v}_i =\lambda_i\mathbf{v}_i$, where $\forall_{i=1, \cdots, M}\lambda_i = \nu_i^2$, the authors identify a gradual phase transition of the ESD away from Marchenko-Pastur (MP) /Tracy-Widom (TW) spectral statistics to $P_{W_{ij}}(X)\sim \frac{1}{x^{1+\mu}}$ with $\mu>0$, i.e. Power-Law (PL)/ Heavy-Tailed (HT) statistics.  While HT-SR's focus is exclusively on the eigenvalue behavior, we argue that the observation is the surface-level feature of the Anderson localization research program in mathematical physics \cite{anderson}.
\subsection{Anderson Localization}
\label{sec:anderson_localization}
\noindent The literature on the spectral behavior of random operators and disordered systems classifies the Tracy-Widom spectral distribution as the eigenvalue statistics of the Gaussian universality class with delocalized eigenvectors, such as GUE \cite{taoRMT}, whereas the Power Law/Heavy-Tailed spectral statistics represent random matrices with localized eigenvectors \cite{aizenman1993, mirlin1996, thouless}.

\section{Gradient Flow}

\subsection{Gradient Flow Dynamics}
All of our derivations so far have been under the assumption that the model has saturated its capacity, which makes the results a static description. If we were to provide one of the dynamics, we need to be cautious. Solving explicitly similar systems from mathematical physics \cite{thomson, smale} remains an open problem. That said, the tools we have introduced in Section \ref{section:generalizable} provide us with enough handle on the dynamics to show that our initial capacity saturation assumption is well-founded. 
\begin{theorem}[Gradient Flow]
    Let $M=W^\top W$ be the autoencoder and $\Phi = \operatorname{Sym}(\nabla_ML)$ be the gradient kernel, where \ref{lemma:gradient-kernel}: 
    $$\Phi=-\mathbb{E}_x [\boldsymbol{\delta}_{(x)}\mathbf{x}^\top + x \pmb{\delta}_{(x)}^\top]$$$$
\boldsymbol{\delta}_{(x)} = I \odot (\mathbf{x} - \mathbf{x}') \odot \mathbb{I}(M\mathbf{x} + b > 0).
$$
    Then: \begin{equation}
        \dot M = -\{M, \Phi\} = -(M\Phi + \Phi M)
    \end{equation}
\end{theorem}
\begin{corollary}[Projector Dynamics]
\begin{equation}\dot{E}_e = \sum_{e \neq f} \frac{\lambda_e + \lambda_f}{\lambda_e - \lambda_f} \left( E_e \Phi E_f + E_e \Phi E_f \right)\end{equation}
\end{corollary}

\begin{corollary}[Eigenvalue Drift]

\begin{equation}\dot{\lambda}_e = -2\lambda_e \frac{\operatorname{tr}(E_e \Phi E_e)}{\dim(E_e)}\end{equation}
\end{corollary}
\begin{corollary}[Spectral Mass Transport]Define $q_{i,e}(t) = (E_e(t))_{ii} = \lambda_e^+p_{i,e}M_{ii}$, representing the mass of feature $i$ supported on eigenspace $e$. As such, its time evolution is given by:
    \begin{equation}
        \dot{q}_{i, e} = \sum_{e \neq f} \mathcal{T}_{e \to f}^{(i)}= \sum_{e\neq f} 2 \left( \frac{\lambda_e + \lambda_f}{\lambda_e - \lambda_f} \right) (E_e \Phi E_f)_{ii}\end{equation}
\end{corollary}

\begin{corollary}[Stability of Capacity Saturation] A configuration is a spectral fixed Point if the gradient kernel $\Phi$ commutes with the Gram operator spectral projector $[E_e, \Phi] = 0 \quad \forall \lambda$. This condition is satisfied by Tight Frames under uniform sparsity, implying 
\end{corollary}

\label{appendix:gradient_flow}
Let $W\in \mathbb{R}^{m\times n}$ be the weight matrix and $M = W^\top W\in \mathbb{R}^{n\times n}$ be its Gram matrix. The standard gradient flow on weights is $\dot W = - \nabla_W L$. We assume gradients defined using the ordinary dot product on the parameter coordinates, i.e. treating the parameters as living in flat $\mathbb{R}^{mn}$
 with no curvature or weighting. Using the Frobenius inner product $\langle A, B\rangle_F = \text{tr}(A^\top B)$, we have: $$dL = \sum_{a,b} \frac{\partial L}{\partial W_{ab}}\, dW_{ab} = \text{tr}\left((\nabla_W L)^\top dW\right) =\langle \nabla_W L, dW\rangle_F$$
Let us derive the induced flow on $M$. First, let us compute the differential:$$dM = d(W^\top W) = dW^\top W + W^\top dW$$. 
\begin{align*}\langle \nabla_M L, dM\rangle_F = \text{tr}\left((\nabla_M L)^\top dM\right)= \\=\text{tr}\left((\nabla_ML)^\top (dW^\top W+W^\top dW\right))=\\=\text{tr}\left((\nabla_ML)^\top dW^\top W\right) + \text{tr}\left((\nabla_ML)^\top W^\top dW\right)\end{align*}
In order to convert this into a Frobenius inner product with $\nabla W$ on the right, we can use the cyclic properties of the transpose, as well as $\text{tr}(A) = \text{tr}(A^\top)$:
\begin{align*}\text{tr}\left(W^\top dW\nabla_ML\right) = \text{tr}((\nabla_ML)W^\top dW) =\\= \text{tr}((W(\nabla_ML)^\top)^\top dW) = \langle W(\nabla_ML)^\top , dW)\end{align*}
Similarly, for the right hand side, we have: $$\text{tr}((\nabla_ML)^\top W^\top dW) = \text{tr}((W\nabla_ML)^\top dW) = \langle W\nabla_M L, dW\rangle$$$$\Rightarrow dL = \langle W(\nabla_ML)^\top + W\nabla_M L, dW\rangle = \langle \nabla_W L, dW\rangle$$$$\Rightarrow \nabla_WL =  W\left(\nabla_ML+ (\nabla_ML)^\top\right)$$  As such, the induced flow on $M$ is: $$\dot{M} = \dot{W}^\top W + W^\top \dot W = -\left(\nabla_W L\right)^\top W - W^\top (\nabla_W L) = $$
\begin{align*}=-(\nabla_ML + (\nabla_ML)^\top )^\top W^\top W \\-W^\top W(\nabla_ML + (\nabla_ML)^\top ) \\
=-(\Psi+ \Psi^\top ) M - M(\Psi + \Psi^\top)\end{align*}
This in fact shows us that only the symmetrized gradient drives the Gram flow (which should be expected per the Frobenius inner product), and it should tell us that we need to define the kernel gradient as:$$\Phi:=2\text{Sym}(\Psi) =(\Psi + \Psi^\top),$$
As such, our induced gradient flow becomes: 
\begin{equation} \dot M=-\{\Phi, M\} = - (\Phi M + M\Phi) = -\{M, \Phi\}\end{equation}
driven by the Gradient Kernel $\Phi = \text{Sym}(\nabla_M L)$. 

\subsection{Explicit Form} \label{lemma:gradient-kernel}
Recall the MSE loss function:

$$
L = \frac{1}{2} \sum_{x} p(x) \sum_{i} I_i \left( x_i - \text{ReLU}\left(\sum_{k} M_{ik}x_k + b_i\right) \right)^2
$$

By the chain rule, we have for $u_i = \sum_{k} M_{ik}x_k + b_i$ and $x'_i = \text{ReLU}(u_i)$:

$$
\frac{\partial L}{\partial M_{ij}} = \sum_{l} \frac{\partial L}{\partial x'_l} \cdot \frac{\partial x'_l}{\partial u_l} \cdot \frac{\partial u_l}{\partial M_{ij}},
$$

where

$$
\frac{\partial L}{\partial x'_l} = -I_l(x_l - x'_l), \quad \frac{\partial x'_l}{\partial u_l} = \mathbb{I}(u_l > 0), \quad \frac{\partial u_l}{\partial M_{ij}} = \frac{\partial}{\partial M_{ij}} \left( \sum_{k} M_{lk}x_k + b_k \right),
$$

where the latter is non-zero only when $l=i$. As such, in expectation form

$$
\Rightarrow \frac{\partial L}{\partial M_{ij}} = -\mathbb{E}_x \left[ I_i(x_i - x'_i) \cdot \mathbb{I}\left(\sum_{k} M_{ik}x_k + b_i > 0\right) \cdot x_j \right]
$$

To recover the matrix form, let us define the effective error vector:

$$
\boldsymbol{\delta}_{(x)} = I \odot (\mathbf{x} - \mathbf{x}') \odot \mathbb{I}(M\mathbf{x} + b > 0),
$$

where $\odot$ is the Hadamard product. Observe that the element $(\nabla_M L)_{ij}$ corresponds to the expectation of the product of the $i$-th error term $\delta_i$ and the $j$-th input $x_j$. Hence:

\begin{align}
\nabla_M L &= -\mathbb{E}_x [\boldsymbol{\delta}_{(x)}\mathbf{x}^\top] \notag \\
\Phi = -\mathbb{E}_x [\boldsymbol{\delta}_{(x)}\mathbf{x}^\top + \mathbf{x}\boldsymbol{\delta}_{(x)}^\top] &= -\sum_{x} p(x) \left[ \boldsymbol{\delta}_{(x)}\mathbf{x}^\top + \mathbf{x}\boldsymbol{\delta}_{(x)}^\top \right] \tag{2}
\end{align}

The sparsity regimes follow from the decomposition of $\Phi$ into diagonal (Feature benefit) $\Phi_{ii}$ and off-diagonal (interference) $\Phi_{ij}~(i \neq j)$ components. More specifically, observe that:

\begin{gather*}
\Phi_{ij} = -\sum_{x} p(x) (\delta_{(x),i}x_j + x_i\delta_{(x),j}) \\
\Rightarrow \Phi_{ii} = -2\sum_{x} p(x) \delta_{(x),i}x_i = -2\sum_{x} p(x) \left[ I_i(x_i - x'_i) \cdot \mathbb{I}(u_i > 0) \cdot x_i \right]
\end{gather*}

\subsection{Association Scheme Reduction}
\label{appendix:association_scheme}
\begin{lemma}[$\Gamma$-invariance of $M, G, \Phi$]
\end{lemma}
\noindent Per Lemma 4.5 and 4.7, we have the following expansions over the association scheme: 
$$M(t) = \sum_{r=0}^R \theta_r(t) A_t\quad , \quad \Phi(t) = \sum_{r=0}^R \varphi_r(t)A_r$$ 
Let us equate the coefficients using our induced gradient flow equation $(1)$ and substitute directly: 
$$M\Phi = \left(\sum_r \theta_r A_r\right)\left(\sum_s \varphi_s A_s\right) = \sum_{r, s}\theta_r\varphi_s(A_rA_s) = $$
$$=\sum_{r, s}\theta_r\varphi_s \sum_u c_{rs}^u A_u = \sum_u\left(\sum_{r, s} \theta_r \varphi_s c_{rs}^u\right)A_u$$
$$\Phi M = \sum_{r, s}\varphi_r\theta_s(A_rA_s) = \sum_{u}\left(\sum_{r, s}\varphi_r\theta_sc_{rs}^u\right)A_u$$
$$ \dot M = -(M\Phi+\Phi M) =$$
$$=-\sum_u \left(\sum_{r, s}\theta_t\varphi_sc_{rs}^u + \sum_{r, s}\varphi_r\theta_sc_{rs}^u\right)A_u$$
\begin{equation}\Rightarrow \dot\theta_u = -\sum_{r, s}\theta_r \varphi_sc_{rs}^u - \sum_{r, s}\varphi_r\theta_sc_{rs}^u = -\sum_{r, s}\theta_r\varphi_s(c_{rs}^u + c_{sr}^u)\end{equation}

\begin{lemmaproofbox}
\begin{theorem}[Gradient Flow (Theorem 6)]
Let $M=W^\top W$ and define the (symmetric) gradient kernel $\Phi=\mathrm{Sym}(\nabla_M L)$ with explicit form
\[
\Phi = -\mathbb E_x[\delta(x)x^\top + x\delta(x)^\top],
\qquad
\delta(x)= I\odot (x-x')\odot \mathbf 1(Mx+b>0),
\]
as in Appendix E.1 / equation (E.1). Then the induced Gram flow is
\[
\dot M = -\{M,\Phi\}=-(M\Phi+\Phi M).
\]
\end{theorem}
\begin{proof}
Let $W(t)$ follow weight-space gradient flow $\dot W=-\nabla_W L$. Differentiate $M=W^\top W$:
\[
\dot M = \dot W^\top W + W^\top \dot W
= -(\nabla_W L)^\top W - W^\top (\nabla_W L).
\]
To express $\nabla_W L$ in terms of $\nabla_M L$, use $dM=dW^\top W + W^\top dW$ and Frobenius inner products:
\[
dL=\langle \nabla_M L, dM\rangle
=\langle \nabla_M L, dW^\top W + W^\top dW\rangle
=\langle W(\nabla_M L + (\nabla_M L)^\top), dW\rangle.
\]
Hence
\[
\nabla_W L = W(\nabla_M L + (\nabla_M L)^\top)=W\Phi,
\]
where $\Phi=(\nabla_M L+(\nabla_M L)^\top)=2\,\mathrm{Sym}(\nabla_M L)$.
Substituting into $\dot M$ gives
\[
\dot M = -(W\Phi)^\top W - W^\top(W\Phi)
=-(\Phi W^\top W + W^\top W\Phi)
=-(\Phi M + M\Phi)
=-\{M,\Phi\}.
\]
The explicit form of $\Phi$ is computed in Appendix E.1, yielding the stated expression.
\end{proof}
\end{lemmaproofbox}

\begin{lemmaproofbox}
\begin{theorem}[Projector Dynamics (Corollary 7)]
Let $M=\sum_e \lambda_e E_e$ be the spectral decomposition of $M$, with $E_e$ the orthogonal spectral projectors.
Along the Gram flow $\dot M=-\{M,\Phi\}$,
\[
\dot E_e
= \sum_{f\neq e}\frac{\lambda_e+\lambda_f}{\lambda_f-\lambda_e}\,\big(E_e\Phi E_f + E_f\Phi E_e\big).
\]
\end{theorem}
\begin{proof}
Apply Kato's first-order eigenprojection drift formula (Lemma B.7) with $M^{(1)}=\dot M$.
For an isolated eigenvalue $\lambda_e$ with projector $E_e$, Lemma B.7 gives
\[
\dot E_e = -E_e \dot M\, S_{\lambda_e} - S_{\lambda_e}\dot M\, E_e,
\]
where for symmetric $M$ the reduced resolvent is (Lemma B.6)
\[
S_{\lambda_e}=\sum_{f\neq e}\frac{1}{\lambda_f-\lambda_e}\,E_f.
\]
Substitute $\dot M=-\{M,\Phi\}=-(M\Phi+\Phi M)$ to get
\[
\dot E_e = E_e(M\Phi+\Phi M)S_{\lambda_e} + S_{\lambda_e}(M\Phi+\Phi M)E_e.
\]
Using $ME_f=\lambda_f E_f$ and $E_eM=\lambda_e E_e$, for $f\neq e$ we have
\[
E_e(M\Phi+\Phi M)E_f
=E_eM\Phi E_f + E_e\Phi M E_f
=(\lambda_e+\lambda_f)\,E_e\Phi E_f.
\]
Therefore
\[
E_e(M\Phi+\Phi M)S_{\lambda_e}
=\sum_{f\neq e}\frac{\lambda_e+\lambda_f}{\lambda_f-\lambda_e}\,E_e\Phi E_f,
\]
and similarly
\[
S_{\lambda_e}(M\Phi+\Phi M)E_e
=\sum_{f\neq e}\frac{\lambda_e+\lambda_f}{\lambda_f-\lambda_e}\,E_f\Phi E_e.
\]
Adding yields the stated formula.
\end{proof}
\end{lemmaproofbox}

\begin{lemmaproofbox}
\begin{theorem}[Eigenvalue Drift (Corollary 8)]
Along the Gram flow $\dot M=-\{M,\Phi\}$, the eigenvalues satisfy
\[
\dot\lambda_e = -2\lambda_e\,\frac{\mathrm{tr}(E_e\Phi E_e)}{\dim(E_e)}.
\]
\end{theorem}
\begin{proof}
Lemma B.8 gives the first-order eigenvalue drift for eigenvalue $\lambda_e$ with spectral projector $E_e$ of rank
$d_e=\mathrm{tr}(E_e)$:
\[
\dot\lambda_e = \frac{1}{d_e}\mathrm{tr}(\dot M\,E_e).
\]
Now $\dot M=-(M\Phi+\Phi M)$, so
\[
\mathrm{tr}(\dot M E_e) = -\mathrm{tr}(M\Phi E_e)-\mathrm{tr}(\Phi M E_e).
\]
Using cyclicity and $ME_e=\lambda_e E_e$,
\[
\mathrm{tr}(M\Phi E_e)=\mathrm{tr}(\Phi E_e M)=\lambda_e\,\mathrm{tr}(\Phi E_e),
\qquad
\mathrm{tr}(\Phi M E_e)=\mathrm{tr}(\Phi \lambda_e E_e)=\lambda_e\,\mathrm{tr}(\Phi E_e).
\]
Also $\mathrm{tr}(\Phi E_e)=\mathrm{tr}(E_e\Phi E_e)$ since cross-terms vanish under the projector decomposition.
Hence
\[
\mathrm{tr}(\dot M E_e) = -2\lambda_e\,\mathrm{tr}(E_e\Phi E_e),
\]
and dividing by $d_e=\dim(E_e)$ gives the claim.
\end{proof}
\end{lemmaproofbox}

\begin{lemmaproofbox}
\begin{theorem}[Spectral Mass Transport (Corollary 9)]
Define
\[
q_{i,e}(t):=(E_e(t))_{ii}=\lambda_e^+\,p_{i,e}(t)\,M_{ii}(t),
\]
interpreted as the “mass” of coordinate/feature $i$ supported in eigenspace $e$. Then
\[
\dot q_{i,e}=\sum_{f\neq e} T^{(i)}_{e\to f},
\qquad
T^{(i)}_{e\to f}
:= 2\Big(\frac{\lambda_e+\lambda_f}{\lambda_e-\lambda_f}\Big)\,(E_e\Phi E_f)_{ii}.
\]
\end{theorem}
\begin{proof}
The identity $q_{i,e}=(E_e)_{ii}=\lambda_e^+\,p_{i,e}\,M_{ii}$ is
given by Lemma B.11(2), together with $M_{ii}=\|W_i\|^2$.

Differentiate $q_{i,e}=(E_e)_{ii}$:
\[
\dot q_{i,e} = (\dot E_e)_{ii}.
\]
Substitute Corollary~7. Using symmetry of $\Phi$ and $E_e$, note that for $f\neq e$,
\[
(E_f\Phi E_e)_{ii}
= e_i^\top E_f\Phi E_e e_i
= (E_f e_i)^\top \Phi (E_e e_i)
= (E_e e_i)^\top \Phi (E_f e_i)
= (E_e\Phi E_f)_{ii}.
\]
Thus the paired terms add, giving
\[
\dot q_{i,e}
=\sum_{f\neq e}2\Big(\frac{\lambda_e+\lambda_f}{\lambda_e-\lambda_f}\Big)\,(E_e\Phi E_f)_{ii}.
\]
Defining $T^{(i)}_{e\to f}$ as in the statement yields the transport form.
\end{proof}
\end{lemmaproofbox}

Let $\Gamma \le S_{n_i}$ act on indices of the block $\Omega_i$ and let
$\{P_\gamma\}_{\gamma\in\Gamma}$ be the associated permutation matrices.
Define the commutant (centralizer)
\[
\mathcal{A}_\Gamma := \{A\in\mathbb{R}^{n_i\times n_i}:\; P_\gamma A P_\gamma^\top = A \ \forall \gamma\in\Gamma\}.
\]
Equivalently, $\mathcal{A}_\Gamma=\{A:\;P_\gamma A = A P_\gamma\ \forall\gamma\}$.

\begin{lemmaproofbox}[Gradient equivariance under $\Gamma$-conjugation]\label{lem:grad-equivariance}
\begin{lemma}
Assume $L:\mathbb{R}^{n_i\times n_i}\to\mathbb{R}$ is Fr\'echet differentiable
(with respect to the Frobenius inner product) and satisfies
\[
L(P_\gamma M P_\gamma^\top)=L(M)\qquad \forall \gamma\in\Gamma,\ \forall M.
\]
Then the gradient is $\Gamma$-equivariant:
\[
\nabla L(P_\gamma M P_\gamma^\top)=P_\gamma\big(\nabla L(M)\big)P_\gamma^\top
\qquad \forall \gamma\in\Gamma,\ \forall M.
\]
\end{lemma}
\begin{proof}
Fix $\gamma\in\Gamma$ and define $\Psi(M):=P_\gamma M P_\gamma^\top$.
The invariance hypothesis is $L(\Psi(M))=L(M)$ for all $M$.
Differentiate at $M$ in an arbitrary direction $H$:
\[
\left.\frac{d}{d\varepsilon}\right|_{\varepsilon=0} L(\Psi(M+\varepsilon H))
=
\left.\frac{d}{d\varepsilon}\right|_{\varepsilon=0} L(M+\varepsilon H).
\]
Since $\Psi(M+\varepsilon H)=\Psi(M)+\varepsilon(P_\gamma H P_\gamma^\top)$, the definition of the gradient gives
\[
\big\langle \nabla L(\Psi(M)),\, P_\gamma H P_\gamma^\top\big\rangle_F
=
\big\langle \nabla L(M),\, H\big\rangle_F.
\]
Using Frobenius invariance under orthogonal conjugation,
$\langle A,\,P_\gamma H P_\gamma^\top\rangle_F=\langle P_\gamma^\top A P_\gamma,\,H\rangle_F$,
we obtain
\[
\big\langle P_\gamma^\top \nabla L(\Psi(M)) P_\gamma,\,H\big\rangle_F
=
\big\langle \nabla L(M),\,H\big\rangle_F
\qquad\forall H.
\]
Hence $P_\gamma^\top \nabla L(P_\gamma M P_\gamma^\top) P_\gamma=\nabla L(M)$, which is equivalent to the claim.
\end{proof}
\end{lemmaproofbox}

\begin{lemmaproofbox}[$\mathcal{A}_\Gamma$-membership of the (symmetrized) gradient]
\begin{lemma}
Assume the hypotheses of Lemma~\ref{lem:grad-equivariance}.
If $M\in\mathcal{A}_\Gamma$, then $\nabla L(M)\in\mathcal{A}_\Gamma$.
Consequently,
\[
\Phi(M):=\mathrm{Sym}(\nabla L(M)):=\tfrac12\big(\nabla L(M)+\nabla L(M)^\top\big)
\in\mathcal{A}_\Gamma.
\]
\end{lemma}
\begin{proof}
If $M\in\mathcal{A}_\Gamma$, then $P_\gamma M P_\gamma^\top=M$ for all $\gamma$.
Lemma~\ref{lem:grad-equivariance} yields
\[
\nabla L(M)=\nabla L(P_\gamma M P_\gamma^\top)=P_\gamma(\nabla L(M))P_\gamma^\top,
\]
so $\nabla L(M)\in\mathcal{A}_\Gamma$.
Since $P_\gamma A P_\gamma^\top=A$ implies $P_\gamma A^\top P_\gamma^\top=A^\top$,
we also have $\nabla L(M)^\top\in\mathcal{A}_\Gamma$.
Because $\mathcal{A}_\Gamma$ is a linear subspace, their average $\mathrm{Sym}(\nabla L(M))$ belongs to $\mathcal{A}_\Gamma$.
\end{proof}
\end{lemmaproofbox}

\begin{lemmaproofbox}{Algebraic closure of $\mathcal{A}_\Gamma$}\label{lem:commutant-algebra}
$\mathcal{A}_\Gamma$ is a matrix algebra: if $X,Y\in\mathcal{A}_\Gamma$ then
$X+Y\in\mathcal{A}_\Gamma$, $\alpha X\in\mathcal{A}_\Gamma$ for all $\alpha\in\mathbb{R}$,
and $XY\in\mathcal{A}_\Gamma$.
Moreover, if $X\in\mathcal{A}_\Gamma$ then $X^\top\in\mathcal{A}_\Gamma$, hence $\mathrm{Sym}(X)\in\mathcal{A}_\Gamma$.
\tcblower
\begin{proof}
Linearity is immediate from the defining relation $P_\gamma(\cdot)P_\gamma^\top=(\cdot)$.
For products, if $X,Y\in\mathcal{A}_\Gamma$, then for every $\gamma\in\Gamma$,
\[
P_\gamma(XY)P_\gamma^\top=(P_\gamma X P_\gamma^\top)(P_\gamma Y P_\gamma^\top)=XY,
\]
so $XY\in\mathcal{A}_\Gamma$.
For transpose, if $P_\gamma X P_\gamma^\top=X$, then taking transpose gives
$P_\gamma X^\top P_\gamma^\top=X^\top$ (since $P_\gamma^\top=P_\gamma^{-1}$), so $X^\top\in\mathcal{A}_\Gamma$.
Finally, $\mathrm{Sym}(X)=\tfrac12(X+X^\top)$ belongs to $\mathcal{A}_\Gamma$ by linearity.
\end{proof}
\end{lemmaproofbox}

\begin{lemmaproofbox}
\begin{lemma}[Invariant-block closure of Gram flow]\label{cor:invariant-block-closure}
Assume $\Phi(t)\in\mathcal{A}_\Gamma$ for all $t$ and consider the Gram flow
\[
\dot M(t)=-(M(t)\Phi(t)+\Phi(t)M(t)).
\]
If $M(t_0)\in\mathcal{A}_\Gamma$ at some time $t_0$, then $M(t)\in\mathcal{A}_\Gamma$ for all $t\ge t_0$.
\end{lemma}
\begin{proof}
Since $M(t_0),\Phi(t_0)\in\mathcal{A}_\Gamma$ and $\mathcal{A}_\Gamma$ is an algebra
(Lemma~\ref{lem:commutant-algebra}), we have $M(t_0)\Phi(t_0)\in\mathcal{A}_\Gamma$ and
$\Phi(t_0)M(t_0)\in\mathcal{A}_\Gamma$, hence $\dot M(t_0)\in\mathcal{A}_\Gamma$.
More generally, whenever $M(t)\in\mathcal{A}_\Gamma$ we get $\dot M(t)\in\mathcal{A}_\Gamma$ by the same closure,
so the vector field is tangent to the linear subspace $\mathcal{A}_\Gamma$ everywhere on $\mathcal{A}_\Gamma$.
Therefore the solution starting at $M(t_0)\in\mathcal{A}_\Gamma$ remains in $\mathcal{A}_\Gamma$ for all $t\ge t_0$.
\end{proof}
\end{lemmaproofbox}

\begin{lemmaproofbox}
\begin{lemma}[Invariant-block closure of Gram flow]\label{corollary:stability}
Assume $\Phi(t)\in\mathcal{A}_\Gamma$ for all $t$ and consider the Gram flow
\[
\dot M(t)=-(M(t)\Phi(t)+\Phi(t)M(t)).
\]
If $M(t_0)\in\mathcal{A}_\Gamma$ at some time $t_0$, then $M(t)\in\mathcal{A}_\Gamma$ for all $t\ge t_0$.
\end{lemma}
\begin{proof}
Fix any $\gamma\in\Gamma$ and define $\widetilde M(t):=P_\gamma M(t)P_\gamma^\top$.
Then
\[
\dot{\widetilde M}(t)
= P_\gamma \dot M(t)P_\gamma^\top
= -\big(\widetilde M(t)\,\widetilde\Phi(t)+\widetilde\Phi(t)\,\widetilde M(t)\big),
\]
where $\widetilde\Phi(t):=P_\gamma \Phi(t)P_\gamma^\top$.
Since $\Phi(t)\in\mathcal{A}_\Gamma$, we have $\widetilde\Phi(t)=\Phi(t)$ for all $t$, hence
\[
\dot{\widetilde M}(t)=-(\widetilde M(t)\Phi(t)+\Phi(t)\widetilde M(t)).
\]
Moreover, because $M(t_0)\in\mathcal{A}_\Gamma$, we have
$\widetilde M(t_0)=P_\gamma M(t_0)P_\gamma^\top=M(t_0)$.
Thus $M$ and $\widetilde M$ solve the same ODE with the same initial condition at $t_0$.
By uniqueness of solutions to linear matrix ODEs, $\widetilde M(t)=M(t)$ for all $t\ge t_0$.
Therefore $P_\gamma M(t)P_\gamma^\top=M(t)$ for all $\gamma\in\Gamma$ and all $t\ge t_0$, i.e.
$M(t)\in\mathcal{A}_\Gamma$ for all $t\ge t_0$.
\end{proof}
\end{lemmaproofbox}
\begin{lemmaproofbox}
\begin{theorem}[Stability of Capacity Saturation (Corollary 10)]
A configuration is a spectral fixed point if the gradient kernel $\Phi$ commutes with every Gram spectral projector:
\[
[E_e,\Phi]=0\quad \forall e.
\]
In particular, tight-frame geometries under uniform sparsity satisfy this condition (hence are spectral fixed points).
\end{theorem}
\begin{proof}
If $[E_e,\Phi]=0$ for all $e$, then for $f\neq e$,
\[
E_e\Phi E_f = E_eE_f\Phi = 0,
\]
since commuting gives $E_e\Phi=\Phi E_e$ and orthogonality gives $E_eE_f=0$.
Plugging into Corollary~7 yields $\dot E_e=0$ for all $e$. Consequently, Corollary~9 gives $\dot q_{i,e}=0$ for all $i,e$:
there is no spectral mass transport between eigenspaces, so the spectral decomposition is dynamically stable.

For the tight-frame/uniform-sparsity claim: Appendix E.2 (Lemma E.1) expands both $M(t)$ and $\Phi(t)$ in the same
association-scheme (orbital) basis $\{A_r\}$:
\[
M(t)=\sum_{r=0}^R \theta_r(t)A_r,\qquad \Phi(t)=\sum_{r=0}^R \phi_r(t)A_r.
\]
When the centralizer is a Bose--Mesner algebra (commutative association scheme), all $A_r$ commute, hence $M(t)$ and $\Phi(t)$ commute.
In a symmetric commuting family, $\Phi$ commutes with the spectral projectors $\{E_e\}$ of $M$ (simultaneous diagonalization),
so $[E_e,\Phi]=0$ for all $e$, i.e.\ a spectral fixed point.
\end{proof}
\end{lemmaproofbox}

\clearpage
\section{Non-uniform Sparsity}
\label{appendix:non_uniform_sparsity}
We observe the same type of projective linearity for uniform sparsity:

\noindent\begin{minipage}{\linewidth}
    \centering
    \begin{minipage}{0.48\textwidth}
        \centering
        \includegraphics[width=\textwidth]{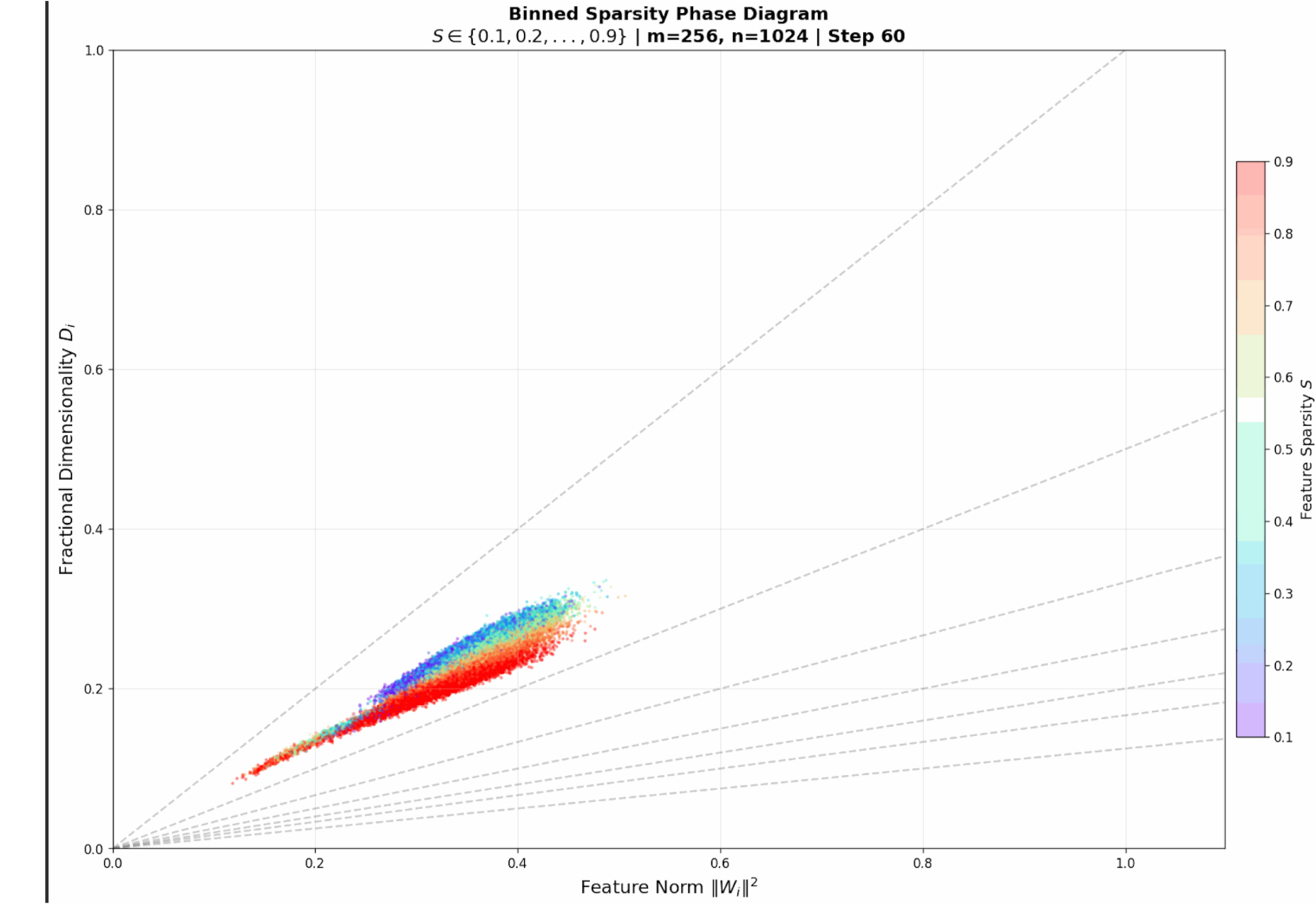}
        \captionof{figure}{Description of image 1}
        \label{fig:nu_1}
    \end{minipage}
    \hfill
    \begin{minipage}{0.48\textwidth}
        \centering
        \includegraphics[width=\textwidth]{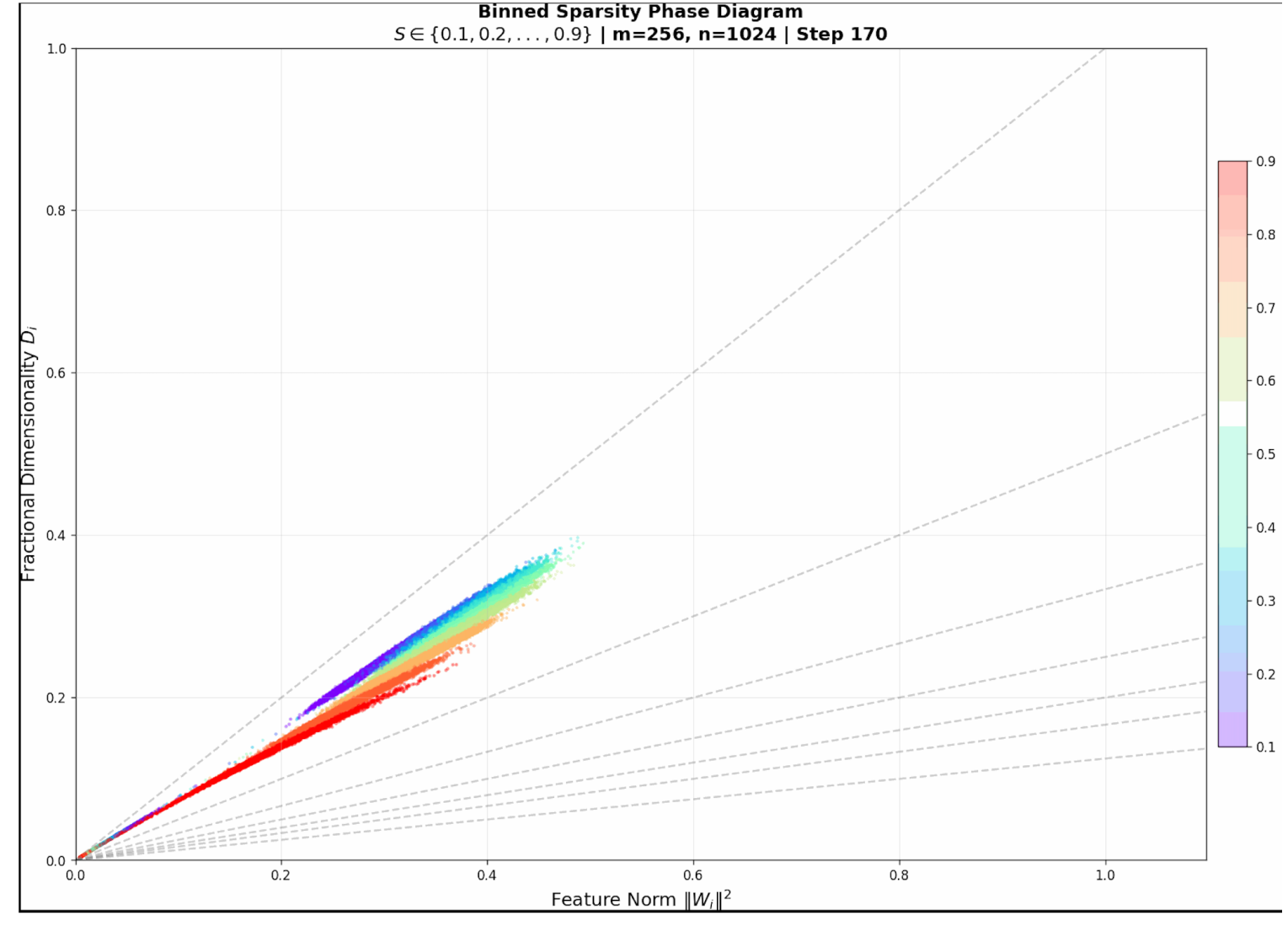}
        \captionof{figure}{Description of image 2}
        \label{fig:nu_2}
    \end{minipage}
\end{minipage}
\vspace{2em}

\noindent\begin{minipage}{\linewidth}
    \centering
    \begin{minipage}{0.48\textwidth}
        \centering
        \includegraphics[width=\textwidth]{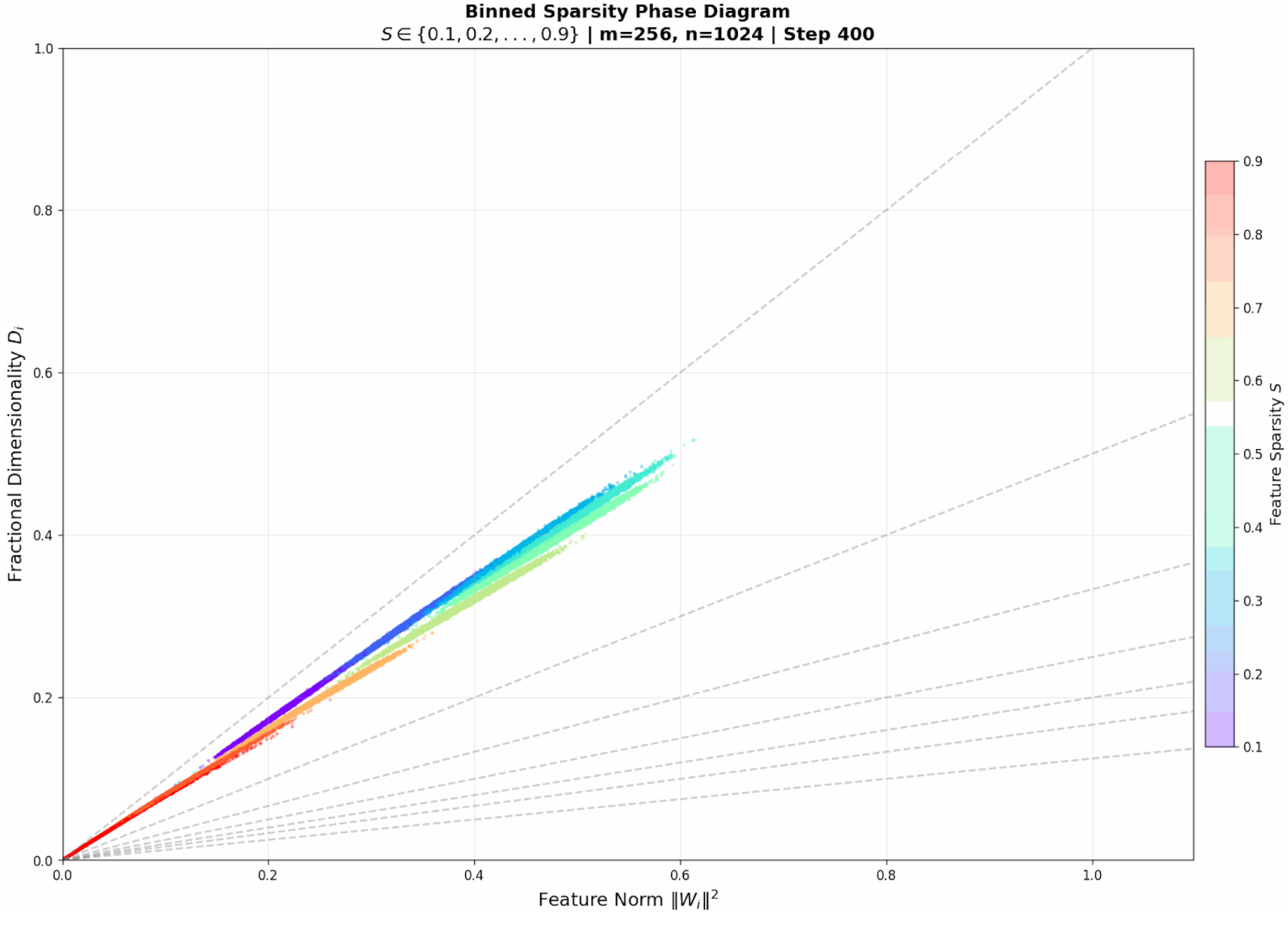}
        \captionof{figure}{Description of image 3}
        \label{fig:nu_3}
    \end{minipage}
    \hfill
    \begin{minipage}{0.48\textwidth}
        \centering
        \includegraphics[width=\textwidth]{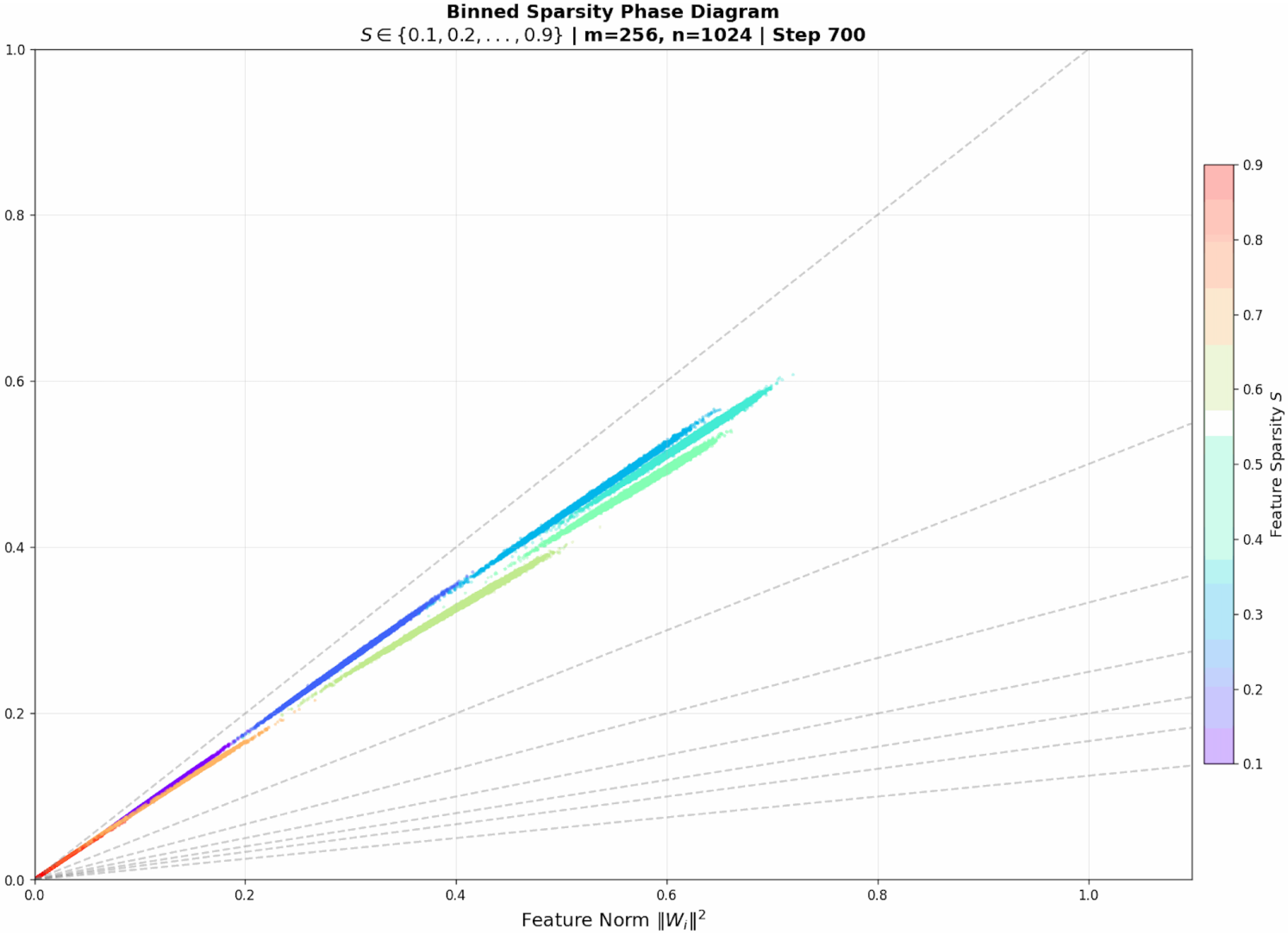}
        \captionof{figure}{Description of image 4}
        \label{fig:nu_4}
    \end{minipage}
\end{minipage}
\vspace{2em}

\noindent\begin{minipage}{\linewidth}
    \centering
    \begin{minipage}{0.48\textwidth}
        \centering
        \includegraphics[width=\textwidth]{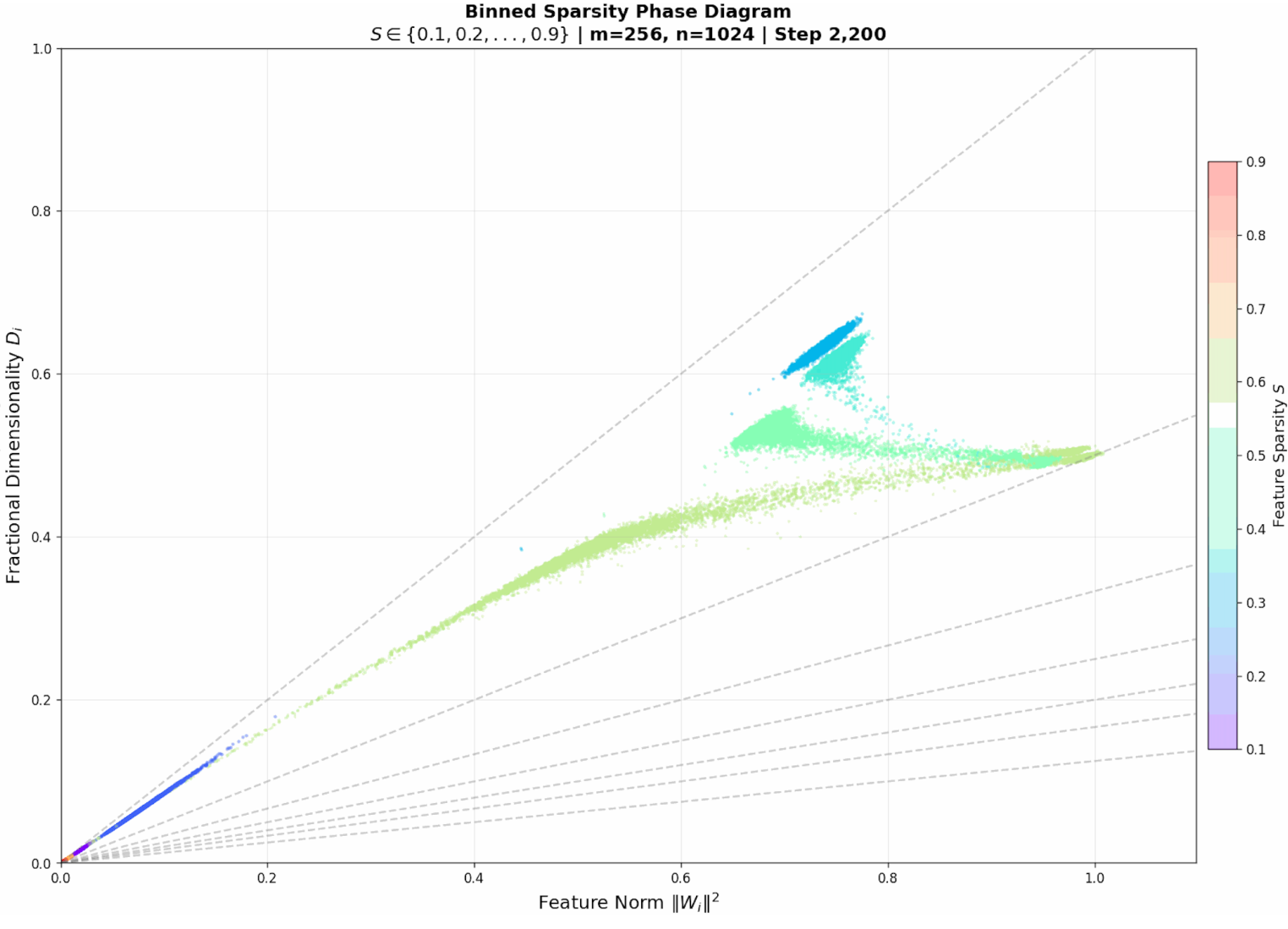}
        \captionof{figure}{Description of image 5}
        \label{fig:nu_5}
    \end{minipage}
    \hfill
    \begin{minipage}{0.48\textwidth}
        \centering
        \includegraphics[width=\textwidth]{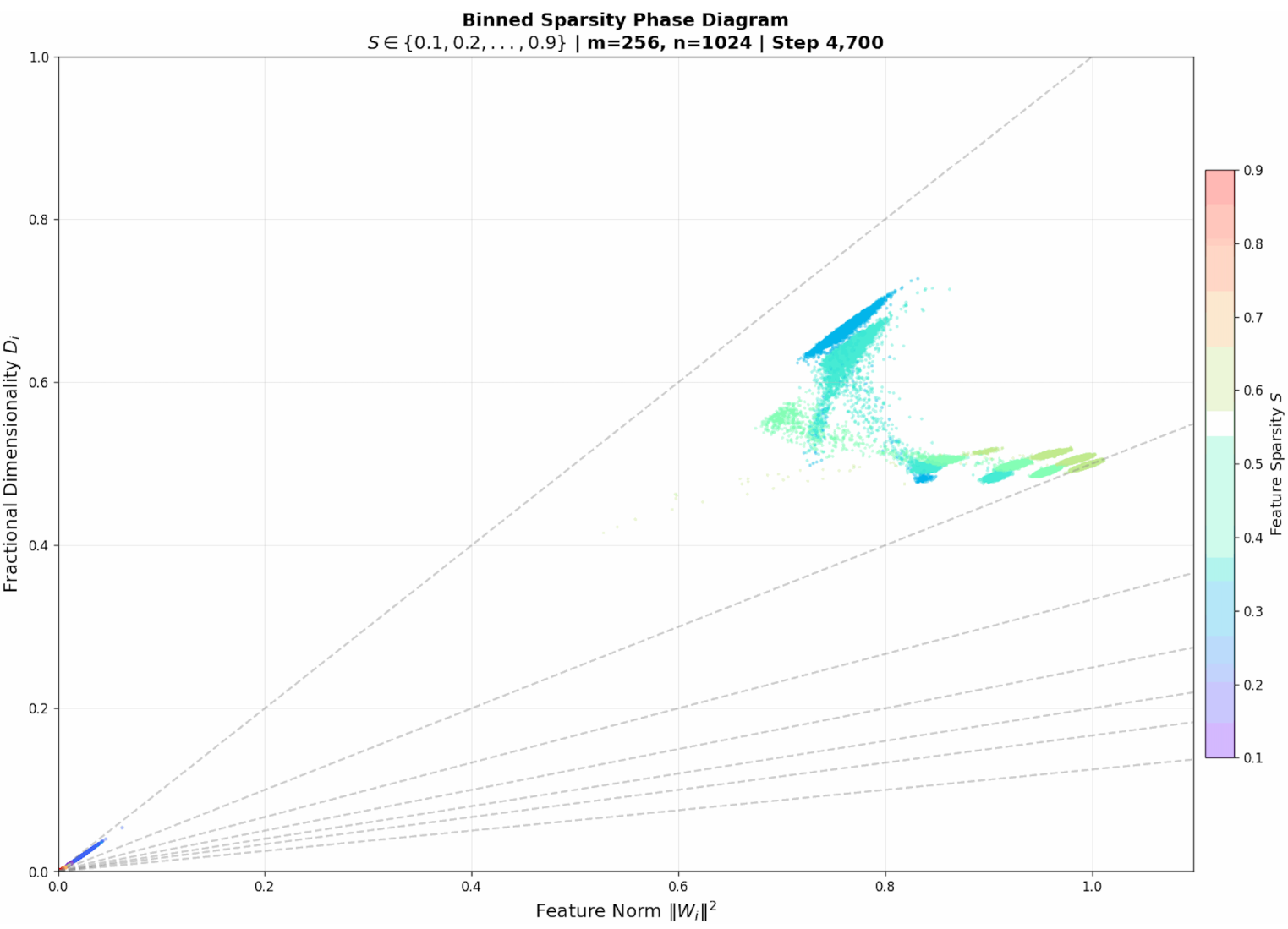}
        \captionof{figure}{Description of image 6}
        \label{fig:nu_6}
    \end{minipage}
\end{minipage}
\vspace{2em}

\noindent\begin{minipage}{\linewidth}
    \centering
    \begin{minipage}{0.48\textwidth}
        \centering
        \includegraphics[width=\textwidth]{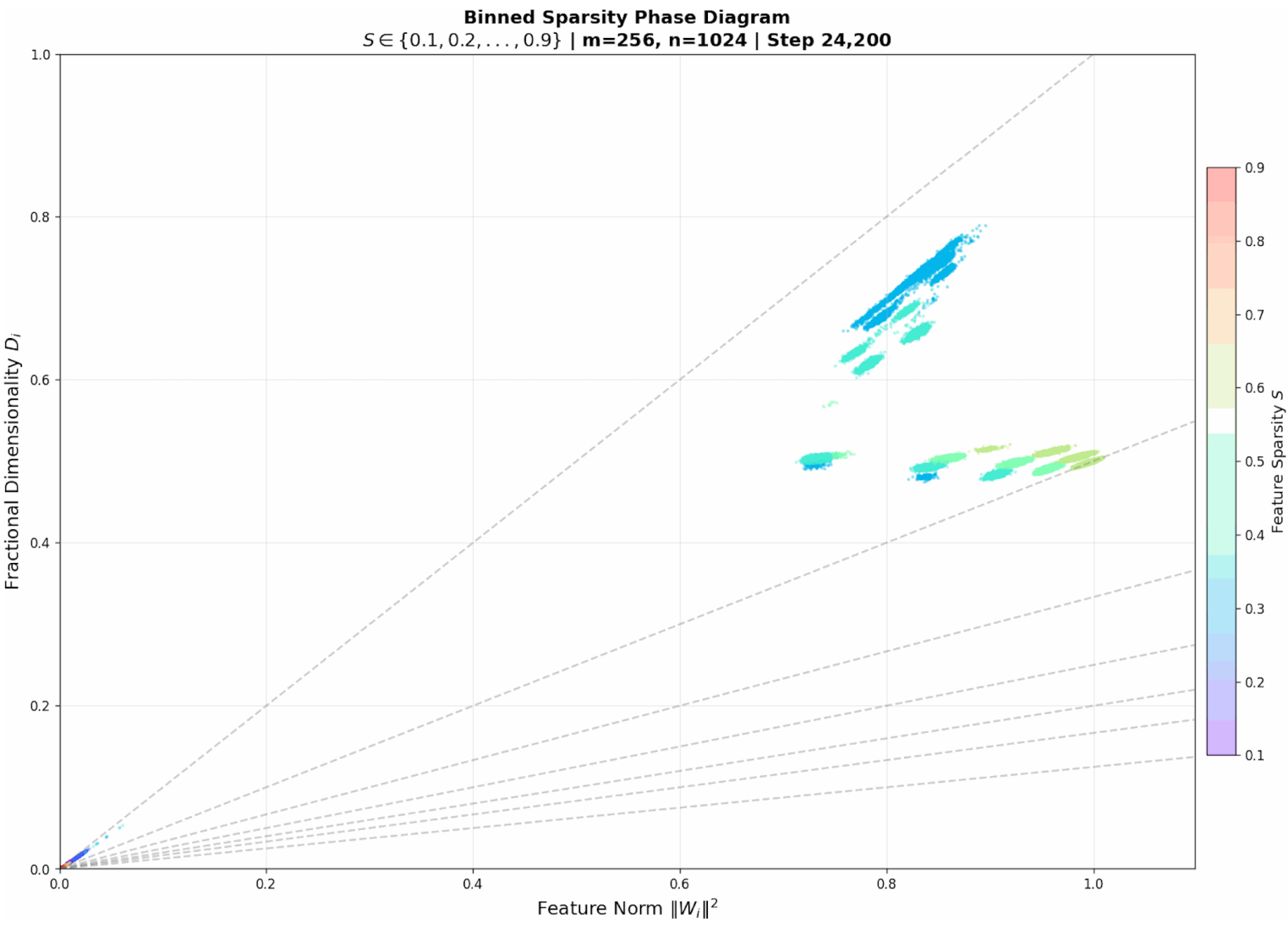}
        \captionof{figure}{Description of image 7}
        \label{fig:nu_7}
    \end{minipage}
\end{minipage}
\vspace{2em}

\clearpage
\section{Related Work}
\label{appendix:related_work}

Before we begin our rigorous examination of the latter model, let us provide context for other relevant work on polysemanticity:
\begin{itemize}

\item \textbf{Lecomte et al. (2023/2024) \cite{lecomte2023incidental} Incidental Polysemanticity}
Lecomte et al.\ argue polysemanticity can arise \emph{incidentally}, even when capacity suffices for monosemantic codes. Using theory and experiments, they show random initialization, regularization, and neural noise can fortuitously align multiple features to one neuron early in training, after which gradient dynamics reinforce the overlap. This complements capacity-driven superposition: polysemanticity can be a stable attractor without being strictly necessary for performance.

\item \textbf{Adler \& Shavit (2024/2025) \cite{adler2024complexity} — Complexity of Computing in Superposition}
Adler and Shavit develop complexity-theoretic bounds for \emph{computing} with superposed features, proving that a broad class of tasks (e.g., permutations, pairwise logic) needs at least $\Omega(\sqrt{m'\log m'})$ neurons and $\Omega(m'\log m')$ parameters to compute $m'$ outputs in superposition. They also provide near-matching constructive upper bounds (e.g., pairwise AND with $O(\sqrt{m'}\log m')$ neurons), revealing large but not unbounded efficiency gains. The results distinguish “representing” features in superposition (much cheaper) from “computing” with them (provably costlier), setting principled limits that any interpretability or compression method must respect.

\item \textbf{Klindt et al. (2025) \cite{klindt2025sparsecodes} — From Superposition to Sparse Codes}
Klindt et al.\ propose a principled route from superposed activations to interpretable factors by leveraging three components: identifiability (classification representations recover latent features up to an invertible linear transform), sparse coding/dictionary learning to find a concept-aligned basis, and quantitative interpretability metrics. In this view, deep nets often linearly overlay concepts; post-hoc sparse coding can invert the mixing and yield monosemantic directions without retraining the model.

\item \textbf{Hollard et al. (2025) \cite{hollard2025lowparam} Superposition in Low-Parameter Vision Models}
Hollard et al.\ study modern $<!1.5$M-parameter CNNs and show that bottleneck designs and superlinear activations exacerbate interference (feature overlap) in feature maps, limiting accuracy scaling. By systematically varying bottleneck structures, they identify design choices that reduce interference and introduce a “NoDepth Bottleneck” that improves ImageNet scaling within tight parameter budgets.

\item \textbf{Pertl et al. (2025)\cite{pertl2025gnn} — Superposition in GNNs}
Pertl et al.\ analyze superposition in graph neural networks by extracting feature directions at node and graph levels and studying their basis-invariant overlaps. They observe a width-driven phase pattern in overlap, topology-induced mixing at the node level, and mitigation via sharper pooling that increases axis alignment; shallow models can fall into metastable low-rank embeddings.

\item \textbf{Hesse et al. (2025) \cite{hesse2025disentangle} Disentangling Polysemantic Channels in CNNs}
Hesse et al.\ present an algorithmic surgery to split a polysemantic CNN channel into multiple channels, each responding to a single concept, by exploiting distinct upstream activation patterns that feed the mixed unit. The method rewires a pretrained network (without retraining) to produce explicit, monosemantic channels, improving clarity of feature visualizations and enabling standard mechanistic tools. Unlike “virtual” decompositions, this yields concrete network components aligned to concepts, demonstrating a practical path to reduce superposition post hoc.

\item \textbf{Dreyer et al. (2024) \cite{dreyer2024pure} PURE: Circuits-Based Decomposition}
Dreyer et al.\ introduce PURE, a post-hoc circuits method that decomposes a polysemantic neuron into multiple \emph{virtual} monosemantic units by identifying distinct upstream subgraphs (circuits) responsible for each concept. The approach improves concept-level visualizations (e.g., via CLIP-based evaluation) and does not modify weights, instead reattributing behavior across disjoint computation paths. PURE shows that much polysemanticity reflects circuit superposition; separating circuits yields purer conceptual units without architectural changes.
\end{itemize}

\clearpage
\section{Plots}
\label{appendix:plots}

\noindent\begin{minipage}{\linewidth}
    \centering
    \begin{minipage}{0.48\textwidth}
        \centering
        \includegraphics[width=\textwidth]{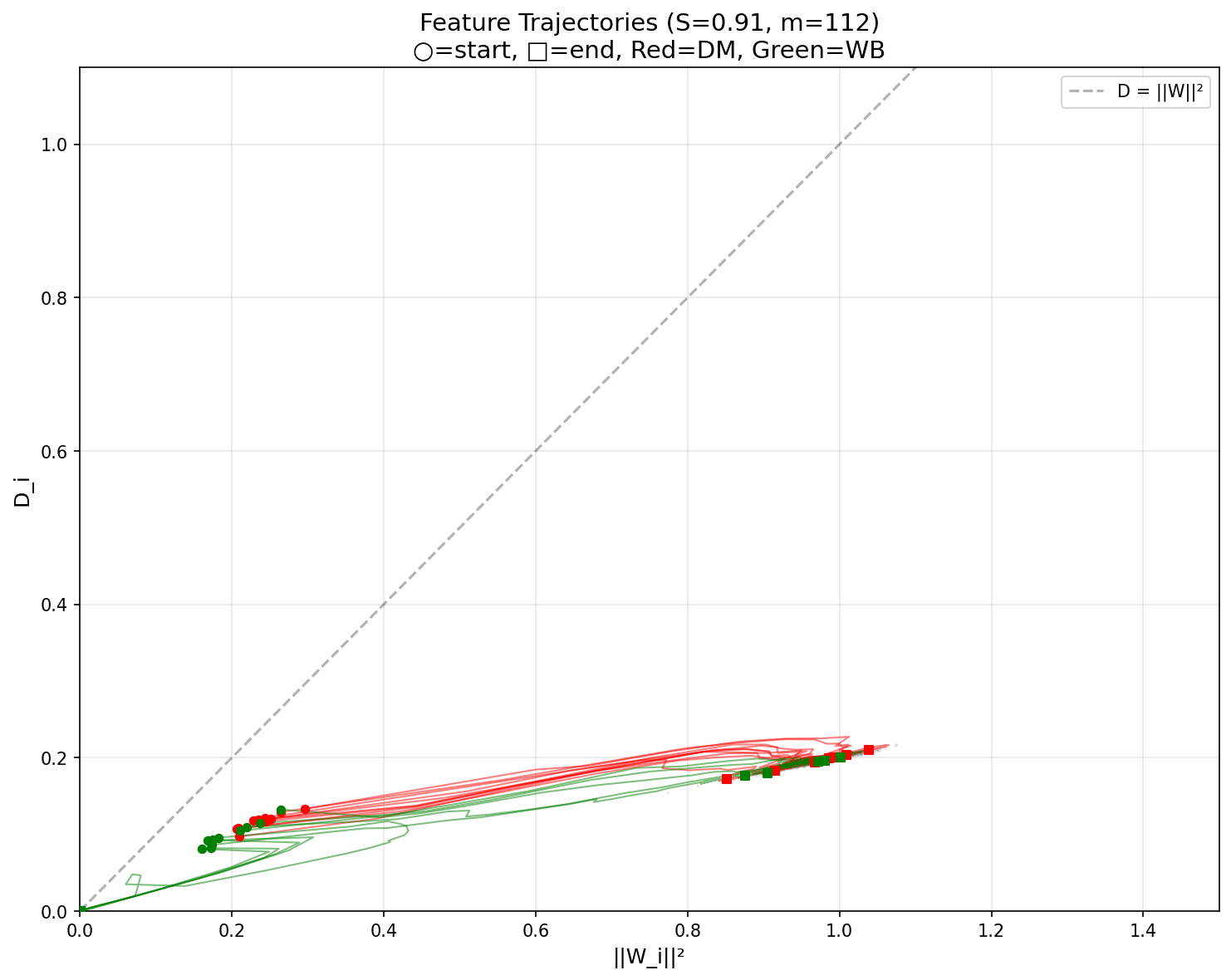}
        \captionof{figure}{Description of image 5}
        \label{fig:dark_matter_tracers}
    \end{minipage}
    \hfill
    \begin{minipage}{0.48\textwidth}
        \centering
        \includegraphics[width=\textwidth]{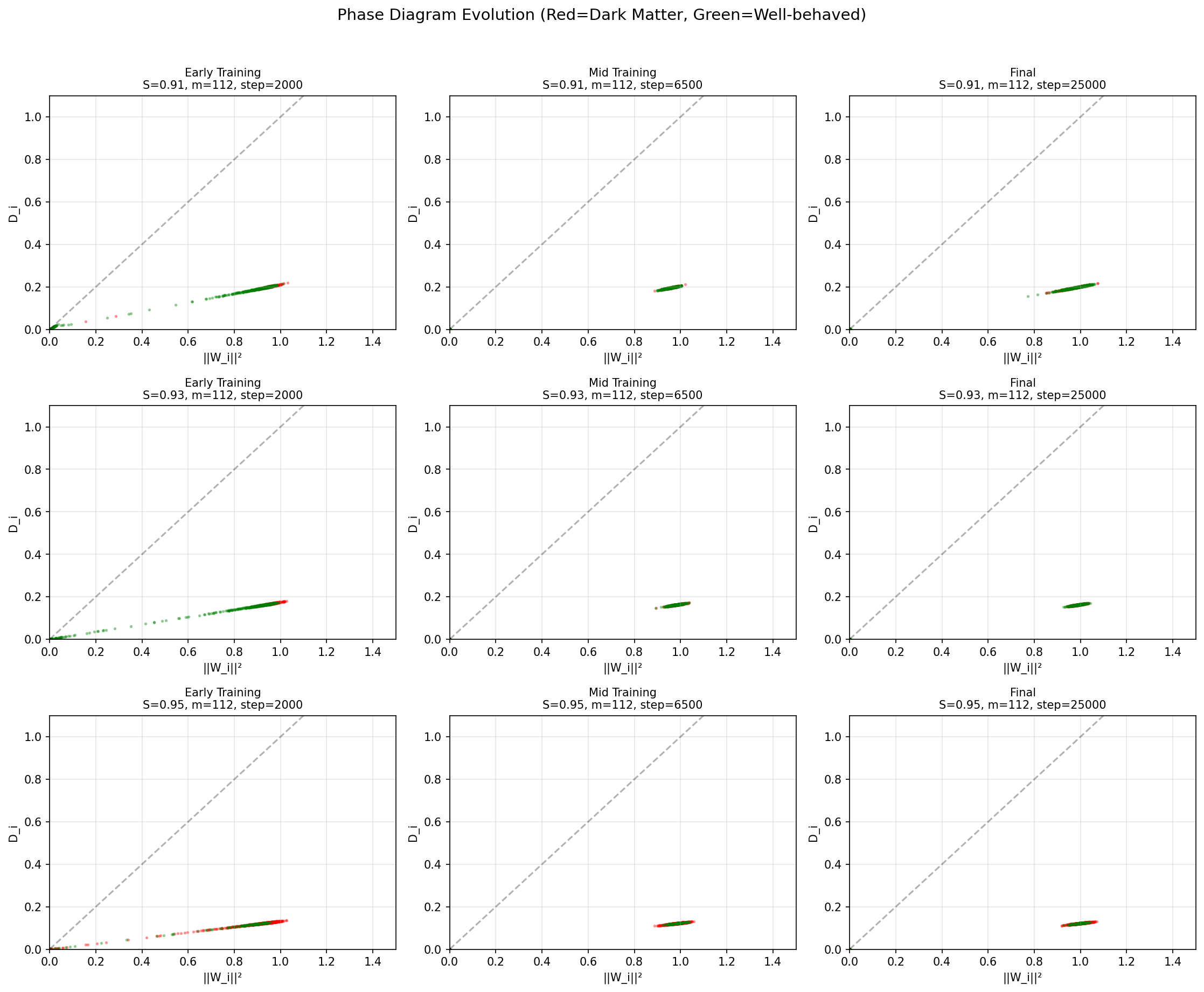}
        \captionof{figure}{Description of image 6}
        \label{fig:dark_matter_comparison}
    \end{minipage}
\end{minipage}
\vspace{2em}

\noindent\begin{minipage}{\linewidth}
    \centering
    \begin{minipage}{0.48\textwidth}
        \centering
        \includegraphics[width=\textwidth]{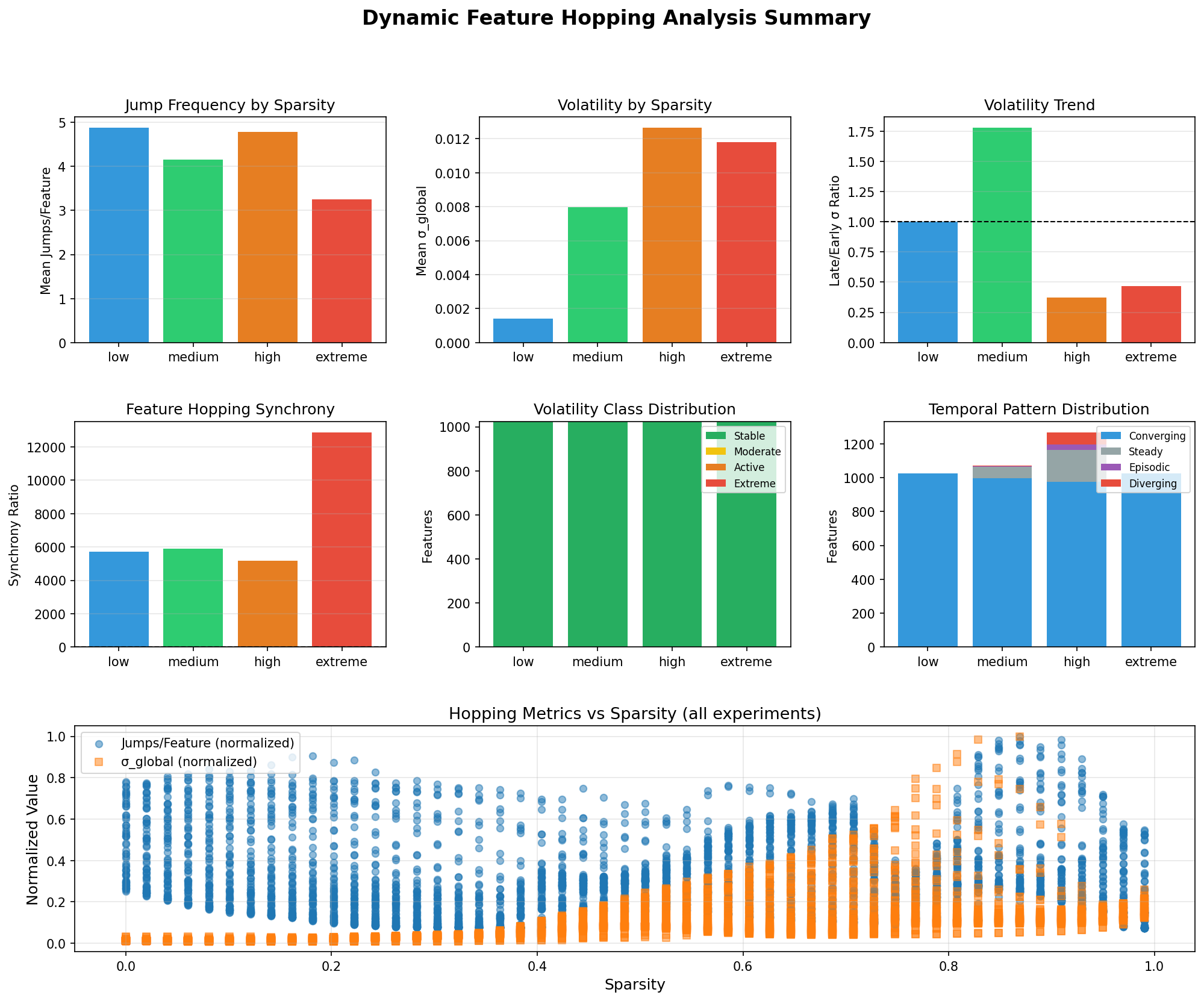}
        \captionof{figure}{Description of image 5}
        \label{fig:hopping_summary}
    \end{minipage}
    \hfill
    \begin{minipage}{0.48\textwidth}
        \centering
        \includegraphics[width=\textwidth]{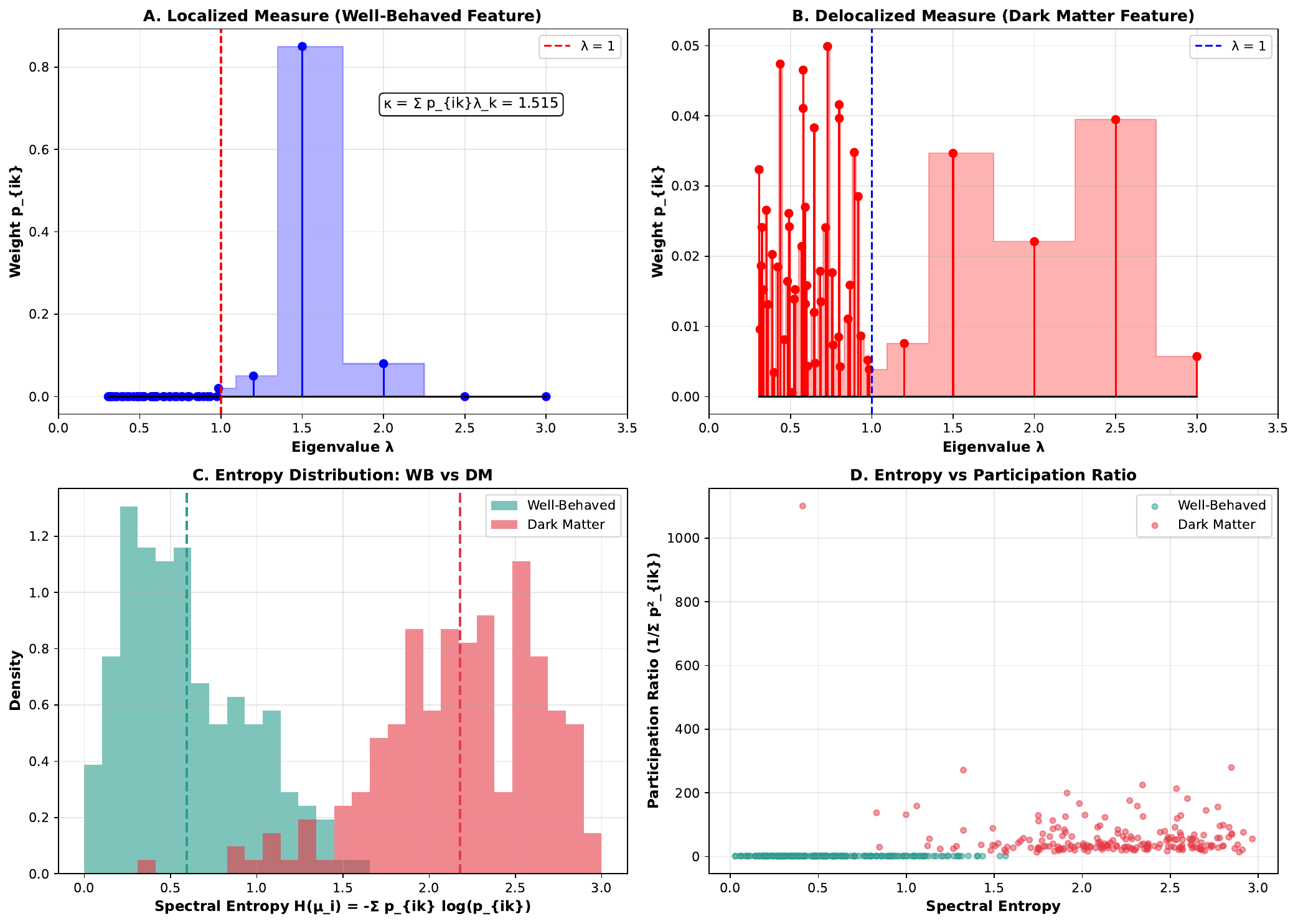}
        \captionof{figure}{Description of image 6}
        \label{fig:spectral_measure_ill}
    \end{minipage}
\end{minipage}
\vspace{2em}

\clearpage

\noindent\begin{minipage}{\linewidth}
    \centering
    \begin{minipage}{0.31\textwidth}
        \centering
        \includegraphics[width=\textwidth]{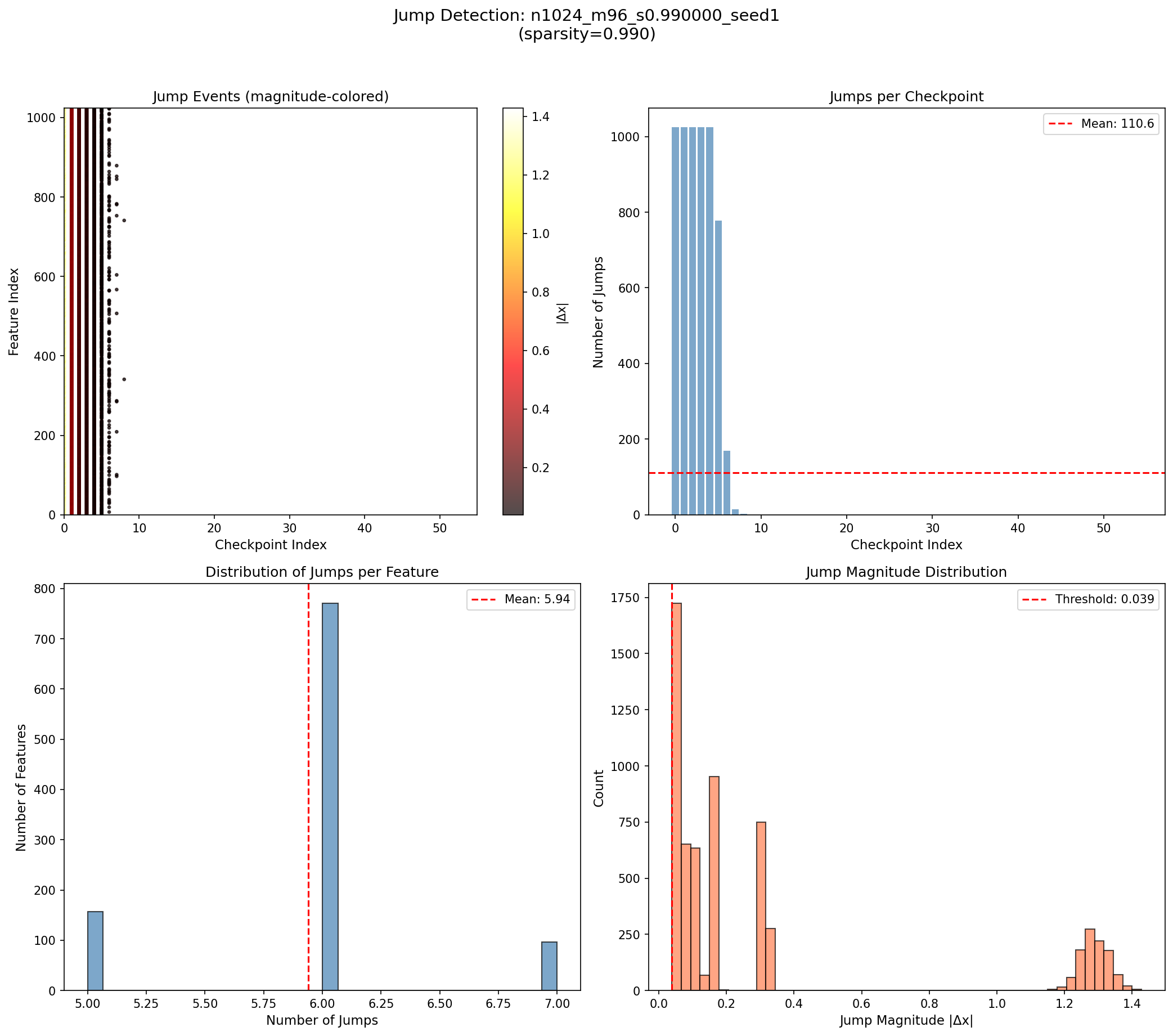}
        \captionof{figure}{Jump m96}
        \label{fig:m96_jump}
    \end{minipage}\hfill
    \begin{minipage}{0.31\textwidth}
        \centering
        \includegraphics[width=\textwidth]{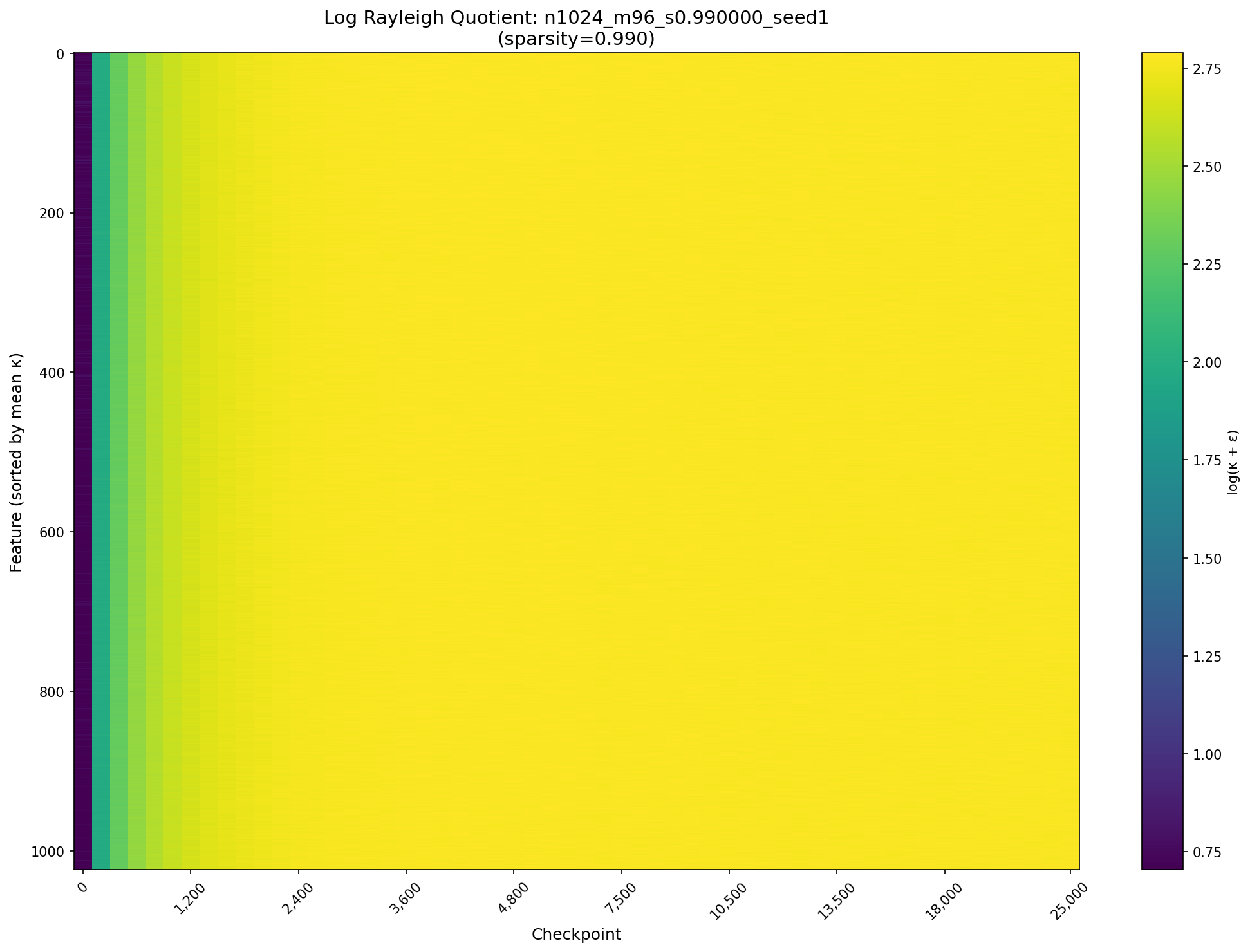}
        \captionof{figure}{Rayleigh m96}
    \end{minipage}\hfill
    \begin{minipage}{0.31\textwidth}
        \centering
        \includegraphics[width=\textwidth]{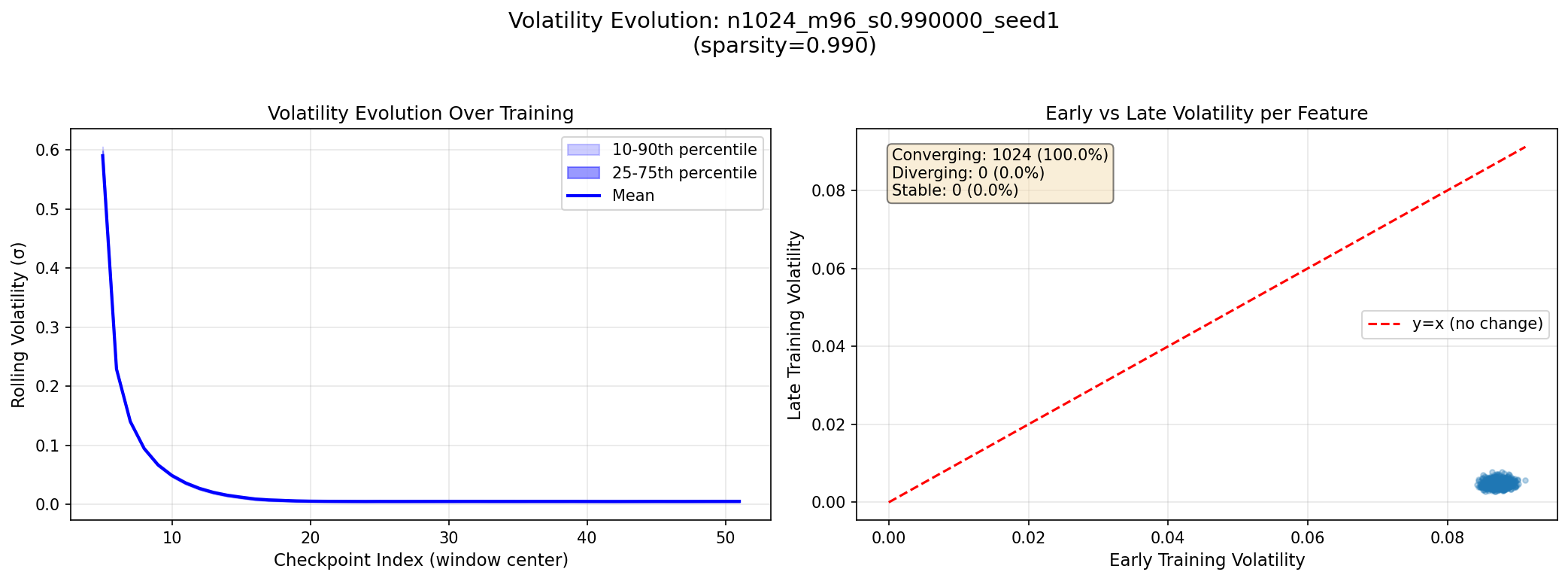}
        \captionof{figure}{Volatility m96}
    \end{minipage}
\end{minipage}
\vspace{2em}

\noindent\begin{minipage}{\linewidth}
    \centering
    \begin{minipage}{0.31\textwidth}
        \centering
        \includegraphics[width=\textwidth]{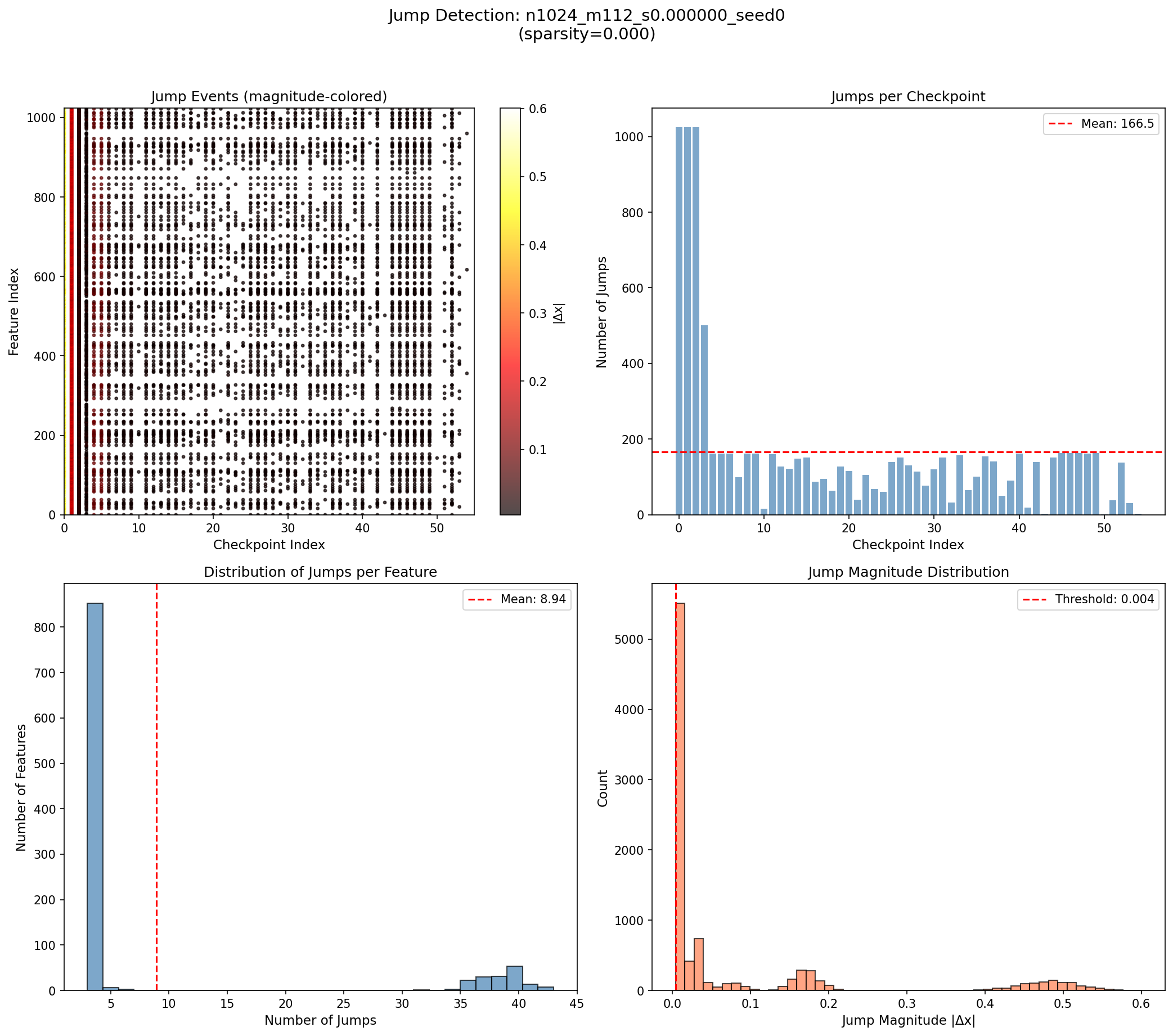}
        \captionof{figure}{Jump m112-0}
        \label{fig:m112_0_jump}
    \end{minipage}\hfill
    \begin{minipage}{0.31\textwidth}
        \centering
        \includegraphics[width=\textwidth]{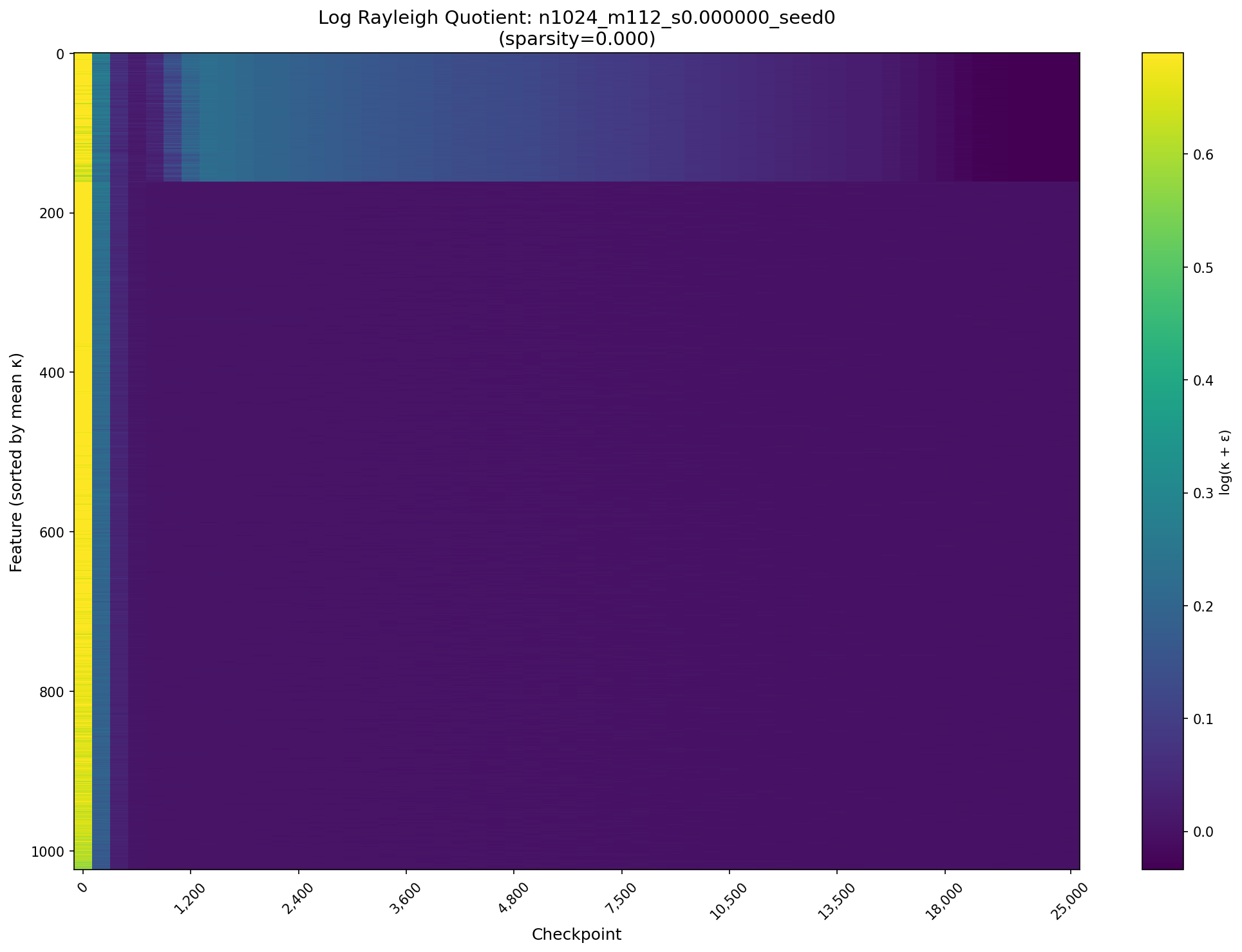}
        \captionof{figure}{Rayleigh m112-0}
    \end{minipage}\hfill
    \begin{minipage}{0.31\textwidth}
        \centering
        \includegraphics[width=\textwidth]{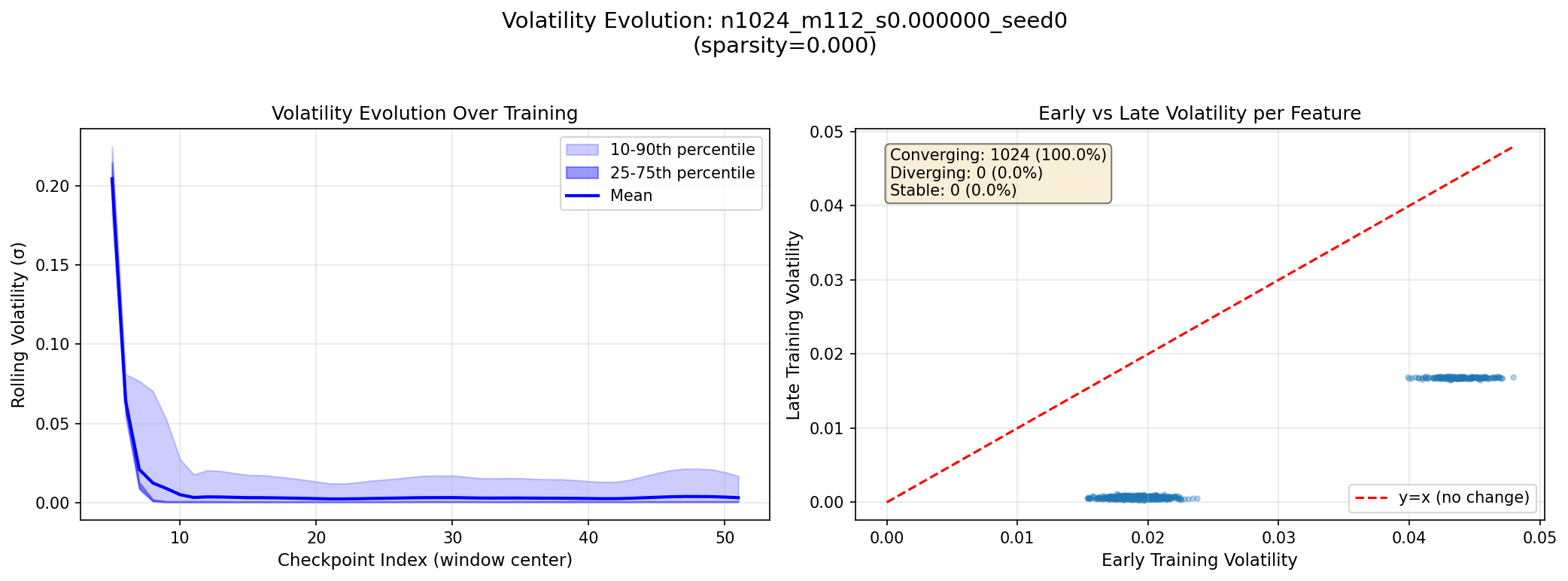}
        \captionof{figure}{Volatility m112-0}
    \end{minipage}
\end{minipage}
\vspace{2em}

\noindent\begin{minipage}{\linewidth}
    \centering
    \begin{minipage}{0.31\textwidth}
        \centering
        \includegraphics[width=\textwidth]{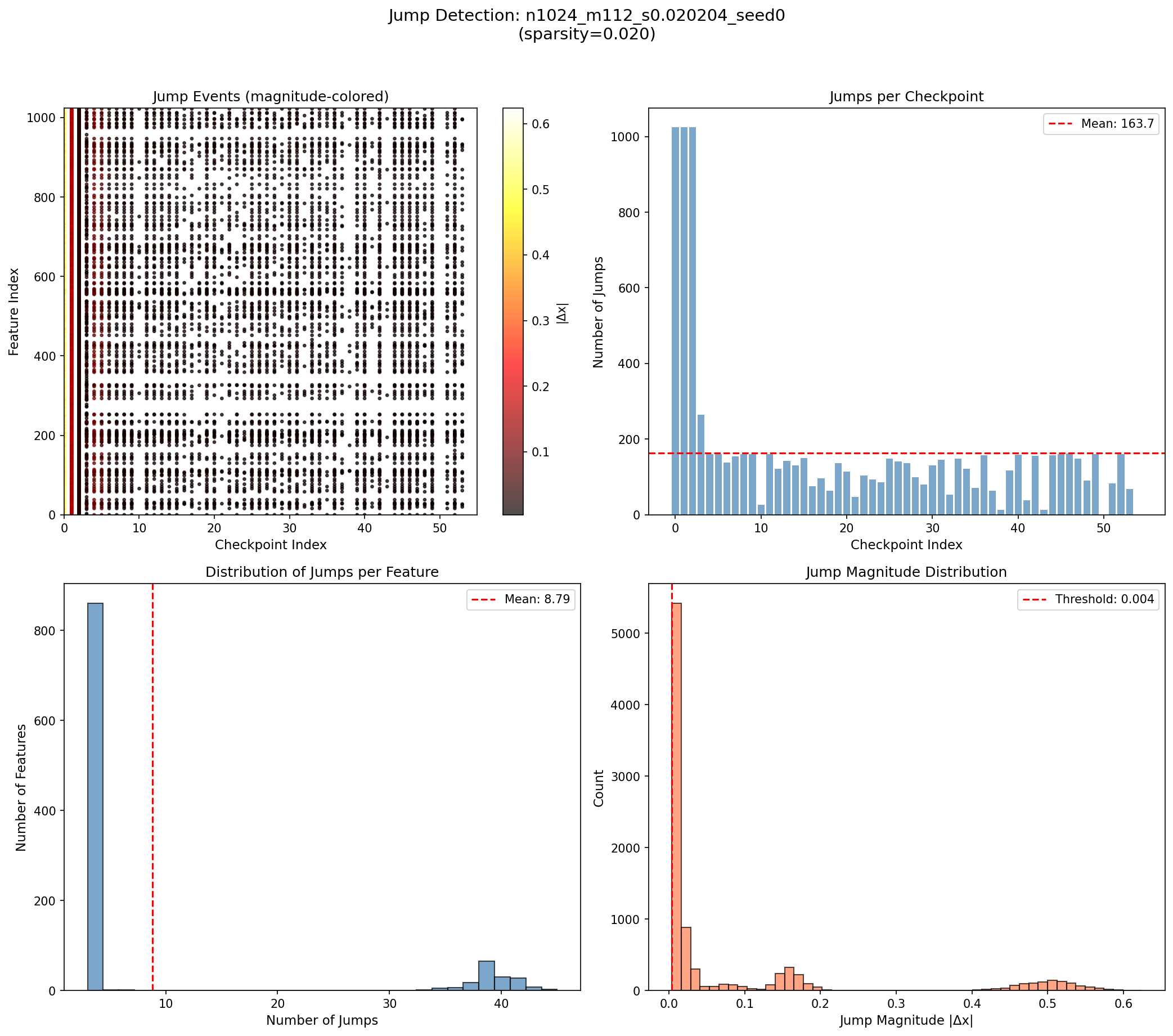}
        \captionof{figure}{Jump m112-2}
        \label{fig:m112_2_jump}
    \end{minipage}\hfill
    \begin{minipage}{0.31\textwidth}
        \centering
        \includegraphics[width=\textwidth]{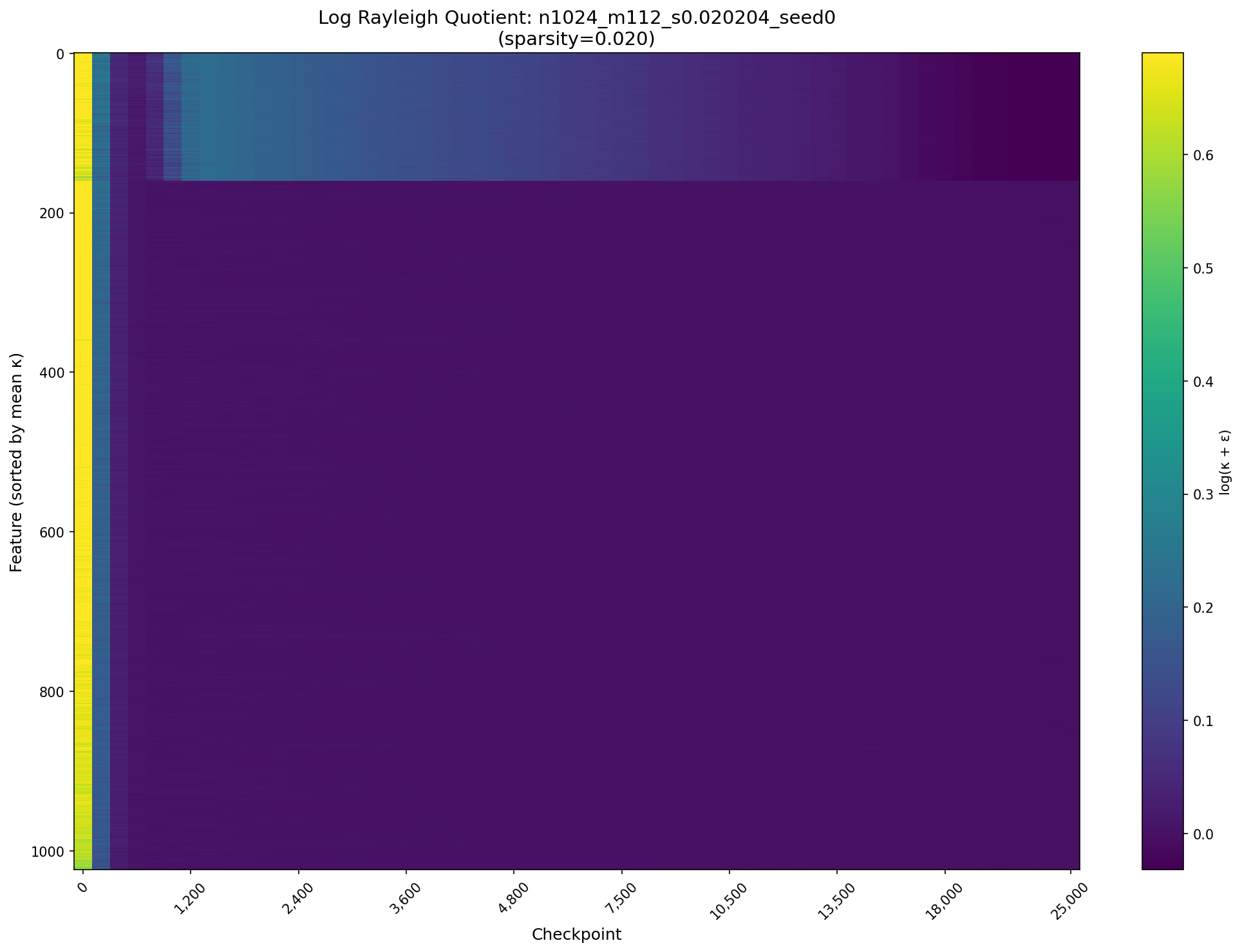}
        \captionof{figure}{Rayleigh m112-2}
    \end{minipage}\hfill
    \begin{minipage}{0.31\textwidth}
        \centering
        \includegraphics[width=\textwidth]{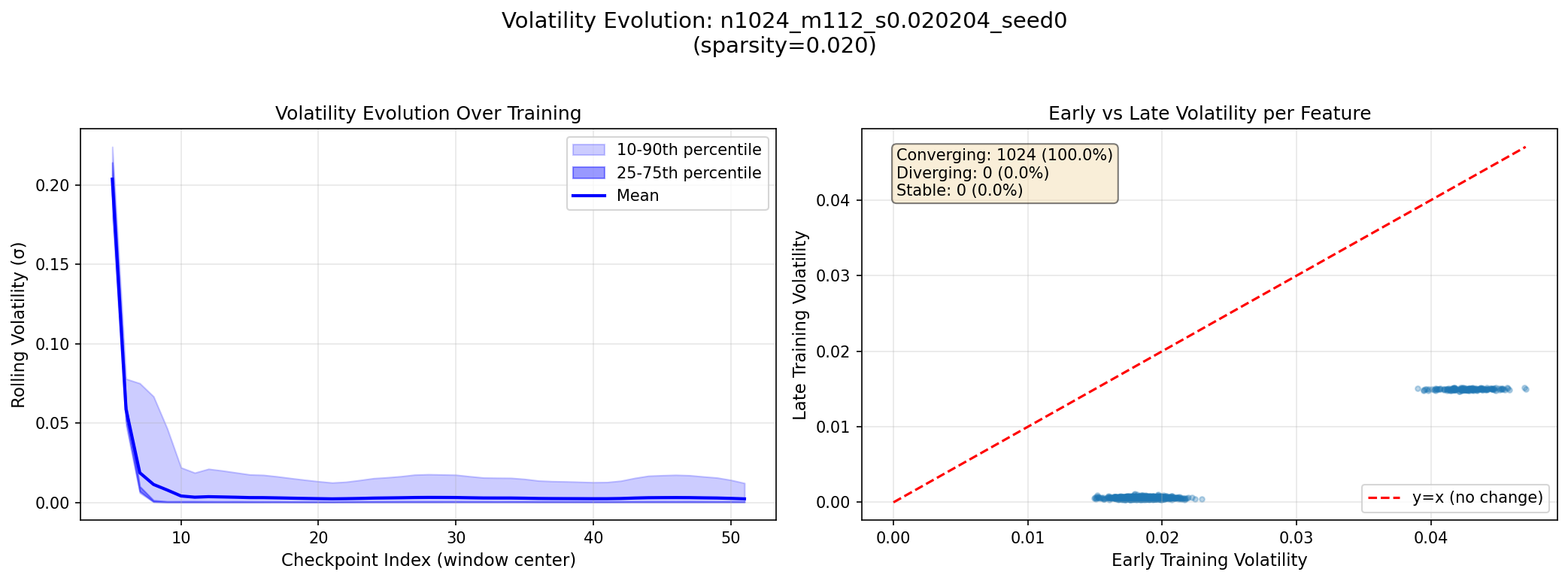}
        \captionof{figure}{Volatility m112-2}
    \end{minipage}
\end{minipage}
\vspace{2em}

\clearpage

\noindent\begin{minipage}{\linewidth}
    \centering
    \begin{minipage}{0.31\textwidth}
        \centering
        \includegraphics[width=\textwidth]{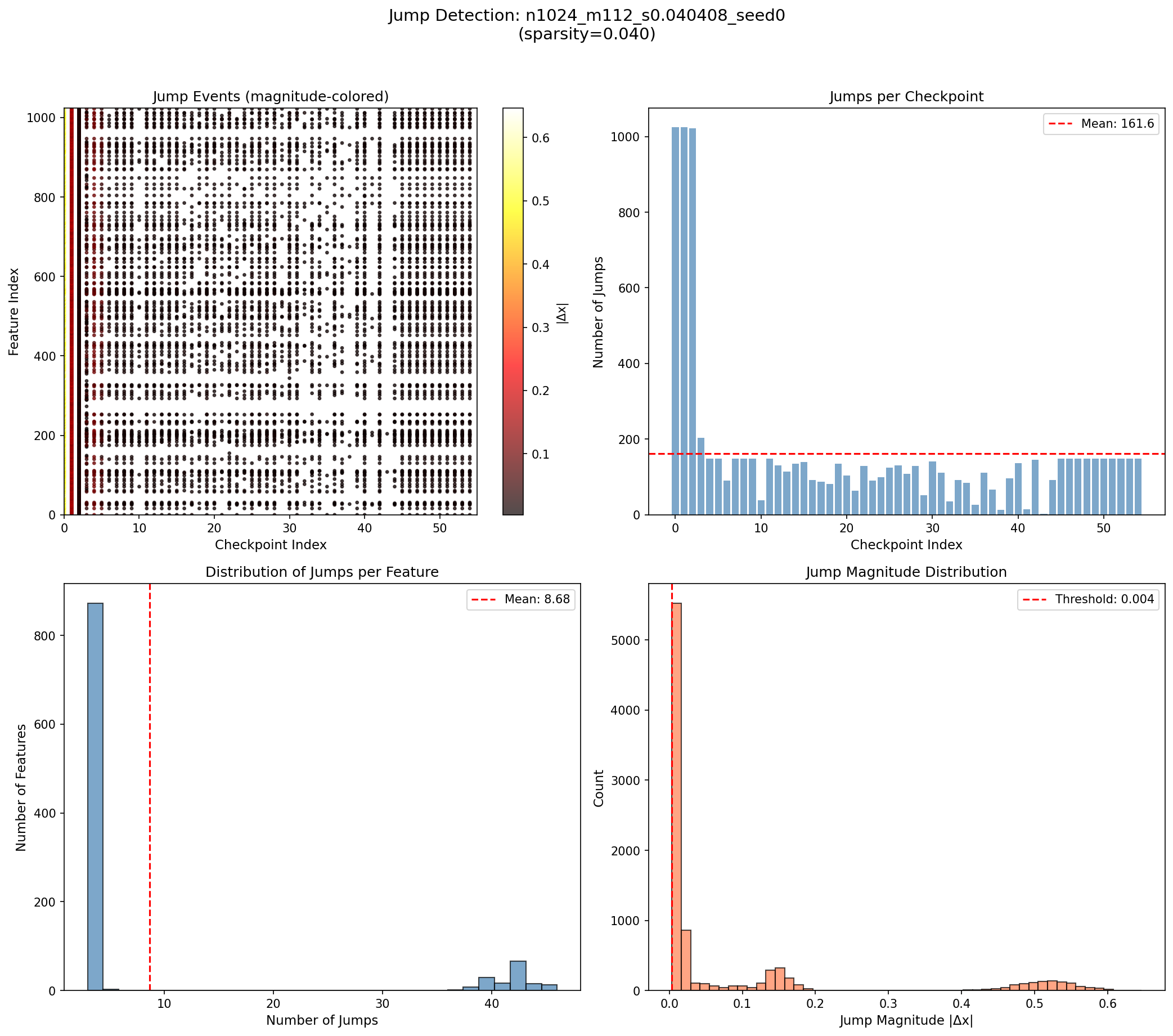}
        \captionof{figure}{Jump m112-4}
        \label{fig:m112_4_jump}
    \end{minipage}\hfill
    \begin{minipage}{0.31\textwidth}
        \centering
        \includegraphics[width=\textwidth]{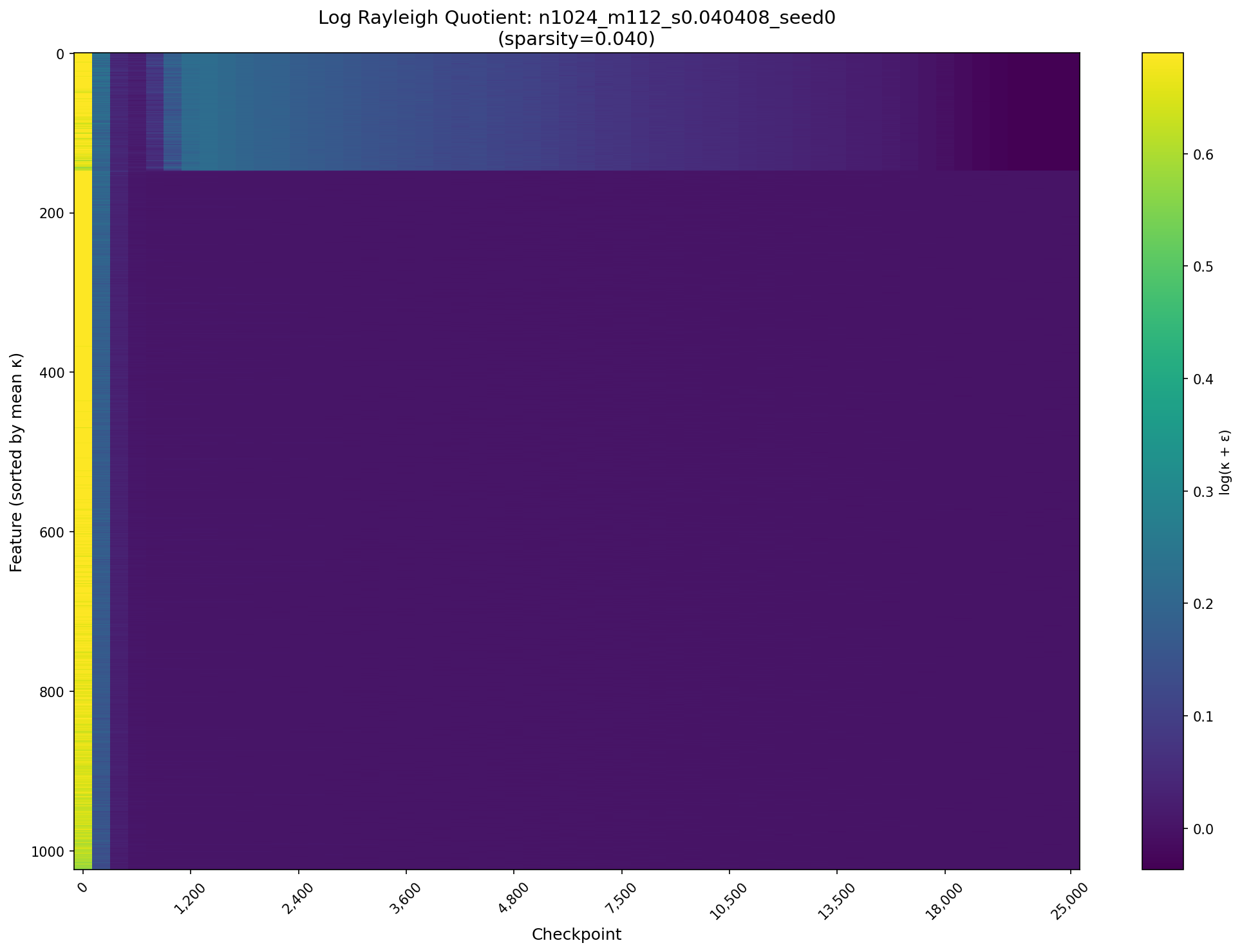}
        \captionof{figure}{Rayleigh m112-4}
    \end{minipage}\hfill
    \begin{minipage}{0.31\textwidth}
        \centering
        \includegraphics[width=\textwidth]{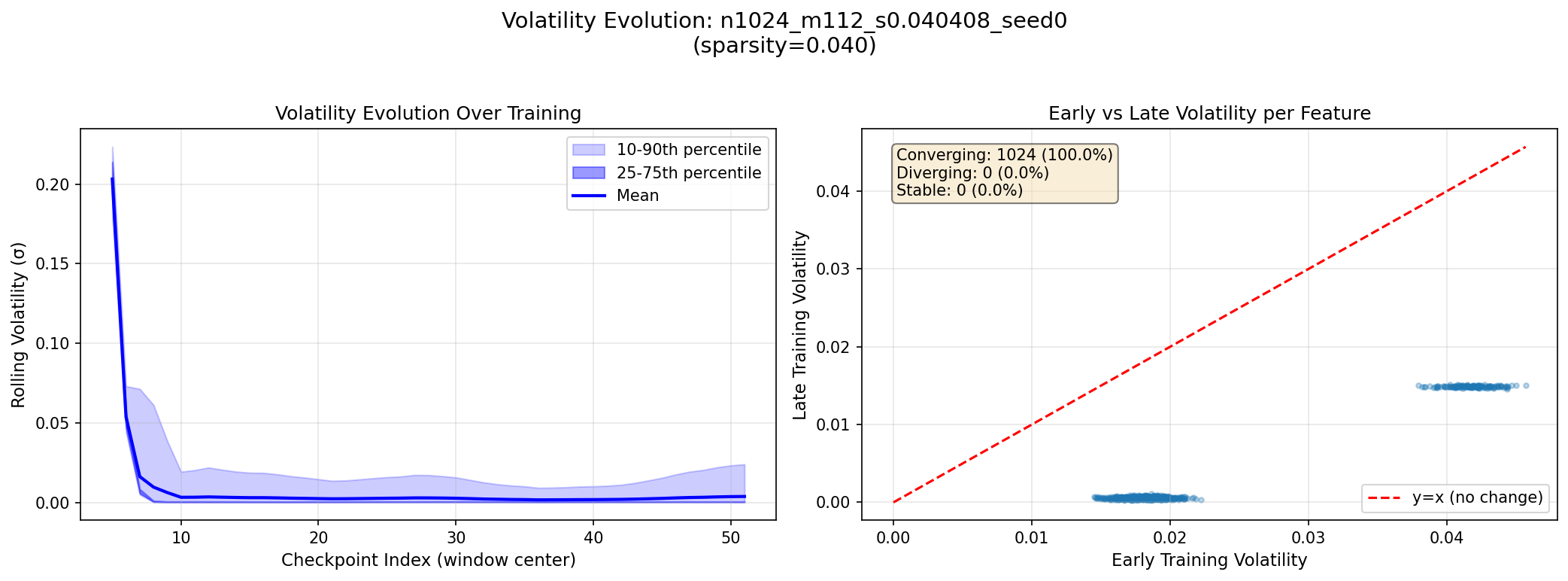}
        \captionof{figure}{Volatility m112-4}
    \end{minipage}
\end{minipage}
\vspace{2em}

\noindent\begin{minipage}{\linewidth}
    \centering
    \begin{minipage}{0.31\textwidth}
        \centering
        \includegraphics[width=\textwidth]{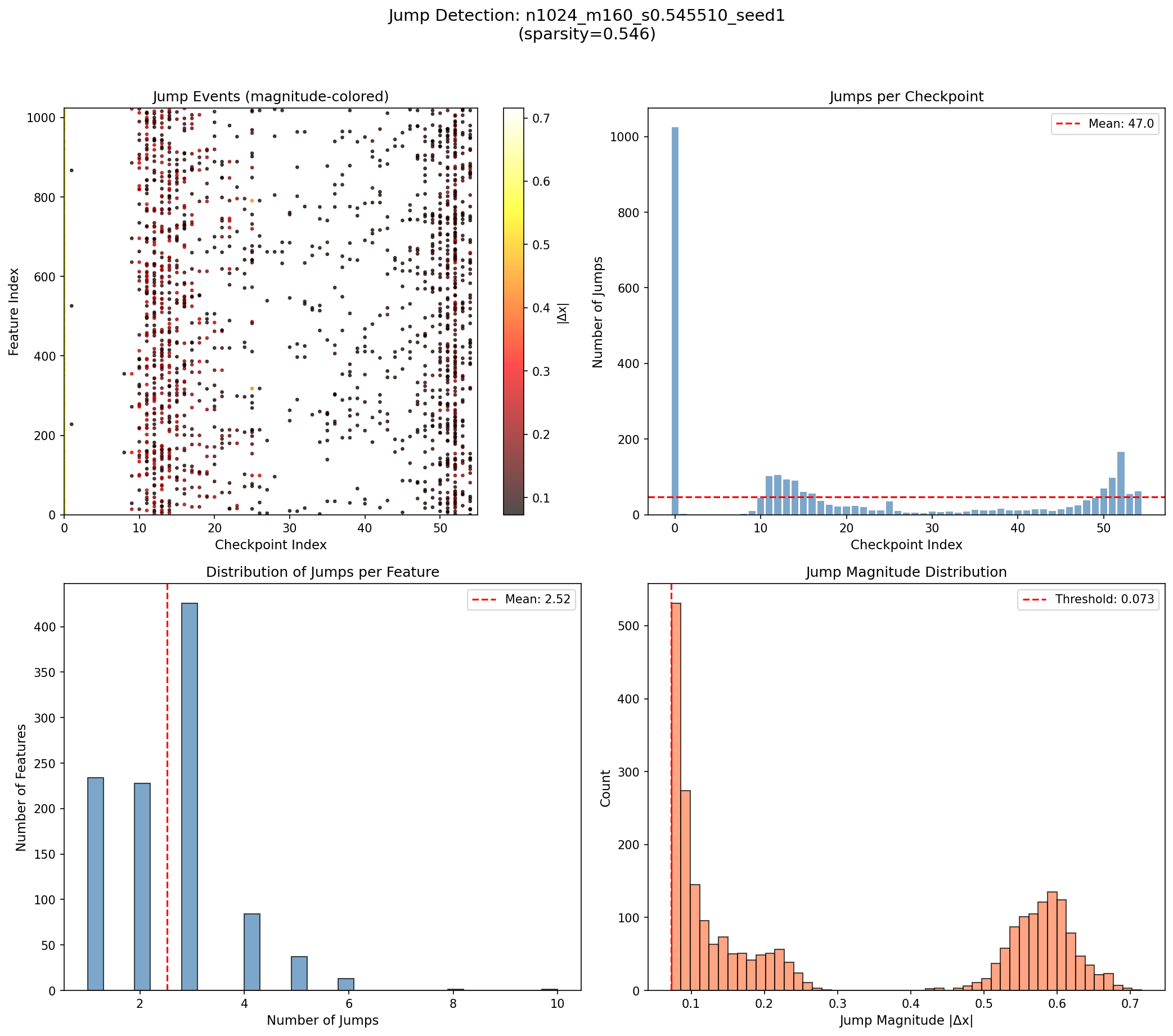}
        \captionof{figure}{Jump m160}
        \label{fig:m160_jump}
    \end{minipage}\hfill
    \begin{minipage}{0.31\textwidth}
        \centering
        \includegraphics[width=\textwidth]{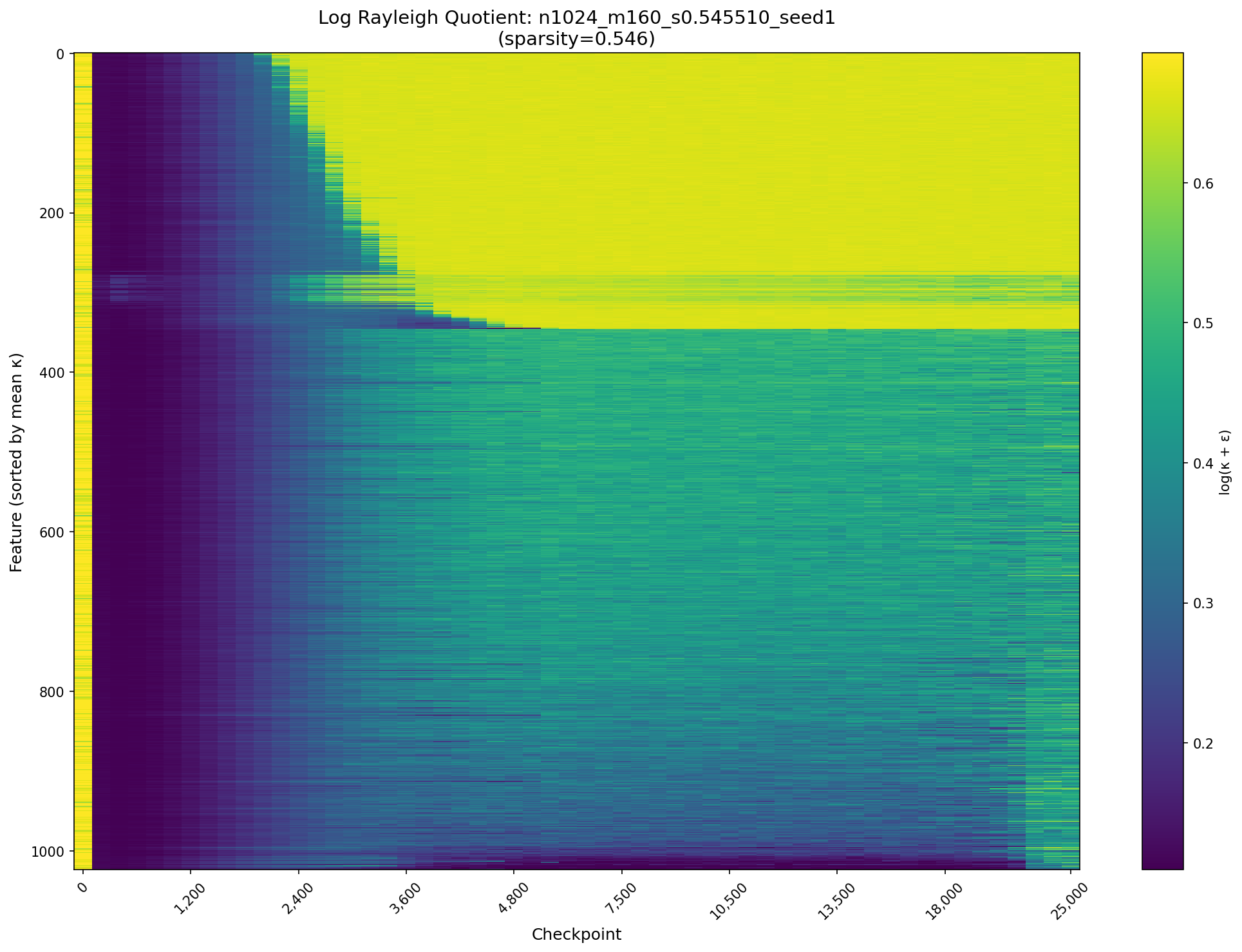}
        \captionof{figure}{Rayleigh m160}
    \end{minipage}\hfill
    \begin{minipage}{0.31\textwidth}
        \centering
        \includegraphics[width=\textwidth]{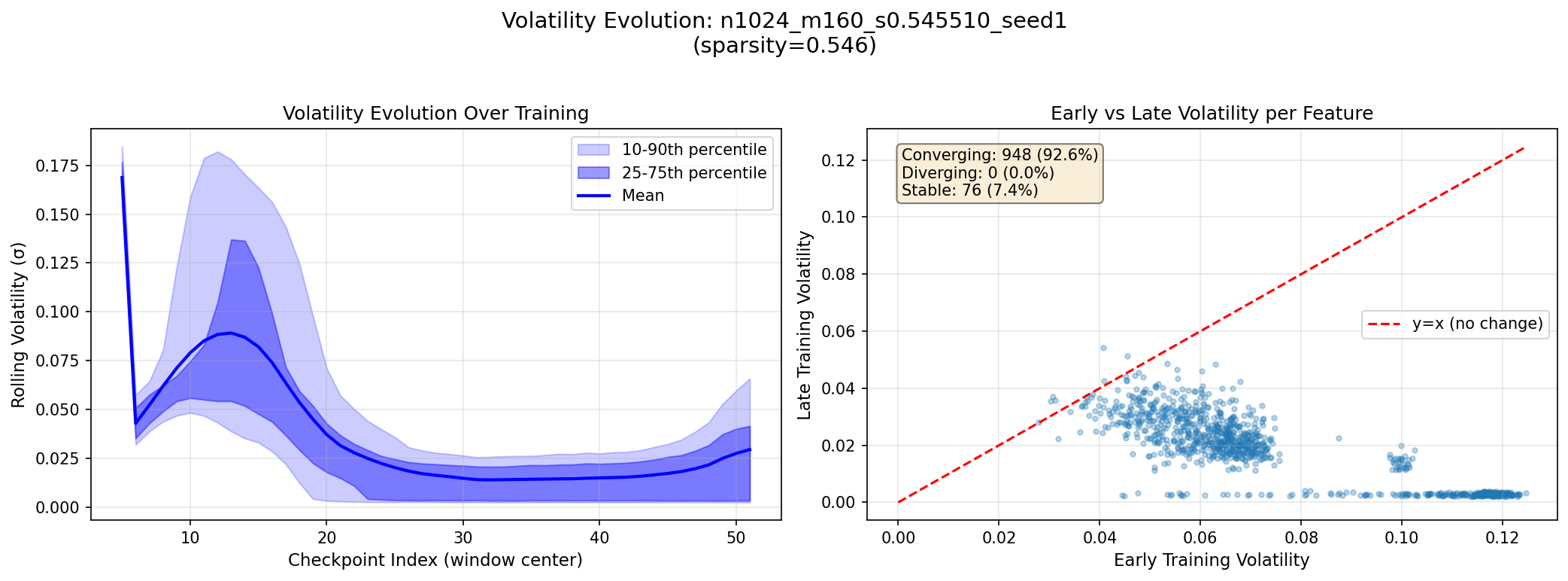}
        \captionof{figure}{Volatility m160}
    \end{minipage}
\end{minipage}
\vspace{2em}

\clearpage

\noindent\begin{minipage}{\linewidth}
    \centering
    \begin{minipage}{0.31\textwidth}
        \centering
        \includegraphics[width=\textwidth]{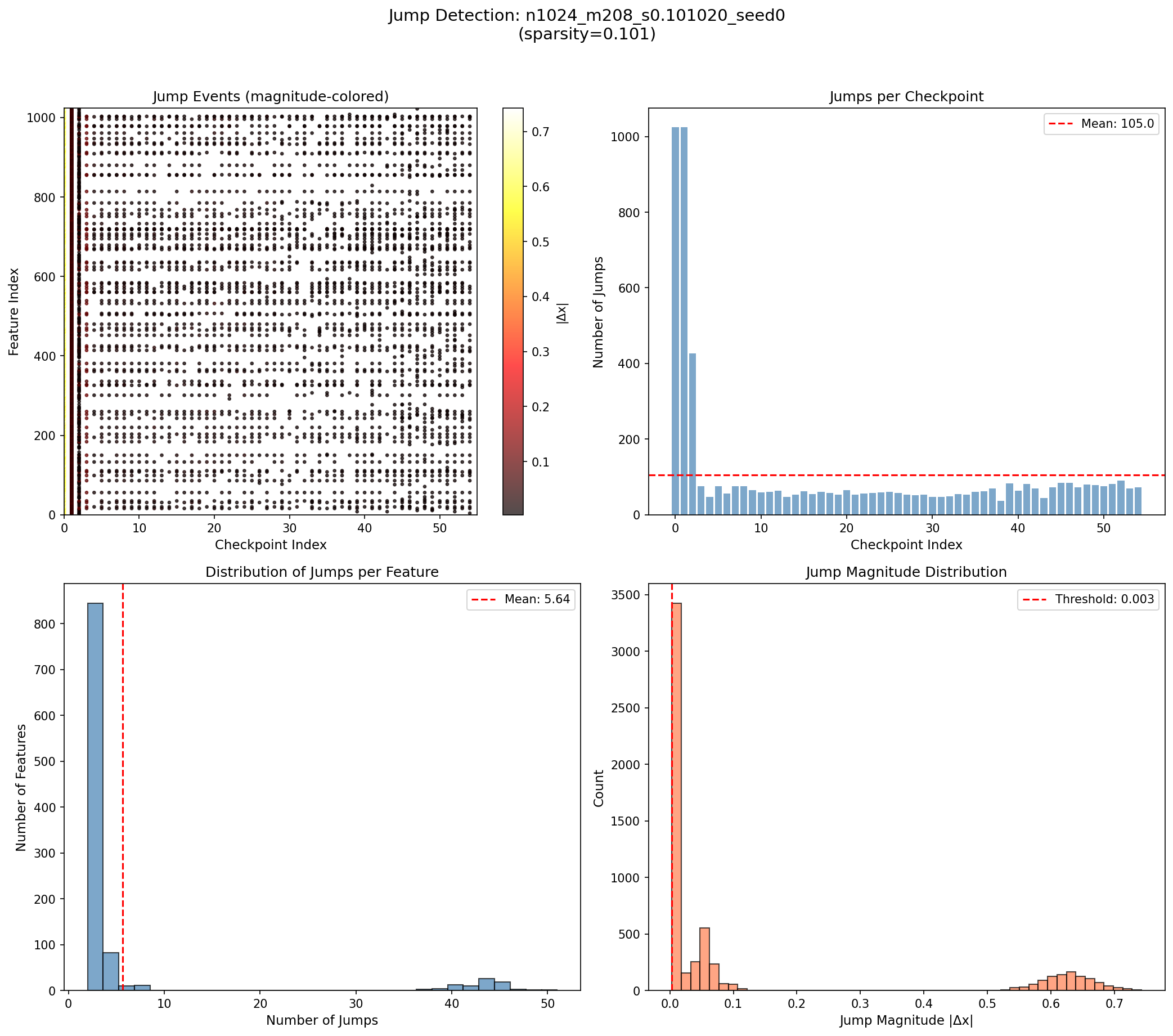}
        \captionof{figure}{Jump m208}
        \label{fig:m208_jump}
    \end{minipage}\hfill
    \begin{minipage}{0.31\textwidth}
        \centering
        \includegraphics[width=\textwidth]{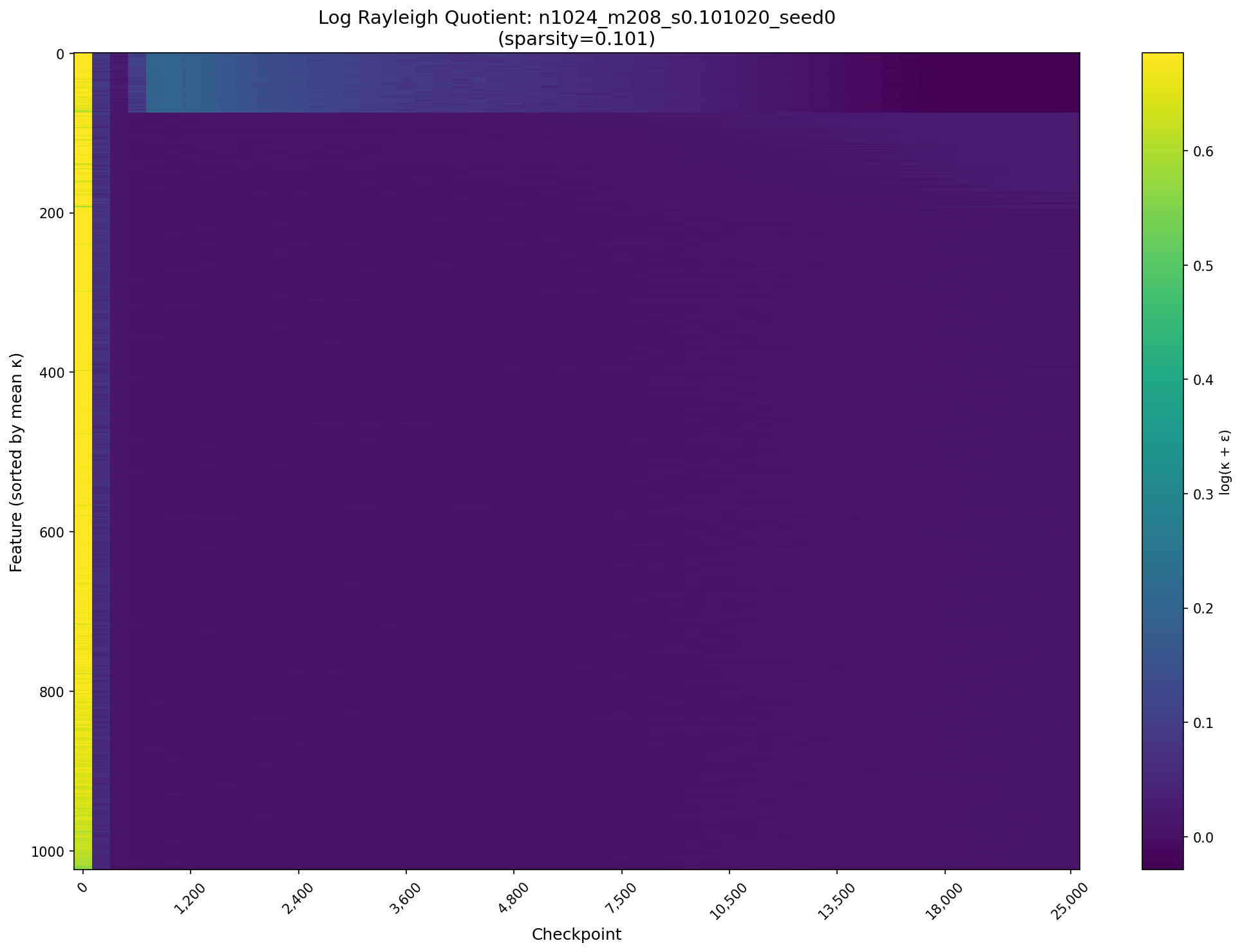}
        \captionof{figure}{Rayleigh m208}
    \end{minipage}\hfill
    \begin{minipage}{0.31\textwidth}
        \centering
        \includegraphics[width=\textwidth]{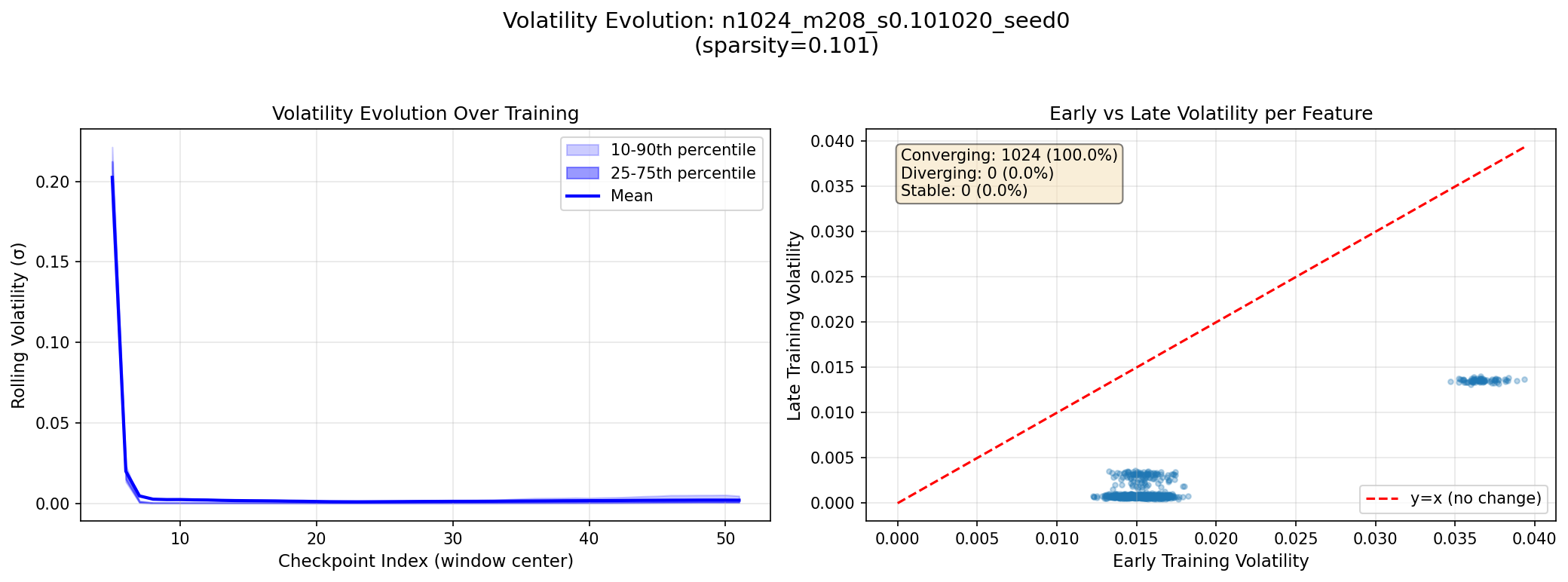}
        \captionof{figure}{Volatility m208}
    \end{minipage}
\end{minipage}
\vspace{2em}

\noindent\begin{minipage}{\linewidth}
    \centering
    \begin{minipage}{0.31\textwidth}
        \centering
        \includegraphics[width=\textwidth]{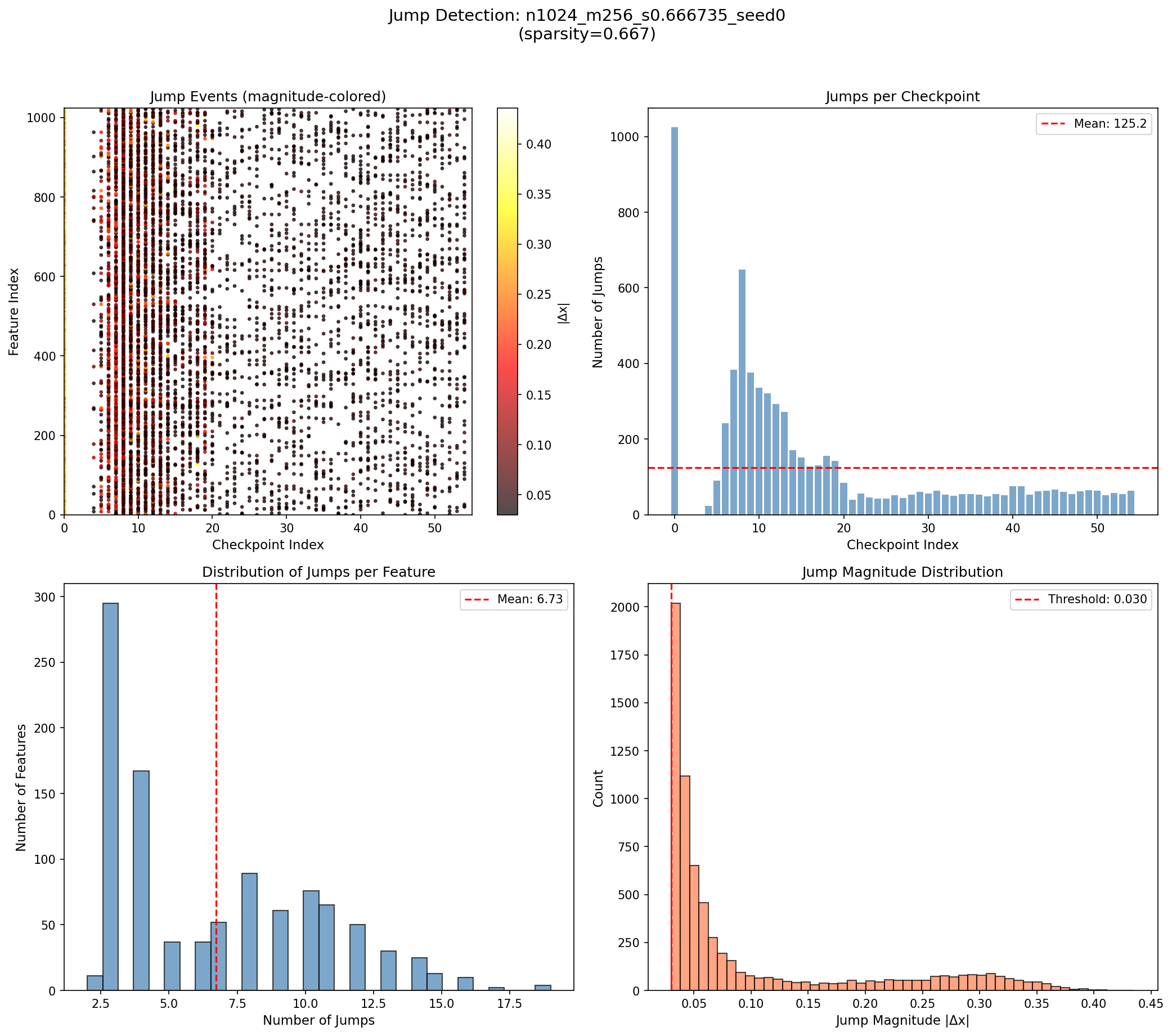}
        \captionof{figure}{Jump m256}
        \label{fig:m256_jump}
    \end{minipage}\hfill
    \begin{minipage}{0.31\textwidth}
        \centering
        \includegraphics[width=\textwidth]{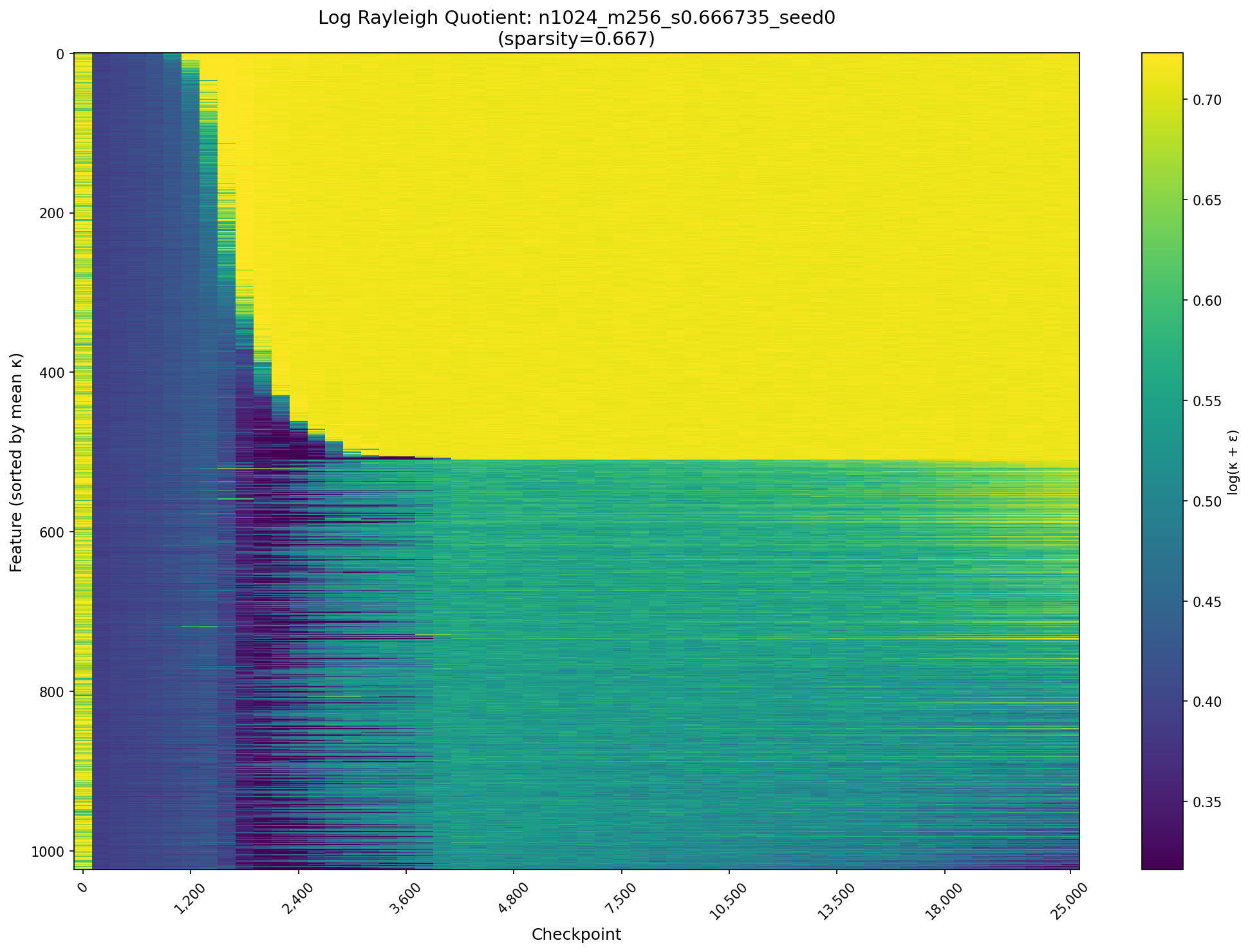}
        \captionof{figure}{Rayleigh m256}
    \end{minipage}\hfill
    \begin{minipage}{0.31\textwidth}
        \centering
        \includegraphics[width=\textwidth]{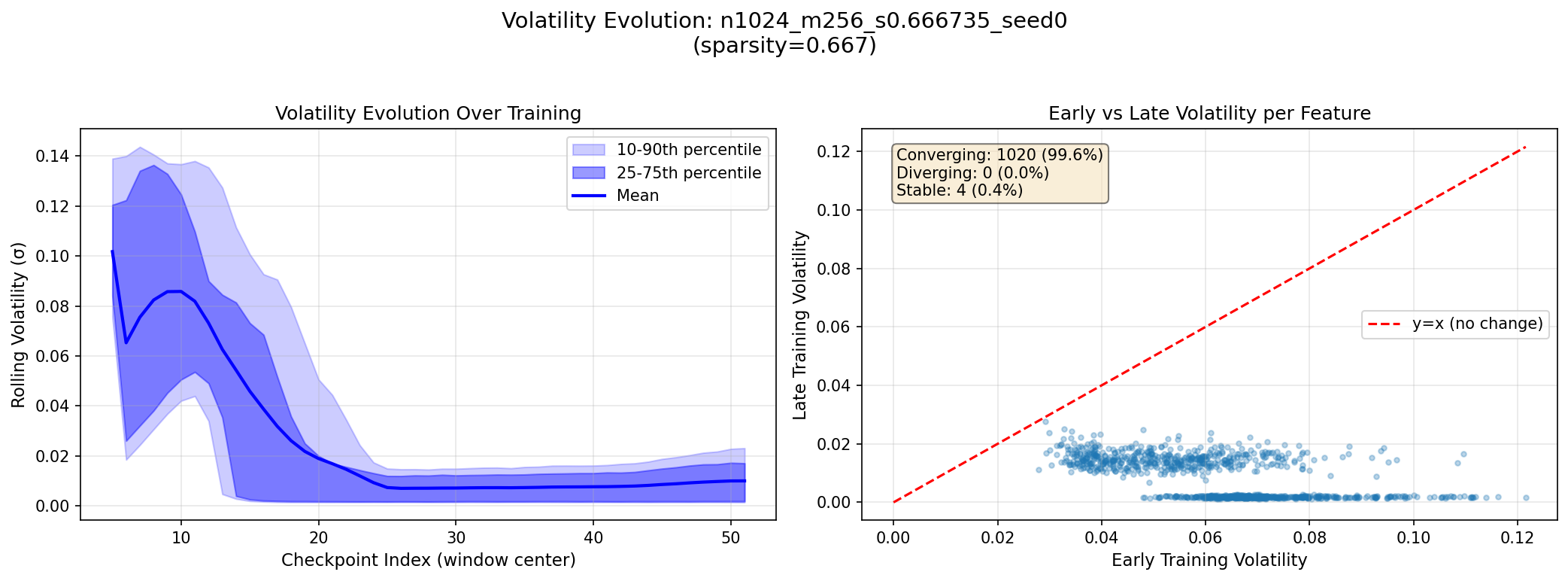}
        \captionof{figure}{Volatility m256}
    \end{minipage}
\end{minipage}
\vspace{2em}

\clearpage

\end{document}